\documentclass{article}

\usepackage[nonatbib,main,final]{neurips_2025}

\usepackage[sort&compress,round]{natbib}
\usepackage{float}
\usepackage[inline]{enumitem}
\usepackage{dsfont}
\usepackage[utf8]{inputenc} 
\usepackage[T1]{fontenc}    
\usepackage{url}            
\usepackage{booktabs}       
\usepackage{amsfonts}       
\usepackage{xcolor}         
\usepackage{chngcntr}

\newenvironment{proofSketch}{%
  \proof%
}{\endproof}



\usepackage{amsmath}
\DeclareMathOperator*{\minimize}{minimize}

\DeclareMathOperator*{\argmax}{arg\,max}
\DeclareMathOperator*{\argmin}{arg\,min} 

\DeclareMathAlphabet{\mathsfit}{T1}{\sfdefault}{\mddefault}{\sldefault}

\makeatletter
\usepackage[bookmarksopen=true,bookmarks=true,colorlinks=true,backref=page,breaklinks]{hyperref}
\renewcommand*{\backref}[1]{}
\renewcommand*{\backrefalt}[4]{({\footnotesize%
  \ifcase #1 Not cited.%
    \or page~#2%
    \else pages #2%
  \fi%
  })}

\makeatother

\usepackage[commenters={kyrkim}, noautonum]{shortex}

\usepackage{proof-at-the-end}

\allowdisplaybreaks

\title{
    Nearly Dimension-Independent Convergence of\\Mean-Field Black-Box Variational Inference
}
\author{
  Kyurae Kim\\
  University of Pennsylvania\\
  \href{mailto:kyrkim@seas.upenn.edu}{\texttt{kyrkim@seas.upenn.edu}}\\
  \And
  Yi-An Ma \\
  University of California San Diego \\
  \href{mailto:yianma@ucsd.edu}{\texttt{yianma@ucsd.edu}}\\
  \And
  Trevor Campbell \\
  University of British Columbia\\
  \href{mailto:trevor@stat.ubc.ca}{\texttt{trevor@stat.ubc.ca}}
  \And
  Jacob R. Gardner\\
  University of Pennsylvania\\
  \href{mailto:jacobrg@seas.upenn.edu}{\texttt{jacobrg@seas.upenn.edu}}
}

\begin{document}

\maketitle
\begin{abstract}
    We prove that, given a mean-field location-scale variational family, black-box variational inference (BBVI) with the reparametrization gradient converges at a rate that is nearly independent of explicit dimension dependence.
    Specifically, for a $d$-dimensional strongly log-concave and log-smooth target, the number of iterations for BBVI with a sub-Gaussian family to obtain a solution $\epsilon$-close to the global optimum has a dimension dependence of $\mathrm{O}(\log d)$.
    This is a significant improvement over the $\mathrm{O}(d)$ dependence of full-rank location-scale families.
    For heavy-tailed families, we prove a weaker $\mathrm{O}(d^{2/k})$ dependence, where $k$ is the number of finite moments of the family.
    Additionally, if the Hessian of the target log-density is constant, the complexity is free of any explicit dimension dependence.
    We also prove that our bound on the gradient variance, which is key to our result, cannot be improved using only spectral bounds on the Hessian of the target log-density.
\end{abstract}

\section{Introduction}
Variational inference (VI; \citealp{peterson_explorations_1989,jordan_introduction_1999,hinton_keeping_1993,blei_variational_2017}) is an effective method for approximating intractable high-dimensional distributions and models with tall datasets.
Among various VI algorithms, black-box VI~(BBVI; \citealp{kucukelbir_automatic_2017,ranganath_black_2014,wingate_automated_2013,titsias_doubly_2014}), which minimizes the exclusive KL divergence~\citep{kullback_information_1951} via stochastic gradient descent (SGD; \citealp{robbins_stochastic_1951,bottou_optimization_2018}) in the space of parameters.
BBVI is widely used in practice due to its flexibility to apply to a wide range of variational families with only minor modifications~\citep{patil_pymc_2010,bingham_pyro_2019,carpenter_stan_2017,ge_turing_2018,fjelde_turingjl_2025}.
Specifically, location-scale variational families---in which 
a base distribution is mutated by an affine transformation---remain a popular choice, encompassing
those with diagonal scale matrices (the ``mean-field'' approximation;~\citealp{peterson_explorations_1989,hinton_keeping_1993}), as well as scale matrices with low rank~\citep{rezende_stochastic_2014,ong_gaussian_2018,tomczak_efficient_2020} and full-rank~\citep{titsias_doubly_2014,kucukelbir_automatic_2017} factors.

The choice of the variational family is generally known to affect the convergence speed of BBVI, where families that are more ``expressive,'' those that contain more complex distributions, result in slower convergence.
For example, in location-scale families, it has been empirically observed that mean-field families often provide faster convergence to an accurate posterior approximation than full-rank families~\citep{ko_provably_2024,giordano_covariances_2018,agrawal_advances_2020,zhang_pathfinder_2022,giordano_black_2024}.
This is because full-rank families often require running SGD with a smaller step size and for longer; even given a large computation budget, BBVI on a full-rank family may not converge adequately~\citep{ko_provably_2024}.
Therefore, choosing the expressiveness of the family corresponds to trading statistical accuracy for computational efficiency~\citep{bhatia_statistical_2022}.
In order to control this trade-off for our benefit, a clear theoretical understanding of the relationship between convergence speed and expressiveness is needed.

Formally, consider the setting of approximating a $\mu$-strongly log-concave and $L$-log-smooth target, where $\kappa \triangleq L/\mu$ is the condition number.
For BBVI with the reparametrization gradient~\citep{titsias_doubly_2014,rezende_stochastic_2014,kingma_autoencoding_2014} on a full-rank location-scale family, an $\epsilon$-close solution to the global optimum in squared distance in parameter space can be obtained after at least $\mathrm{O}(d \kappa^2 \epsilon^{-1})$ iterations~\citep{domke_provable_2019,kim_convergence_2023}.
For mean-field location-scale families, on the other hand, the iteration complexity improves to $\mathrm{O}(\sqrt{d} \kappa^2 \epsilon^{-1})$~\citep{kim_convergence_2023}.
While this is clearly better than the $\mathrm{O}(d)$ explicit dimension dependence of full-rank families, it has been conjectured that a better dependence is more likely~\citep{kim_convergence_2023}.

In this work, we positively resolve this conjecture by obtaining stronger convergence guarantees for BBVI on mean-field location-scale families (\cref{section:main_results}).
In particular, under the conditions stated above, we prove that BBVI with a mean-field location-scale family with sub-Gaussian tails can obtain an $\epsilon$-accurate solution in squared distance after $\mathrm{O}( (\log d) \kappa^2 \epsilon^{-1})$ iterations. 
Heavier-tailed families achieve a weaker $\mathrm{O}( d^{2/k} \kappa^2 \epsilon^{-1})$ iteration complexity guarantee, where $k$ is the number of finite moments of the variational family.
For the Student-$t$ variational family with a high-enough degrees of freedom $\nu$, this corresponds to a $\mathrm{O}(d^{2/(\nu - 2)})$ explicit dimension dependence.
In addition, if the Hessian of the target log-density is constant, any mean-field location-scale family attains a $\mathrm{O}(\kappa^2 \epsilon^{-1})$ iteration complexity without any explicit dependence on $d$.

The key element of the proof is a careful probabilistic analysis of the variance of the reparametrization gradient (\cref{section:gradient_variance_analysis}):
In general, the reparametrization gradient of the scale parameters contains heavy-tailed components that grow not-so-slowly in $d$.
However, for mean-field families, only a \textit{single} random coordinate turns out to be heavy-tailed.
Through a probabilistic decomposition, the influence of this heavy-tailed component can be averaged out over all $d$ coordinates.
Then the lighter-tailed components of the gradient dominate as $d$ increases, resulting in a benign dimension dependence (\cref{thm:gradient_variance_upperbound_meanfield_general}).
We also provide a lower bound (\cref{thm:gradient_variance_lowerbound}) showing that our analysis cannot be improved when using only spectral bounds on the Hessian of the target log density.

\vspace{-1ex}
\section{Preliminaries}\label{section:background}
\vspace{-1ex}

\paragraph{Notation}
We denote random variables in sans serif (\textit{e.g.}, $\mathsfit{u}$, $\mathsfit{U}$).
\(\mathbb{S}_{\succ 0}^d \subset \mathbb{R}^{d \times d}\) denotes the set of $d \times d$ positive definite (PD) matrices, \(\mathbb{D}^d \subset \mathbb{R}^{d \times d}\) denotes the set of diagonal matrices, and \(\mathbb{D}^d_{\succ 0} \subset \mathbb{D}^d \cap \mathbb{S}_{\succ 0}^{d \times d}\) is its positive definite subset.
\(\inner{\cdot, \cdot}\) and $\norm{\cdot}_2$ denote the Euclidean inner product and norm.
For a matrix \(A \in \mathbb{R}^{d \times d}\), \(\norm{A}_{\mathrm{F}} = \sqrt{\mathrm{tr}\lt(A^{\top}A\rt)}\) is the Frobenius norm, \(\norm{A}_2 = \sigma_{\mathrm{max}}\lt(A\rt)\) is the \(\ell_2\) operator norm, where \(\sigma_{\mathrm{max}}\lt(\cdot\rt)\) and \(\sigma_{\mathrm{min}}\lt(\cdot\rt)\) are the largest and smallest singular values.

\vspace{-1ex}
\subsection{Problem Setup}
\vspace{-1ex}
Our problem of interest is an optimization problem over some space \(\Lambda \subseteq \mathbb{R}^p \) of the form of
{%
\setlength{\belowdisplayskip}{.5ex} \setlength{\belowdisplayshortskip}{.5ex}
\setlength{\abovedisplayskip}{.5ex} \setlength{\abovedisplayshortskip}{.5ex}
\[
    \minimize_{\lambda \in \Lambda} \; 
    \lt\{ 
        F\lt(\lambda\rt) 
        \triangleq 
        f\lt(\lambda\rt) + h\left(\lambda\right) 
    \rt\} \; ,
    \quad\text{where}\quad
    f\lt(\lambda\rt)
    \triangleq
    \mathbb{E}_{\mathsfit{z} \sim q_{\lambda}} \ell\left(\mathsfit{z}\right) \; ,
    \label{eq:objective}
\]
}%
$\ell : \mathbb{R}^{d} \to \mathbb{R}$ is a measurable function we refer to as the ``target function'', $h : \Lambda \to \mathbb{R}$ is a potentially non-smooth convex regularizer, and the expectation $\mathbb{E}_{\mathsfit{z} \sim q_{\lambda}} \ell\lt(\mathsfit{z}\rt)$ is assumed to be intractable.

BBVI is a special case of \cref{eq:objective} where \(\ell = -\log \pi\) is the negative (unnormalized) log density of some distribution $\pi$ with respect to the Lebesgue measure and \(h(\lambda) = -\mathbb{H}[q_{\lambda}]\) is the negative differential entropy of $q_{\lambda}$.
Then $F$ is the exclusive Kullback-Leibler divergence $\mathrm{D}_{\mathrm{KL}}$~\citep{kullback_information_1951} up to an additive constant~\citep{jordan_introduction_1999}, where \cref{eq:objective} reduces to 
{%
\setlength{\belowdisplayskip}{1ex} \setlength{\belowdisplayshortskip}{1ex}
\setlength{\abovedisplayskip}{2ex} \setlength{\abovedisplayshortskip}{2ex}
\[
    \minimize_{\lambda \in \Lambda} \; 
    \Big\{ 
        \mathrm{D}_{\mathrm{KL}}\lt(q_{\lambda}, \pi\rt) 
        \propto
        - \mathbb{E}_{\mathsfit{z} \sim q_{\lambda}} \log \pi\lt(\mathsfit{z}\rt) - \mathbb{H}\left(q_{\lambda}\right) 
    \Big\} \; ,
\]
}%
We assume $\pi$ is supported on $\mathbb{R}^d$, which, unless discrete-valued variables are involved, is often valid after appropriate support transformations~\citep[\S 2.2]{kim_convergence_2023}.
Such a setup for BBVI has been proposed by~\citet{kucukelbir_automatic_2017}, and now encompasses most practical use of BBVI with the reparametrization gradient as implemented in Stan~\citep{carpenter_stan_2017}, PyMC~\citep{patil_pymc_2010}, Pyro~\citep{bingham_pyro_2019}, and Turing~\citep{ge_turing_2018,fjelde_turingjl_2025}.

For the purpose of a quantitative theoretical analysis, we will consider the following properties:
\begin{definition*}[Smoothness]
For some $\phi : \mathbb{R}^d \to \mathbb{R}$, we say $\phi$ is L-(Lipschitz )smooth if there exists some \(L \in (0, +\infty)\) such that, for all \(z, z' \in \mathbb{R}^d\), 
{%
\setlength{\belowdisplayskip}{0ex} \setlength{\belowdisplayshortskip}{0ex}
\setlength{\abovedisplayskip}{1ex} \setlength{\abovedisplayshortskip}{1ex}
\[
    \norm{\nabla \phi \lt(z\rt) - \nabla \phi \lt(z'\rt) }_2 \leq L \norm{z - z'}_2 \; .
    \nonumber
\]
}%
\end{definition*}
\vspace{-1ex}
\begin{definition*}[Strong Convexity]
For some $\phi : \mathbb{R}^d \to \mathbb{R}$, we say $\phi$ is $\mu$-strongly convex if there exists some constant \(\mu \in (0, L]\) such that, for all \(z, z' \in \mathbb{R}^d\), 
{%
\setlength{\belowdisplayskip}{0ex} \setlength{\belowdisplayshortskip}{0ex}
\setlength{\abovedisplayskip}{0ex} \setlength{\abovedisplayshortskip}{0ex}
\[
    \inner{ \nabla \phi(z) , z - z' } \geq  \phi(z) - \phi(z')  + \frac{\mu}{2} \norm{ z - z' }_2^2 \; .
    \nonumber
\]
}%
\end{definition*}

In the context of BBVI, assuming that $\ell = -\log \pi$ is both $\mu$-strongly convex and $L$-smooth is equivalent to assuming $\pi$ is $\mu$-strongly log-concave and $L$-log-Lipschitz smooth, respectively, which is common in the analysis of MCMC~\citep{chewi_logconcave_2024} and VI~\citep{kim_convergence_2023,domke_provable_2023,lambert_variational_2022,diao_forwardbackward_2023,arnese_convergence_2024,lavenant_convergence_2024}.

\vspace{-1ex}
\subsection{Variational Family}
\vspace{-1ex}

We consider the location-scale family~\citep[\S 3.5]{casella_statistical_2001}:

\begin{definition}[Location-Scale Variational Family]\label{def:locscale}
    A family of distributions \(\mathcal{Q}\) is referred to as a location-scale variational family if there exists some univariate distribution \(\varphi\) dominated by the Lebesgue measure such that each member of \(\mathcal{Q}\) indexed by \(\lambda = (m, C) \in \mathbb{R}^d \times \mathcal{C}\), where \(\mathcal{C} \subset \mathbb{R}^{d \times d}\) and \(q_{\lambda} \in \mathcal{Q}\), satisfies
    {%
    \setlength{\belowdisplayskip}{0ex} \setlength{\belowdisplayshortskip}{0ex}
    \setlength{\abovedisplayskip}{0ex} \setlength{\abovedisplayshortskip}{0ex}
    \[
        \mathsfit{z} \sim q_{\lambda}  \qquad\Leftrightarrow\qquad  \mathsfit{z} \stackrel{\mathrm{d}}{=} \mathcal{T}_{\lambda}\left(\mathsfit{u}\right) \; ,
        \nonumber
    \]
    }%
    where 
    {%
    \setlength{\belowdisplayskip}{0ex} \setlength{\belowdisplayshortskip}{0ex}
    \setlength{\abovedisplayskip}{0ex} \setlength{\abovedisplayshortskip}{0ex}
    \[
        \mathcal{T}_{\lambda}\left(\mathsfit{u}\right) &\triangleq C \mathsfit{u} + m, 
        \qquad
        \mathsfit{u} \triangleq \lt(\mathsfit{u}_1, \ldots, \mathsfit{u}_d \rt), 
        \qquad 
        \mathsfit{u}_i \stackrel{\mathrm{i.i.d.}}{\sim} \varphi \; ,
        \nonumber
    \]
    }%
    and \(\stackrel{\mathrm{d}}{=}\) is equivalence in distribution. 
    Then $\mathcal{T}_{\lambda}$ is referred to as the ``reparametrization function,'' while $m$ and $C$ are referred to as the location and scale parameters, respectively.
\end{definition}
In addition, we impose mild regularity assumptions on the moments of the base distribution:
\begin{assumption}\label{assumption:noise}
    $\varphi$ satisfies the following: 
    \begin{enumerate*}[label=(\roman*)]
        \item It is standardized such that \(\mathbb{E}\mathsfit{u}_i = 0\) and \(\mathbb{E}\mathsfit{u}_i^2 = 1\),
        \item symmetric such that \(\mathbb{E}\mathsfit{u}_i^3 = 0\), and 
        \item its kurtosis is finite such that \(\mathbb{E} \mathsfit{u}_i^4 = r_4 < \infty\).
    \end{enumerate*}
\end{assumption}

The location-scale family with \cref{assumption:noise} encompasses many variational families used in practice, such as Gaussians, Student-$t$ with a high-enough degrees of freedom $\nu$, Laplace, and so on, and enables the use of the reparametrization gradient.

While the choice of $\varphi$ gives control over the tail behavior of the family, the choice of the structure of the scale matrix $C$ gives control over how much correlation between coordinates of $\ell$ the variational approximation can represent.
This ability to represent correlations is often referred to as the ``expressiveness'' of a variational family, where the most expressive choice is the following:

\begin{definition}[Full-Rank Location-Scale Family]\label{def:fullrank}
    We say \(\mathcal{Q}\) is a full-rank location-scale family if it satisfies \cref{def:locscale} and, for any $C \in \mathcal{C}$, $C$ is invertible and the squared $C$s, \(CC^{\top}\), span the whole space of dense \(\mathbb{R}^{d \times d}\) positive definite matrices as \(\{ CC^{\top} \mid C \in \mathcal{C} \} = \mathcal{S}_{\succ 0}^d\).
\end{definition}

Typically, full-rank location-scale families are formed by setting \(\mathcal{C}\) to be the set of invertible triangular matrices (the ``Cholesky factor parametrization''; \citealp{titsias_doubly_2014,kucukelbir_automatic_2017}) or the set of symmetric square roots~\citep{domke_provable_2020,domke_provable_2023}.
Adding further restrictions on \(\mathcal{C}\) forms various subsets of the broader location-scale family.
In this work, we focus on the case where $C \in \mathcal{C}$ is restricted to be diagonal such that $\mathcal{C} \subset \mathbb{D}^d$, which is known as the mean-field approximation~\citep{peterson_explorations_1989,hinton_keeping_1993}:

\begin{definition}[Mean-Field Location-Scale Family]\label{def:meanfield}
    We say \(\mathcal{Q}\) is a mean-field location-scale family if it satisfies \cref{def:locscale} and all \(C \in \mathcal{C}\) are diagonal such that
    \(
         \mathcal{C} \subset \mathbb{D}^{d} \, .
    \)
\end{definition}

\vspace{-1ex}
\subsection{Algorithm Setup}
\vspace{-1ex}

Recall that BBVI is essentially SGD in the space of parameters of the variational distribution.
Therefore, we have to define the space of parameters.
For this, we use the ``linear'' parametrization:
{%
\setlength{\belowdisplayskip}{1ex} \setlength{\belowdisplayshortskip}{1ex}
\setlength{\abovedisplayskip}{1ex} \setlength{\abovedisplayshortskip}{1ex}
\[
    \Lambda = \left\{ \lambda = (m, C) \mid m \in \mathbb{R}^d, C \in \mathbb{D}^d_{\succ 0} \right\}
    \subset \mathbb{R}^p
     \; .
    \label{eq:linear_parametrization}
\]
}%
Under this parametrization, the desirable properties of $\ell$ easily transfer to $f$. 
For instance, if $\ell$ is $\mu$-strongly convex and $L$-smooth, $f$ is also $\mu$-strongly convex and $L$-smooth~\citep{domke_provable_2020}.
This contrasts with ``non-linear parametrizations'' commonly used in practice, such as making the diagonal positive by $C_{ii} = \exp\lt(\lambda_{C_{ii}}\rt)$.
Such practice rules
out transfer of strong convexity and smoothness~\citep{kim_convergence_2023} unless constraints such as $C_{ii} \geq \delta$ for some $\delta > 0$ are enforced~\citep{hotti_benefits_2024}.
(Though they can sometimes be beneficial by reducing gradient variance;~\citealp{kim_practical_2023,hotti_benefits_2024}.)
The flip side of using the linear parametrization is that we must now enforce the constraint $C \succ 0$.
Furthermore, $h$ then becomes non-smooth with respect to $C$:
{%
\setlength{\belowdisplayskip}{1ex} \setlength{\belowdisplayshortskip}{1ex}
\setlength{\abovedisplayskip}{1ex} \setlength{\abovedisplayshortskip}{1ex}
\[
    h\lt(\lambda\rt) 
    \quad=\quad 
    - \mathbb{H}\lt(q_{\lambda}\rt) 
    \quad=\quad 
    -\log \abs{\det C} - d \, \mathbb{H}\lt(\varphi\rt) 
    \quad=\quad 
    - {\textstyle\sum_{i=1}^d} \log \abs{ C_{ii} } - d \, \mathbb{H}\lt(\varphi\rt) 
    \; .
    \label{eq:entropy}
\]
}%
This corresponds to log-barrier functions~\citep[\S 6.7.5]{parikh_proximal_2014}, which are non-smooth.
Thus, the optimization algorithm must somehow deal with these difficulties~\citep{domke_provable_2020}.

In this work, we will rely on the proximal variant of stochastic gradient descent (SGD; \citealp{robbins_stochastic_1951,bottou_online_1999,bottou_optimization_2018,shalev-shwartz_pegasos_2011,nemirovski_robust_2009}), often referred to as stochastic proximal gradient descent (SPGD; \citealp{nemirovski_robust_2009}).
Proximal methods are a family of methods that rely on proximal operators~\citep{parikh_proximal_2014}, which are well defined as long as the following hold:
\begin{assumption}\label{assumption:hregular}
    $h : \Lambda \to \mathbb{R} \cup \{+\infty\}$ is convex, bounded below, and lower semi-continuous.
    \vspace{-1ex}
\end{assumption}
The non-smoothness of $h$ and the domain constraint are handled by the proximal operator
{%
\setlength{\belowdisplayskip}{1ex} \setlength{\belowdisplayshortskip}{1ex}
\setlength{\abovedisplayskip}{1ex} \setlength{\abovedisplayshortskip}{1ex}
\[
    \mathrm{prox}_{\gamma h}\lt(\lambda\rt)
    &\triangleq
    \argmin_{\lambda^{\prime} \in \Lambda} \Big\{ h\lt(\lambda^{\prime}\rt) +
    (1/\gamma) \norm{\lambda - \lambda^{\prime}}_2^2
    \Big\} \; ,
    \nonumber
\]
}%
while the intractability of $f$ is handled through stochastic estimates of $\nabla f$ (\cref{def:reparam_gradient}).
For a step size schedule ${(\gamma_t)}_{t \geq 0}$, $\widehat{\nabla f}$, an unbiased estimator of $\nabla f\lt(\lambda_t\rt) = \mathbb{E} \widehat{\nabla f}\lt(\lambda_t; \mathsfit{u}\rt)$, and a sequence of i.i.d. noise ${(\mathsfit{u}_t)}_{t \geq 0}$, for each $t \geq 0$, SPGD iterates
{%
\setlength{\belowdisplayskip}{1ex} \setlength{\belowdisplayshortskip}{1ex}
\setlength{\abovedisplayskip}{0ex} \setlength{\abovedisplayshortskip}{0ex}
\[
    \lambda_{t+1} = \mathrm{prox}_{\gamma_t h}\big( \lambda_t - \gamma_t \widehat{\nabla f}\lt(\lambda; \mathsfit{u}_t\rt) \big) \; .
    \nonumber
\]
}%
In the case of BBVI with a mean-field location-scale family, the proximal operator of \cref{eq:entropy} is identical to that of log-barrier functions~\citep[\S 6.7.5]{parikh_proximal_2014}:
{%
\setlength{\belowdisplayskip}{1ex} \setlength{\belowdisplayshortskip}{1ex}
\setlength{\abovedisplayskip}{1ex} \setlength{\abovedisplayshortskip}{1ex}
\[
    \mathrm{prox}_{\gamma h}\lt(\lambda = (m, C)\rt)
    =
    (m, C'), \quad \text{where} \quad C_{ii}' = (1/2) \Big( C_{ii} + \sqrt{C_{ii}^2 + 4 \gamma} \Big) \;  .
    \nonumber
\]
}%
Instead of using SPGD, one can also use projected SGD, where $C$ is projected to a subset where $F$ is smooth~\citep{domke_provable_2020} and use the ``closed-form entropy'' gradient $\widehat{\nabla F} \triangleq \widehat{\nabla f} + \nabla h$~\citep{titsias_doubly_2014,kucukelbir_automatic_2017}.
However, the resulting theoretical guarantees are indistinguishable~\citep{domke_provable_2023}, and the need for setting a closed domain of $C$ is inconvenient.
Therefore, we only consider SPGD.
But our results can easily be applied to projected SGD.

For $\widehat{\nabla f}$, we will use the classic \textit{reparametrization gradient}~\citep{ho_perturbation_1983,rubinstein_sensitivity_1992}:

\begin{definition}[Reparametrization Gradient]\label{def:reparam_gradient}
    For a differentiable function \(\ell : \mathbb{R}^{d} \to \mathbb{R}\),
{%
\setlength{\belowdisplayskip}{1ex} \setlength{\belowdisplayshortskip}{1ex}
\setlength{\abovedisplayskip}{1ex} \setlength{\abovedisplayshortskip}{1ex}
    \[
        \widehat{\nabla f}\lt(\lambda; \mathsfit{u}\rt) 
        \triangleq 
        \nabla_{\lambda} \ell\left( \mathcal{T}_{\lambda} \left(\mathsfit{u}\right) \right) 
        =
        \frac{ \partial \mathcal{T}_{\lambda}\lt(\mathsfit{u
        }\rt) }{ \partial \lambda } 
        \nabla \ell\left( \mathcal{T}_{\lambda} \left(\mathsfit{u}\right) \right) 
        \; , \quad\text{where}\quad \mathsfit{u} \sim \varphi \; ,
        \nonumber
    \]
}%
    is an unbiased estimator of \(\nabla f\) such that \( \nabla_{\lambda} \mathbb{E}_{\mathsfit{z} \sim q_{\lambda}} \ell\lt(\mathsfit{z}\rt) = \nabla f\lt(\lambda\rt) \).
\end{definition}
The reparametrization gradient, also known as the push-in gradient or pathwise gradient, was introduced to VI by \citet{titsias_doubly_2014,rezende_stochastic_2014,kingma_autoencoding_2014}.
(See also the reviews by \citealt{mohamed_monte_2020,glasserman_gradient_1991,pflug_optimization_1996}.)
It is empirically observed to outperform alternatives~\citep{kucukelbir_automatic_2017,mohamed_monte_2020} such as the score gradient~\citep{williams_simple_1992,glynn_likelihood_1990} and \textit{de facto} standard whenever $\ell$ is differentiable.
(Though theoretical evidence of this superiority is limited to the quadratic setting;~\citealp{xu_variance_2019}.)

\vspace{-1ex}
\subsection{General Analysis of Stochastic Proximal Gradient Descent}
\vspace{-1ex}
Analyzing the convergence of BBVI corresponds to analyzing the convergence of SPGD (or more broadly, of SGD) for the class of problems that corresponds to BBVI.
For this, we will first discuss sufficient conditions for the convergence of SPGD and the resulting consequences.

\begin{assumption}[Lipschitz Gradients in Expectation]\label{assumption:expected_smoothness}
There exists some constant $\mathcal{L} \in [0, \infty)$ such that, for all $\lambda, \lambda' \in \Lambda$, 
{%
\setlength{\belowdisplayskip}{0ex} \setlength{\belowdisplayshortskip}{0ex}
\setlength{\abovedisplayskip}{-1ex} \setlength{\abovedisplayshortskip}{-1ex}
\[
    \mathbb{E}\norm{ \widehat{\nabla f}\lt(\lambda; \mathsfit{u}\rt) - \widehat{\nabla f}\lt(\lambda'; \mathsfit{u}\rt) }_2^2 &\leq \mathcal{L}^2 \norm{ \lambda - \lambda' }_2^2 \; .
    \nonumber
\]
}%
\end{assumption}

\begin{assumption}[Bounded Variance]\label{assumption:bounded_variance}
There exists some constant $\sigma \in [0, \infty)$ such that, for all $\lambda_* \in \argmin_{\lambda \in \Lambda} F\lt(\lambda\rt)$, 
{%
\setlength{\belowdisplayskip}{0ex} \setlength{\belowdisplayshortskip}{0ex}
\setlength{\abovedisplayskip}{-1ex} \setlength{\abovedisplayshortskip}{-1ex}
\[
    \mathbb{E}\norm{ \widehat{\nabla f}\lt(\lambda_*; \mathsfit{u}\rt) }_2^2 \leq \sigma^2 \; .
    \nonumber
\]
}%
\end{assumption}
\vspace{-1ex}

Both assumptions were initially used by \citet[Assumptions H2 and H4]{bach_nonasymptotic_2011} to analyze the convergence of SGD.
Here,~\cref{assumption:expected_smoothness} serves as an analog of $L$-smoothness, and thus determines the largest stepsize we can use.
The strategy of combining~\cref{assumption:expected_smoothness,assumption:bounded_variance} is referred to as ``variance transfer''~\citep[\S 4.3.3]{garrigos_handbook_2023}.
Previously, for analyzing BBVI, a slightly different assumption called quadratically-bounded variance (QV)---which assumes the existence of $\alpha, \beta \in [0, +\infty)$ such that, for all $\lambda \in \Lambda$, $\mathbb{E}\norm{\widehat{\nabla f}\lt(\lambda; \mathsfit{u}\rt)}_2^2 \leq \alpha \norm{\lambda - \lambda_*}_2^2 + \beta$ holds---has been commonly used~\citep{domke_provable_2019,domke_provable_2023,kim_linear_2024}.
While similar, our assumptions result in a constant-factor improvement in the resulting bounds.

For the analysis, we will use a two-stage step size schedule~\citep[Theorem 3.2]{gower_sgd_2019}:
{%
\setlength{\belowdisplayskip}{0.5ex} \setlength{\belowdisplayshortskip}{0.5ex}
\setlength{\abovedisplayskip}{0.5ex} \setlength{\abovedisplayshortskip}{0.5ex}
    \[
        \gamma_t = \begin{cases} 
            \gamma_0 & \text{if $t \leq t_*$} \\
            \frac{1}{\mu} \frac{2 t + 1}{ {\lt(t + 1\rt)}^2 } & \text{if $t \geq t_* + 1$} \\
        \end{cases} \; ,
        \quad\text{where}\quad
        0 < \gamma_0 \leq \frac{\mu}{2 \mathcal{L}^2} \; 
        \label{eq:stepsize_schedule}
    \]
}%
This operates by first maintaining a fixed step size $\gamma_0$ until some switching time $t_* \in \{0, \ldots, T\}$, and then switches to the $1/t$ schedule of~\citet{lacoste-julien_simpler_2012}.

Under \cref{assumption:expected_smoothness,assumption:bounded_variance}, we can now provide a complexity guarantee for solving \cref{eq:objective} via SPGD.
Since BBVI consists of a subset of \cref{eq:objective}, establishing  \cref{assumption:expected_smoothness,assumption:bounded_variance} and invoking the following result will constitute our complexity guarantee for BBVI.

\begin{theoremEnd}[%
    restate,
    category=spgdcomplexity,
    text proof={},
    text link={\textit{Proof}. See the \hyperref[proof:prAtEnd\pratendcountercurrent]{\textit{full proof}} in~\cref{section:proof_spgd_complexity}, p.~\pageref{proof:prAtEnd\pratendcountercurrent}. \qed},
    text proof={Proof.}
]{proposition}\label{thm:spgd_complexity}
    Suppose $f$ is $\mu$-strongly convex, $h$ satisfies \cref{assumption:hregular}, and $\widehat{\nabla f}$ satisfies \cref{assumption:expected_smoothness,assumption:bounded_variance}.
    Then, for the global optimum $\lambda_* = \argmin_{\lambda \in \Lambda} F\lt(\lambda\rt)$ and $\Delta \triangleq \norm{\lambda_{0} - \lambda_*}_2$, there exists some $t_*$ and $\gamma_0$ such that SPGD with the step size schedule in \cref{eq:stepsize_schedule} guarantees
    \[
        T 
        \geq 
        \mathrm{O}\lt\{
        \frac{\sigma^2}{\mu^2}
        \frac{1}{\epsilon}
        + 
        \frac{\sigma \mathcal{L}}{\mu^2}
        \log\lt( \frac{\mathcal{L}^2}{\sigma^2} \Delta^2 \rt) 
        \frac{1}{\sqrt{\epsilon}}
        +
        \frac{\mathcal{L}^2}{\mu^2} \log\lt( \Delta^2 \frac{1}{\epsilon} \rt) + 1
        \rt\}
        \quad\Rightarrow\quad
        \mathbb{E}\norm{\lambda_T - \lambda_*}_2^2 \leq \epsilon
        \nonumber \; .
    \]
\end{theoremEnd}
\vspace{-1ex}
\begin{proofEnd}
    Since $f$ is strongly convex and $h$ is convex, $F$ is also strongly convex.
    This implies that, by the property of strictly convex functions, $F$ has a unique global optimum, which we denote as $\lambda_*$.
    
    From \cref{thm:spgd_convergence_bound}, we have
    \[
        \mathbb{E} \norm{\lambda_{T} - \lambda_*}_2^2
        &\leq
        \norm{\lambda_{0} - \lambda_*}_2^2 \,
        \rho^{t_*}
        \lt( \frac{{t_*}^2}{T^2} \rt)
        +
        2 \gamma_0
        \frac{\sigma^2}{\mu}
        \frac{{t_*}^2}{T^2}
        +
        \frac{8 \sigma^2}{\mu^2}
        \frac{T - t_*}{T^2}
        \; ,
        \label{eq:thm_spgd_complexity_eq1}
    \]
    where $\rho = 1 - \gamma_0 \mu$.
    We will optimize the upper bound over the parameters $t_*$, $\gamma_0$, and $T$ so that we can ensure the $\epsilon$-accuracy guarantee $\mathbb{E} \norm{\lambda_{T} - \lambda_*}_2^2 \leq \epsilon$.

    Consider the choice
    \[
        t_* = \min\lt\{ \lt\lceil \frac{1}{\log 1/\rho} \log\lt( \frac{\mu}{2 \gamma_0 \sigma^2} \norm{\lambda_0 - \lambda_*}_2^2 \rt)  \rt\rceil , T \rt\}
        \quad\text{and}\quad
        \gamma_0 = \frac{\mu}{2 \mathcal{L}^2} \; .
        \label{eq:schedule_parameters}
    \]
    Using this, we will separately analyze the total error in \cref{eq:thm_spgd_complexity_eq1} for the cases of $t_* = T$ and $t_* \neq T$.
    
    The case $t_* = T$ happens only if
    \[
        \lt\lceil \frac{1}{\log 1/\rho} \log\lt( \frac{\mu}{2 \gamma_0 \sigma^2} \norm{\lambda_0 - \lambda_*}_2^2 \rt)  \rt\rceil \geq T
        \nonumber
    \]
    is true.
    Then an immediate implication is that 
    \[
        &
        &\frac{1}{\log 1/\rho} \log\lt( \frac{\mu}{2 \gamma_0 \sigma^2} \norm{\lambda_0 - \lambda_*}_2^2 \rt) + 1
        &\geq T
        \nonumber
        \\
        &\Leftrightarrow
        &\log\lt( \frac{\mu}{2 \gamma_0 \sigma^2} \norm{\lambda_0 - \lambda_*}_2^2 \rt)
        &\geq \lt(\log 1/\rho\rt)  \lt( T - 1 \rt)
        \nonumber
        \\
        &\Leftrightarrow
        &\frac{\mu}{2 \gamma_0 \sigma^2} \norm{\lambda_0 - \lambda_*}_2^2
        &\geq \rho^{- (T - 1)}
        \nonumber
        \\
        &\Leftrightarrow
        &\norm{\lambda_0 - \lambda_*}_2^2 \rho^{T-1}
        &\geq 
        2 \gamma_0 
        \frac{\sigma^2}{\mu}
        \; .
        \label{eq:thm_spgd_complexity_case1_condition}
    \]
    Considering this fact, \cref{eq:thm_spgd_complexity_eq1} becomes
    \[
        \mathbb{E} \norm{\lambda_{T} - \lambda_*}_2^2
        &\leq
        \norm{\lambda_{0} - \lambda_*}_2^2
        \, \rho^{T}
        +
        2 \gamma_0
        \frac{\sigma^2}{\mu}
        &&\text{($t_* = T$)}
        \nonumber
        \\
        &\leq
        \norm{\lambda_{0} - \lambda_*}_2^2
        \, \rho^{T}
        +
        \norm{\lambda_0 - \lambda_*}_2^2 \rho^{T-1}
        &&\text{(\cref{eq:thm_spgd_complexity_case1_condition})}
        \nonumber
        \\
        &\leq
        2 \norm{\lambda_{0} - \lambda_*}_2^2 \rho^{T-1} \; .
        &&\text{($\rho < 1$)}
        \nonumber
    \]
    The number of required steps for achieving the $\epsilon$-accuracy requirement follows from
    \[
        &
        &
        2 \norm{\lambda_{0} - \lambda_*}_2^2 \rho^{T-1} &\leq \epsilon
        \nonumber
        \\
        &\Leftrightarrow
        &  2 \norm{\lambda_{0} - \lambda_*}_2^2 \frac{1}{\epsilon}  &\leq {(1/\rho)}^{T-1}
        \nonumber
        \\
        &\Leftrightarrow
        &  \log\lt( 2 \norm{\lambda_{0} - \lambda_*}_2^2 \frac{1}{\epsilon} \rt)  &\leq \lt(T - 1\rt) \log{(1/\rho)}
        \nonumber
        \\
        &\Leftrightarrow
        & \frac{1}{\log 1/\rho} \log\lt( 2 \norm{\lambda_{0} - \lambda_*}_2^2\frac{1}{\epsilon} \rt)  &\leq T - 1
        \nonumber
        \\
        &\Leftarrow
        &\frac{1}{1 - \rho} \log\lt( 2 \norm{\lambda_{0} - \lambda_*}_2^2 \frac{1}{\epsilon} \rt)  &\leq T - 1
        && \text{($\log (1/\rho) \geq 1 - \rho$)}
        \nonumber
        \\
        &\Leftrightarrow
        &\frac{2 \mathcal{L}^2}{\mu^2} \log\lt( 2 \norm{\lambda_{0} - \lambda_*}_2^2 \frac{1}{\epsilon} \rt) + 1 &\leq T
        &&\text{($1 - \rho = \gamma_0 \mu = \mu^2/(2 \mathcal{L}^2)$)}
        \label{eq:thm_spgd_complexity_case1}
    \]
    
    For the case $t_* \neq T$, 
    \[
        t_*
        &= 
        \lt\lceil \frac{1}{\log 1/\rho} \log\lt( \frac{\mu}{2 \gamma_0 \sigma^2} \norm{\lambda_0 - \lambda_*}_2^2 \rt) \rt\rceil
        \label{eq:thm_spgd_complexity_tneqtstar_constant}
        \\
        &\geq
        \frac{1}{\log 1/\rho} \log\lt( \frac{\mu}{2 \gamma_0 \sigma^2} \norm{\lambda_0 - \lambda_*}_2^2 \rt)
        \nonumber
        \\
        &=
        \frac{1}{\log \rho} \log\lt( \frac{2 \gamma_0 \sigma^2}{\mu} \frac{1}{\norm{\lambda_0 - \lambda_*}_2^2} \rt) \, .
        \nonumber
    \]
    This implies 
    \[
        \rho^{t_*}
        \leq
        \frac{2 \gamma_0 \sigma^2}{\mu} \frac{1}{\norm{\lambda_0 - \lambda_*}_2^2} \; .
        \nonumber
    \]
    Substituting for this in \cref{eq:thm_spgd_complexity_eq1},
    \[
        \mathbb{E} \norm{\lambda_{T} - \lambda_*}_2^2
        &\leq
        2 \gamma_0 
        \frac{\sigma^2}{\mu} 
        \frac{{t_*}^2}{T^2}
        +
        2 \gamma_0
        \frac{\sigma^2}{\mu}
        \frac{{t_*}^2}{T^2}
        +
        8
        \frac{\sigma^2}{\mu^2}
        \frac{T - t_*}{T^2}
        \nonumber
        \\
        &=
        4 \gamma_0 
        \frac{\sigma^2}{\mu} 
        \frac{{t_*}^2}{T^2}
        +
        8
        \frac{\sigma^2}{\mu^2} 
        \frac{T - t_*}{T^2}
        \nonumber
        \\
        &\leq
        4 \gamma_0 
        \frac{\sigma^2}{\mu} 
        \frac{{t_*}^2}{T^2}
        +
        8
        \frac{\sigma^2}{\mu^2} 
        \frac{1}{T} 
        \nonumber
        \\
        &=
        a
        \frac{1}{T^2}
        +
        b
        \frac{1}{T} 
        \nonumber
        \; ,
        \nonumber
    \]
    which is a quadratic function of $1/T$ with the coefficients
    \[
        a \triangleq
        4 \gamma_0 
        \frac{\sigma^2}{\mu} 
        {t_*}^2
        \qquad\text{and}\qquad
        b 
        \triangleq
        8 \frac{\sigma^2}{\mu^2} \; .
        \nonumber
    \]
    
    Achieving the $\epsilon$-accuracy guarantee is equivalent to finding the largest $x = 1/T$ satisfying the inequalities $x > 0$ and
    \[
        a x^2 + bx \leq \epsilon \; .
        \nonumber
    \]
    By the quadratic formula, this is equivalent to finding the largest $x$ satisfying
    \[
        0 \leq x \leq \frac{- b + \sqrt{b^2 + 4 a \epsilon} }{2 a} \; .
        \nonumber
    \]
    Therefore, picking any
    \[
        T \geq \frac{2 a}{- b + \sqrt{b^2 + 4 a \epsilon}}
        \nonumber
    \]
    is sufficient to obtain an $\epsilon$-accurate solution.
    To make the bound more interpretable, after defining $\alpha = 4 a \epsilon$ and $\beta = b$, we can use the inequality~\citep{symbol-1_answer_2022}
    \[
        \frac{\alpha}{2 \sqrt{\beta^2 + \alpha }}
        \leq
        -\beta + \sqrt{\beta^2 + \alpha} \; .
        \nonumber
    \]
    Then
    \[
        \frac{2 a}{- b + \sqrt{b^2 + 4 a \epsilon}}
        \quad&\leq\quad
        2 a \,
        \frac{2 \sqrt{b^2 + 4 a \epsilon } }{ 4 a \epsilon }
        \quad=\quad
        \sqrt{b^2 + 4 a \epsilon } \,
        \frac{1}{\epsilon}
        \quad\leq\quad
        b
        \frac{1}{\epsilon}
        + 2 \sqrt{a} 
        \frac{1}{\sqrt{\epsilon}} \; ,
        \nonumber
    \]
    where we used the inequality $\sqrt{a + b} \leq \sqrt{a} + \sqrt{b}$.
    Thus, we have
    \[
        T  
        \geq
        b
        \frac{1}{\epsilon}
        + 2 \sqrt{a} 
        \frac{1}{\sqrt{\epsilon}}
        \qquad\Rightarrow\qquad
        \mathbb{E}\norm{\lambda_T - \lambda_*}_2^2  \leq \epsilon \; .
        \nonumber
    \]
    Substituting $t_*$ and $\gamma_0$ with the expressions in \cref{eq:schedule_parameters},
    \[
        & &
        T 
        &\geq
        8 \frac{\sigma^2}{\mu^2}
        \frac{1}{\epsilon}
        + 
        2 
        \sqrt{
        4 \gamma_0 
        \frac{\sigma^2}{\mu} 
        {t_*}^2
        } 
        \frac{1}{\sqrt{\epsilon}}
        \quad=\quad
        8 \frac{\sigma^2}{\mu^2}
        \frac{1}{\epsilon}
        + 
        4
        \sqrt{\gamma_0}
        \frac{\sigma}{\mu^{1/2}} 
        {t_*}
        \frac{1}{\sqrt{\epsilon}}
        \nonumber
        \\
        &\Leftarrow&
        T &\geq
        8 \frac{\sigma^2}{\mu^2}
        \frac{1}{\epsilon}
        + 
        4
        \sqrt{\gamma_0}
        \frac{\sigma}{\mu^{1/2}} 
        \lt(
          \frac{1}{\log 1/\rho} \log\lt( \frac{\mu}{2 \gamma_0 \sigma^2} \norm{\lambda_0 - \lambda_*}_2^2 \rt) + 1
        \rt)
        \frac{1}{\sqrt{\epsilon}} 
        &&\text{(\cref{eq:thm_spgd_complexity_tneqtstar_constant})}
        \nonumber
        \\
        &\Leftarrow&
        T &\geq
        8 \frac{\sigma^2}{\mu^2}
        \frac{1}{\epsilon}
        + 
        4
        \sqrt{ \frac{\mu}{2 \mathcal{L}^2} }
        \frac{\sigma}{\mu^{1/2}} 
        \lt(
          \frac{2 \mathcal{L}^2 }{\mu^2}
          \log\lt( \frac{\mu}{2 \sigma^2} \frac{2 \mathcal{L}^2}{\mu} \norm{\lambda_0 - \lambda_*}_2^2 \rt) + 1
        \rt)
        \frac{1}{\sqrt{\epsilon}} 
        \nonumber
        &&\text{($\log 1/\rho \geq 1 - \rho = \mu^2/(2 \mathcal{L}^2)$)}
        \\
        & &
        &=
        8 \frac{\sigma^2}{\mu^2}
        \frac{1}{\epsilon}
        + 
        2 \sqrt{2}
        \frac{\sigma}{\mathcal{L}}
        \lt(
          \frac{2 \mathcal{L}^2 }{\mu^2}
          \log\lt( \frac{\mathcal{L}^2}{\sigma^2} \norm{\lambda_0 - \lambda_*}_2^2 \rt) + 1
        \rt)
        \frac{1}{\sqrt{\epsilon}} 
        \nonumber
    \]
    Now, \cref{thm:smoothness_implication} asserts that $\mathcal{L} \geq \mu$.
    This allows us to further simplify the term
    \[
        2 \sqrt{2}
        \frac{\sigma}{\mathcal{L}}
        \lt(
          \frac{2 \mathcal{L}^2 }{\mu^2}
          \log\lt( \frac{\mathcal{L}^2}{\sigma^2} \norm{\lambda_0 - \lambda_*}_2^2 \rt) + 1
        \rt)
        &\leq
        2 \sqrt{2}
        \frac{\sigma}{\mathcal{L}}
        \lt(
          \frac{2 \mathcal{L}^2 }{\mu^2}
          \log\lt( \frac{\mathcal{L}^2}{\sigma^2} \norm{\lambda_0 - \lambda_*}_2^2 \rt) + \frac{2 \mathcal{L}^2}{\mu^2} \log 3
        \rt)
        \nonumber
        \\
        &=
        2 \sqrt{2}
        \frac{\sigma}{\mathcal{L}}
        \frac{2 \mathcal{L}^2 }{\mu^2}
        \log\lt( 3 \frac{\mathcal{L}^2}{\sigma^2} \norm{\lambda_0 - \lambda_*}_2^2 \rt) 
        \nonumber
        \\
        &=
        4 \sqrt{2}
        \frac{\sigma \mathcal{L}}{\mu^{2}}
        \log\lt( 3 \frac{\mathcal{L}^2}{\sigma^2} \norm{\lambda_0 - \lambda_*}_2^2 \rt)  \; .
        \nonumber
    \]
    Considering this, the sufficient condition for $\mathbb{E}\norm{\lambda_T - \lambda_*}_2^2 \leq \epsilon$ is now
    \[
        & &
        T
        &\geq
        \frac{8 \sigma^2}{\mu^2}
        \frac{1}{\epsilon}
        + 
        4 \sqrt{2}
        \frac{\sigma \mathcal{L}}{\mu^2}
        \log\lt( 3 \frac{\mathcal{L}^2}{\sigma^2} \norm{\lambda_0 - \lambda_*}_2^2 \rt) 
        \frac{1}{\sqrt{\epsilon}} \; .
        \label{eq:thm_spgd_complexity_case2}
    \]
    
    Combining both cases, that is, \cref{eq:thm_spgd_complexity_case1,eq:thm_spgd_complexity_case2}, we have
    \[
        & &
        T 
        &\geq 
        \max\lt(
        \frac{8 \sigma^2}{\mu^2}
        \frac{1}{\epsilon}
        + 
        4 \sqrt{2}
        \frac{\sigma \mathcal{L}}{\mu^2}
        \log\lt( 3 \frac{\mathcal{L}^2}{\sigma^2} \norm{\lambda_0 - \lambda_*}_2^2 \rt) 
        \frac{1}{\sqrt{\epsilon}}
        , \; 
        \frac{2 \mathcal{L}^2}{\mu^2} \log\lt( 2 \norm{\lambda_{0} - \lambda_*}_2^2 \frac{1}{\epsilon} \rt) + 1
        \rt) 
        \; .
        \nonumber
    \]
    This implies the stated result. 
\end{proofEnd}

This result is a slight improvement over past analysis of SPGD with \cref{eq:stepsize_schedule}~\citep[Theorem 7]{domke_provable_2023}.
In particular, the dependence on the initialization $\Delta$ has been improved to be logarithmic instead of polynomial.
Furthermore, it encompasses the case where we have ``interpolation'' ($\sigma^2 = 0$; \citealp{schmidt_fast_2013,vaswani_fast_2019,kim_linear_2024}) automatically resulting in a $\mathrm{O}(\log 1/\epsilon)$ complexity.
The key difference in the analysis is that we choose a different switching time $t_*$ in a way adaptive to $\sigma^2$ and $\Delta$, ensuring that the dependence on both is optimized.

For a non-strongly convex $f$, using the strategy of~\citet[Theorem 8 and 11]{domke_provable_2023} should yield a corresponding $\mathrm{O}\lt(1/\epsilon^2\rt)$ complexity guarantee under the same set of assumptions.
However, this requires fixing the horizon $T$ in advance, and it is currently unknown how to obtain an anytime $\mathrm{O}(1/\sqrt{T})$ convergence bound for SGD under \cref{assumption:expected_smoothness,assumption:bounded_variance} or QV.
If one moves away from the canonical SGD update by incorporating  Halpern iterations~\citep{halpern_fixed_1967}, it is possible to obtain any-time convergence under a QV-like assumption~\citep{alacaoglu_weaker_2025}.

\vspace{-1ex}
\section{Main Results}\label{section:main_results}
\vspace{-1ex}
\subsection{General Result}\label{section:general_result}
\vspace{-1ex}

For our results, we impose an additional assumption that is a generalization of $L$-smoothness under twice differentiability of $\ell$.
\begin{assumption}\label{assumption:almost_constant_hessian}
    \(\ell\) is twice differentiable and, for all \(z \in \mathbb{R}^d\), there exist some matrix \(H \in \mathbb{R}^{d \times d}\) and constant \(\delta \in [0, \infty)\) satisfying
    {%
    \setlength{\belowdisplayskip}{0ex} \setlength{\belowdisplayshortskip}{0ex}
    \setlength{\abovedisplayskip}{1ex} \setlength{\abovedisplayshortskip}{1ex}
    \[
        \norm{H}_2 < \infty
        \quad\text{and}\quad
        \norm{ \nabla^2 \ell\lt(z\rt) - H}_2 \leq \delta  \; .
        \nonumber
    \]
    }
\end{assumption}
Notably, if $\ell$ is twice differentiable, $\mu$-strongly convex, $L$-smooth, it already satisfies \cref{assumption:almost_constant_hessian} with \(H = \frac{L + \mu}{2} \mathrm{I}_d\) and \(\delta = \frac{L - \mu}{2}\).
If $\ell$ is only $L$-smooth, it satisfies it with \(H = \mathrm{0}_{d \times d}\) and \(\delta = L\).
The key advantage of this assumption, however, is that it characterizes Hessians that are not necessarily well-conditioned, but almost constant. 
This crucially affects the dimension dependence.


Given our assumptions on the target function $\ell$, variational family $\mathcal{Q}$, and our choice of gradient estimator, we can guarantee that SPGD applied to a problem structure corresponding to BBVI (\cref{eq:objective}) achieves a given level of accuracy $\epsilon$ after $\mathrm{O}( g\lt(d, H, \delta, \mu, \varphi\rt) \epsilon^{-1} )$ number of iterations:

\begin{theoremEnd}[%
    restate,
    category=bbvicomplexity, 
    text proof={},
    text link={\textit{Proof.} The \hyperref[proof:prAtEnd\pratendcountercurrent]{\textit{full proof}} can be found in~\cref{section:proof_bbvi_complexity}, p.~\pageref{proof:prAtEnd\pratendcountercurrent}. \qed}
]{theorem}\label{thm:bbvi_complexity}
    Suppose the following hold:
    \begin{enumerate}[topsep=0ex,itemsep=0.5ex,partopsep=0ex,parsep=0ex,leftmargin=2em]
        \vspace{-1ex}
        \item $\ell$ is $\mu$-strongly convex and satisfies \cref{assumption:almost_constant_hessian} and $\mu \leq  \sigma_{\mathrm{min}}\lt(H\rt) \leq \sigma_{\mathrm{max}}\lt(H\rt) \leq L$.
        \item $h$ satisfies \cref{assumption:hregular}.
        \item $\mathcal{Q}$ is a mean-field location-scale family, where \cref{assumption:noise} holds.
        \item $\widehat{\nabla f}$ is the reparametrization gradient.
        \vspace{-.5ex}
    \end{enumerate}
    Denote the global optimum $\lambda_* = (m_*, C_*) = \argmin_{\lambda \in \Lambda} F\lt(\lambda\rt)$, the irreducible gradient noise as $\sigma_*^2 \triangleq \norm{m_* - \bar{z}}_2^2 + \norm{C_*}_{\mathrm{F}}^2$, and the stationary point of $\ell$ as $\bar{z} \triangleq \argmin_{z \in \mathbb{R}^d} \ell\lt(z\rt)$.
    Then there exists some $t_*$ and $\gamma_0$ such that SPGD with the step size schedule in \cref{eq:stepsize_schedule} guarantees
    {%
    \setlength{\belowdisplayskip}{0ex} \setlength{\belowdisplayshortskip}{0ex}
    \setlength{\abovedisplayskip}{1ex} \setlength{\abovedisplayshortskip}{1ex}
    \[
        T
        &\geq
        \mathrm{O}\lt\{
        g\lt(d, H, \delta, \mu, \varphi\rt)
        \lt(
        \sigma_*^2 \epsilon^{-1}
        +
        \sigma_* \log\lt(\norm{ \lambda_0 - \lambda_* }^2_2\rt) \epsilon^{-1/2}
        \rt)
        \rt\}
        \quad\Rightarrow\quad
        \mathbb{E}\norm{\lambda_T - \lambda_*}_2^2 \leq \epsilon \; ,
        \nonumber
    \]
    }%
    where 
    {%
    \setlength{\belowdisplayskip}{-1ex} \setlength{\belowdisplayshortskip}{-1ex}
    \setlength{\abovedisplayskip}{0ex} \setlength{\abovedisplayshortskip}{0ex}
    \[
        g\lt(d, H, \delta, \mu, \varphi\rt)
        \triangleq
        2 \, \lt(1 + r_4\rt) \lt( \norm{H}_2^2/\mu^2 \rt)
        +
        4 \lt( \delta^2 / \mu^2 \rt)
        \Big(
        ({1}/{2})
        +
        r_4 
        +
        \mathbb{E}\max_{j = 1, \ldots, d} \mathsfit{u}_j^2
        \Big)
        \; .
        \nonumber
    \]
    }%
\end{theoremEnd}
\vspace{-2ex}
\begin{proofEnd}
    The proof consists of establishing the sufficient conditions of \cref{thm:spgd_complexity} as follows:
    \begin{enumerate}[label=(\roman*),topsep=1ex,itemsep=1ex,partopsep=1ex,parsep=1ex]
        \item $\text{$\ell$ is $\mu$-strongly convex} \quad\Rightarrow\quad \text{$f$ is $\mu$-strongly convex} $.\label{item:ellstrongconvexity_implies_fstrongconvexity}
    
        \item \cref{assumption:almost_constant_hessian} $\quad\Rightarrow\quad$ \cref{assumption:expected_smoothness,assumption:bounded_variance}.
        \label{item:gradient_variance_conditions}
    \end{enumerate}
    Under the linear parametrization, \labelcref{item:ellstrongconvexity_implies_fstrongconvexity} was established by \citet[Thm. 9]{domke_provable_2020}.
    It remains to establish \labelcref{item:gradient_variance_conditions}.
    Therefore, the proof focuses on analyzing the variance of the gradient estimator $\widehat{\nabla f}$.

    Since \cref{assumption:almost_constant_hessian} holds, \cref{thm:gradient_variance_upperbound_meanfield_general} states that, for all $\lambda, \lambda' \in \mathbb{R}^d \times \mathbb{D}^{d}$, the inequality 
    \[
        \mathbb{E}\norm{ \widehat{\nabla f}\lt(\lambda; \mathsfit{u}\rt) - \widehat{\nabla f}\lt(\lambda'; \mathsfit{u}\rt) }_2^2
        \leq
        \Big\{ 
            2 \lt(1 + r_4\rt) \norm{H}_2^2
            +
            4 \delta^2
            \Big(
            \nicefrac{1}{2}
            +
            r_4 
            +
            \mathbb{E}\max_{j = 1, \ldots, d} \mathsfit{u}_j^2
            \Big)
        \Big\}
        \norm{\lambda - \lambda'}_2^2
        \nonumber
    \]
    holds.
    Since $\Lambda \subset \mathbb{R}^d \times \mathbb{D}^{d}$ under the linear parametrization, this implies we satisfy \cref{assumption:expected_smoothness} with 
    \[
        \mathcal{L}^2 
        =
        2 \lt(1 + r_4\rt) \norm{H}_2^2 
        +
        4 \delta^2
        \Big(
        \nicefrac{1}{2}
        +
        r_4 
        +
        \mathbb{E}\max_{j = 1, \ldots, d} \mathsfit{u}_j^2
        \Big)
        \label{eq:expected_smoothness_constant}
        \; .
    \]
    Furthermore, For the specific choice of $\lambda_* = (m_*, C_*) = \argmin_{\lambda \in \Lambda} F\lt(\lambda\rt)$ and $\bar{\lambda} = (\bar{z}, 0_{d \times d})$ (which is not part of $\Lambda$), we have the equality
    \[
        \mathbb{E}
        \norm{ 
            \widehat{\nabla f}\lt(\lambda_*; \mathsfit{u}\rt) - \widehat{\nabla f}\lt(\bar{\lambda}; \mathsfit{u}\rt)
        }_2^2
        =
        \mathbb{E}
        \norm{ 
            \widehat{\nabla f}\lt(\lambda_*; \mathsfit{u}\rt) - \widehat{\nabla f}\lt(\bar{z}; \mathsfit{u}\rt)
        }_2^2
        =
        \mathbb{E}
        \norm{ 
            \widehat{\nabla f}\lt(\lambda_*; \mathsfit{u}\rt)
        }_2^2 
        \; .
        \nonumber
    \]
    This means \cref{thm:gradient_variance_upperbound_meanfield_general} also implies \cref{assumption:bounded_variance} with the constant
    \[
        \sigma^2 
        \quad=\quad 
        \mathcal{L}^2 \norm{\lambda_* - \bar{\lambda}}_2^2
        \quad=\quad  
        \mathcal{L}^2 \lt( \norm{m_* - \bar{z}} + \norm{C_*}_{\mathrm{F}}^2 \rt)
        \quad=\quad  
        \mathcal{L}^2 \sigma_*^2
        \; .
        \label{eq:variance_on_optimum_constant}
    \]

    We are now able to invoke \cref{thm:spgd_complexity}.
    Substituting $\mathcal{L}$ and $\sigma^2$ in \cref{thm:spgd_convergence_precise} with the expressions above, we obtain the condition
    \[
        & &
        T 
        &\geq 
        \max\lt(
        \frac{8 \sigma_*^2 \mathcal{L}^2 }{\mu^2}
        \frac{1}{\epsilon}
        + 
        4 \sqrt{2}
        \frac{\sigma_* \mathcal{L}^2 }{\mu^2}
        \log\lt( \frac{3}{\sigma_*^2} \norm{\lambda_0 - \lambda_*}_2^2 \rt) 
        \frac{1}{\sqrt{\epsilon}}
        , \; 
        \frac{2 \mathcal{L}^2}{\mu^2} \log\lt( 2 \norm{\lambda_{0} - \lambda_*}_2^2 \frac{1}{\epsilon} \rt) + 1
        \rt)  \; .
        \nonumber
    \]
    Using the fact $\mathcal{L} \geq \mu$ from \cref{thm:smoothness_implication}, we finally have
    \[
        T
        &\geq
        \frac{ \mathcal{L}^2 }{\mu^2}
        \max\lt(
        8 \sigma_*^2
        \frac{1}{\epsilon}
        + 
        4 \sqrt{2}
        \sigma_* 
        \log\lt( \frac{3}{\sigma_*^2} \norm{\lambda_0 - \lambda_*}_2^2 \rt) 
        \frac{1}{\sqrt{\epsilon}}
        , \; 
        2 \log\lt( 2 \norm{\lambda_{0} - \lambda_*}_2^2 \frac{1}{\epsilon} \rt) + 1
        \rt) \; .
        \nonumber
    \]
    Finally, substituting for \cref{eq:expected_smoothness_constant} yields our stated result.
\end{proofEnd}

Due to the identity 
$
        \norm{\lambda - \lambda'}_2^2 =  \mathbb{E}_{\mathsfit{u} \sim \varphi^{\otimes d}} \norm{ \mathcal{T}_{\lambda}\lt(\mathsfit{u}\rt) - \mathcal{T}_{\lambda'}\lt(\mathsfit{u}\rt) }_2^2
$
(\cref{thm:reparam_identity}), which is the squared cost of a coupling between $q_{\lambda_T}$ and $q_{\lambda_*}$, our guarantee also translates to a guarantee in Wasserstein-2 distance:
$
    \mathbb{E}\norm{\lambda_T - \lambda_*}_2^2 \leq \epsilon
    \,\Rightarrow\,
    \mathbb{E} {\mathrm{W}_2\lt(q_{\lambda_T}, q_{\lambda_*}\rt)}^2 \leq \epsilon
$.
In the general case where $\delta > 0$, the dimension dependence enters through $\mathbb{E} \max_{j = 1, \ldots, d} \mathsfit{u}_{j}^2 $, which depends on the order-statistics of the base distribution $\varphi$.
In case $\ell$ is a quadratic, corresponding to $\pi$ being a Gaussian target distribution in the BBVI context, there exists some $H$ such that $\nabla^2\ell\lt(z\rt) = H$ for all $z \in \mathbb{R}^d$.
Thus, \cref{assumption:almost_constant_hessian} holds with $\delta = 0$, implying a dimension-independent convergence rate.
We will present additional special cases with more explicit choices of $\varphi$ in the next section.

In case we do not want to assume \cref{assumption:almost_constant_hessian} and only assume that $\ell$ is $\mu$-strongly convex and $L$-smooth instead, we can replace them with the generic choices of $H = \frac{L + \mu}{2} \mathrm{I}_d$ and $\delta = \frac{L - \mu}{2}$, which hold for all $\ell$s that are $\mu$-strongly convex, $L$-smooth, and twice differentiable.
This then makes the role of the condition number $\kappa \triangleq L/\mu$ more explicit.
\begin{corollary}
    Suppose $\ell$ is is twice differentiable, $\mu$-strongly convex, and $L$-smooth.
    Then, denoting the condition number as $\kappa \triangleq L/\mu$, \cref{thm:bbvi_complexity} holds with 
{%
\setlength{\belowdisplayskip}{0ex} \setlength{\belowdisplayshortskip}{0ex}
\setlength{\abovedisplayskip}{0ex} \setlength{\abovedisplayshortskip}{0ex}
    \[
        {\textstyle g\lt(d, \frac{L + \mu}{2} \mathrm{I}_d, \frac{L - \mu}{2}, \mu, \varphi\rt)} = (1/2) (1 + r_4) {(\kappa + 1)}^2 + {(\kappa - 1)}^2 \Big((1/2) + r_4 + \mathbb{E}\max_{j = 1, \ldots, d} \mathsfit{u}_j^2 \Big) \; .
        \nonumber
    \]
}%
\end{corollary}
\vspace{-1ex}
This makes the $\mathrm{O}(\kappa^2)$ condition number dependence explicit, but the downside is that we lose dimension independence in the case of ill-conditioned quadratic $\ell$s.
This fact suggests that dimension dependence is more fundamentally related to how close the Hessian is to a constant rather than how well-conditioned it is.

\vspace{-1ex}
\subsection{Special Cases with Benign Dimension Dependence}\label{section:special_cases}
\vspace{-1ex}

We now present some special cases of \cref{thm:bbvi_complexity}, which has yet to exhibit an explicit dependence on dimensionality.
As mentioned in the previous section, dimension dependence depends on the order statistics of $\varphi$, which is related to the tail behavior of $\varphi$.

\vspace{-0.5ex}
\paragraph{Variational Families with Sub-Gaussian Tails.}
The most commonly used variational family in practice is the Gaussian variational family.
More broadly, for sub-Gaussian variational families, \(\mathsfit{u}_i^2\) is sub-exponential and therefore admits a moment generating function (MGF)~\citep[Theorem 2.6]{wainwright_highdimensional_2019}, which leads to a $\mathrm{O}\lt(\log d\rt)$ explicit dimension dependence.

\begin{theoremEnd}[%
    restate,
    category=specialcasegaussian,
    text link={\textit{Proof.} The \hyperref[proof:prAtEnd\pratendcountercurrent]{\textit{full proof}} can be found in~\cref{section:proof_special_case_gaussian}, p.~\pageref{proof:prAtEnd\pratendcountercurrent}. \qed
    },
    text proof={Proof.}
]{proposition}
\label{thm:special_case_gaussian}
    Suppose there exists some $t > 0$ such that the MGF of $\mathsfit{u}_i^2$ satisfies $M_{\mathsfit{u}_i^2}\lt(t\rt) < \infty$.
    Then
{%
\setlength{\belowdisplayskip}{0ex} \setlength{\belowdisplayshortskip}{0ex}
\setlength{\abovedisplayskip}{0ex} \setlength{\abovedisplayshortskip}{0ex}
    \[
        \mathbb{E} \max_{i=1,\ldots, d}\mathsfit{u}_i^2 \quad\leq\quad \left({1}/{t}\right) \big( \log M_{\mathsfit{u}_i^2}\lt(t\rt) + \log d \big) \; .
        \nonumber
    \]
}%
    For example, if \(\varphi\) is a standard Gaussian, then
{%
\setlength{\belowdisplayskip}{1ex} \setlength{\belowdisplayshortskip}{1ex}
\setlength{\abovedisplayskip}{1ex} \setlength{\abovedisplayshortskip}{1ex}
    \[
        g\lt(d, H, \delta, \mu, \varphi\rt)
        \quad\leq\quad
        8 \lt(\norm{H}_2^2/\mu^2\rt)
        +
        \lt( \delta^2/\mu^2 \rt)
        \lt(
        22
        + 
        16 \log d
        \rt)
        \; .
        \nonumber
    \]
}%
\end{theoremEnd}
\vspace{-1ex}
\begin{proofEnd}
    The first part of the statement is a re-statement of \cref{thm:expected_maximum_mgf}.
    
    For the special case of \(\mathsfit{u}_i \sim \mathcal{N}\lt(0, 1\rt)\), we know that \(\mathsfit{u}_i^2 \sim \chi_1^2\)~\citep[Eq. 29.1]{johnson_continuous_1995}, which is the \(\chi^2\) distribution with 1 degree of freedom.
    The MGF of \(\chi^2_1\) is given as 
    \[
        M_{\mathsfit{u}_i^2}\lt(t\rt) = {\lt(1 - 2 t\rt)}^{-1/2}
        \qquad\text{\citep[Eq. 29.6]{johnson_continuous_1995}}
        \nonumber
    \]
    for \( t \in (0, 1/2)\).
    Then we can invoke \cref{thm:expected_maximum_mgf}, which suggests
    \[
        \mathbb{E} \max_{i = 1, \ldots, d} \mathsfit{u}_i^2
        &\leq 
        \min_{t \in (0, 1/2)} \frac{1}{t} \lt( -\frac{1}{2} \log {\lt(1 - 2 t\rt)} + \log d \rt) \; .
        \nonumber
    \]
    Any fixed choice of \(t \in (0, 1/2)\) is a valid upper bound.
    Picking \(t = \frac{1}{2} \lt(1 - \frac{1}{\mathrm{e}}\rt) \geq \frac{1}{4}\) yields 
    \[
        \mathbb{E} \max_{i = 1, \ldots, d} \mathsfit{u}_i^2
        &\leq 
        4 \lt( \frac{1}{2} + \log d \rt) \; .
        \label{eq:gaussian_expected_maximum_bound}
    \]
    Furthermore, the kurtosis of the standard Gaussian is \(r_4 = 3\)~\citep[Eq. 13.11]{johnson_continuous_1994}.
    Plugging $r_4$ and \cref{eq:gaussian_expected_maximum_bound} into $g$ in \cref{thm:bbvi_complexity} yields the statement.
\end{proofEnd}

\vspace{-0.5ex}
\paragraph{Variational Families with Finite Higher Moments.}
For families with tails heavier than sub-Gaussian, however, $\mathsfit{u}_i^2$ may not have an MGF.
While we then lose the $\mathrm{O}(\log d)$ dependence, we may still obtain a polynomial dependence that can be better than $\mathrm{O}(\sqrt{d})$ obtained in previous works~\citep{kim_practical_2023}.
In particular, the result that will follow states that the highest order of the available moments determines the order of dimension dependence.
For Student-$t$ families, this implies that using a high-enough degree of freedom $\nu$ can make the dimension dependence benign.

\begin{theoremEnd}[%
    restate,
    category=specialcasestudentt,
    text link={\textit{Proof}. See the \hyperref[proof:prAtEnd\pratendcountercurrent]{\textit{full proof}} in~\cref{section:proof_special_case_studentt}, p.~\pageref{proof:prAtEnd\pratendcountercurrent}. \qed},
    text proof={Proof.}
]{proposition}\label{thm:special_case_studentt}
    Suppose, for $k \geq 2$, the $k$th moment of $\mathsfit{u}_i^2$ is finite as $r_{2k} = \mathbb{E}\mathsfit{u}_i^{2 k} < \infty$.
    Then
{%
\setlength{\belowdisplayskip}{0ex} \setlength{\belowdisplayshortskip}{0ex}
\setlength{\abovedisplayskip}{0ex} \setlength{\abovedisplayshortskip}{0ex}
    \[
        \mathbb{E} \max_{i=1, \ldots, d} \mathsfit{u}_i^{2} \quad\leq\quad \sqrt{2} \, d^{1/k} \,  r_{2 k}^{1/k}
        \nonumber
        \; .
    \]
}%
    For example, if $\varphi$ is a Student-$t$ with $\nu > 4$ degrees of freedom and unit variance, then
{%
\setlength{\belowdisplayskip}{0ex} \setlength{\belowdisplayshortskip}{0ex}
\setlength{\abovedisplayskip}{0ex} \setlength{\abovedisplayshortskip}{0ex}
    \[
        g\lt(d, H, \delta, \mu, \varphi\rt)
        \quad\leq\quad
        8 (\norm{H}_2^2/\mu^2)
        +
        (\delta^2/\mu^2)
        \lt(
        16
        +
        \sqrt{2}
        \,
        \nu^3
        d^{\frac{2}{\nu - 2}}
        \rt)
        \; .
        \nonumber
    \]
}%
\end{theoremEnd}
\vspace{-2ex}
\begin{proofEnd}
    The first part of the statement directly follows from \cref{thm:expected_maximum_highestmoment}, where we simplified $\lt(\nicefrac{k}{k - 1}\rt)^{(k-1)/k}$. 
    In particular, for \( k \geq 2\), $\lt(\nicefrac{k}{k - 1}\rt)^{(k-1)/k}$ is monotonically decreasing.
    Since an order \(k \geq 2\) moment exists by the assumption on the degrees of freedom, $\lt(\nicefrac{k}{k - 1}\rt)^{(k-1)/k} \leq \sqrt{2}$.

    Let's turn to the second part of the statement. 
    We will denote a Student-$t$ distribution with $\nu$-degrees of freedom as $t_{\nu}$.
    Since $t_{\nu}$ does not have unit variance~\citep[Eq. 28.7a]{johnson_continuous_1995}, we have to set the sampling process from $\varphi$ to be 
    \[
        \mathsfit{u}_i \sim \varphi
        \qquad\Leftrightarrow\qquad
        \mathsfit{u}_i \stackrel{\mathrm{d}}{=} \frac{\nu - 2}{\nu} \mathsfit{v}_i \; ,
        \quad\text{where}\quad
        \mathsfit{v}_i \stackrel{\text{i.i.d.}}{\sim} t_{\nu} \; .
        \nonumber
    \]
    Now, it is known that $\mathsfit{v}_i^2 \stackrel{\mathrm{d}}{=} \mathsfit{w}_i \sim \mathrm{FDist}\lt(1, \nu_2\rt)$~\citep[\S 28.7]{johnson_continuous_1995}, where $\mathrm{FDist}\lt(\nu_1, \nu_2\rt)$ is Fisher's $F$-distribution with $(\nu_1, \nu_2)$ degrees of freedom.
    The $k$th raw moment of $\mathrm{FDist}\lt(\nu_1, \nu_2\rt)$, denoted as $m_k \triangleq \mathbb{E} \mathsfit{w}_i^k$, exists up to $2 k < \nu_2 = \nu$ and is given as
    \[
        m_k
        = 
        {\lt(\frac{\nu_2}{\nu_1}\rt)}^k 
        \frac{\Gamma\lt(\nu_1/2 + k\rt)}{\Gamma\lt(\nu_1/2\rt)}
        \frac{\Gamma\lt(\nu_2/2 + k\rt)}{\Gamma\lt(\nu_2/2\rt)} \; .
        \qquad\text{\citep[Eq. 27.43]{johnson_continuous_1995}} 
        \nonumber
    \]
    This means that we can invoke \cref{thm:expected_maximum_highestmoment} as
    \[
        \mathbb{E}\max_{i = 1, \ldots, d} \mathsfit{u}_i^2
        \quad=\quad
        {\lt(\frac{\nu - 2}{\nu}\rt)}^{2}
        \mathbb{E}\max_{i = 1, \ldots, d}  \mathsfit{w}_i
        \quad\leq\quad
        \sqrt{2}
        {\lt(\frac{\nu - 2}{\nu}\rt)}^{2}
        d^{1/k}
        m_k^{1/k}
        \; ,
        \nonumber
    \]
    with any \(k < \nu/2\).
    
    For $m_k^{1/k}$, we can use the fact that the gamma function satisfies the recursion $\Gamma\lt(z + 1\rt) = z \Gamma\lt(z\rt) $, which implies $\Gamma\lt(a/2 + k\rt) = \Gamma\lt(a/2\rt) \prod_{i=0}^{k-1} \lt(a / 2 + i\rt) $ for any \(a > 0\).
    Therefore, 
    \[
        {\lt(
        \frac{\Gamma\lt(a/2 + k\rt)}{\Gamma\lt(a/2\rt)} 
        \rt)}^{1/k}
        &=
        {\lt(
        \prod_{i=0}^{k-1} \lt(\frac{a}{2} + i\rt)
        \rt)}^{1/k}
        \nonumber
        \\
        &\leq
        \frac{1}{k}
        \sum^{k-1}_{i=0} \lt(\frac{a}{2} + i\rt)
        &&\text{(AM-GM inequality)}
        \nonumber
        \\
        &=
        \frac{a}{2}
        +
        \frac{1}{k} \frac{k \lt(k - 1\rt)}{2}
        &&\text{(geometric series sum formula)}
        \nonumber
        \\
        &=
        \frac{a + k - 1 }{2} \; .
        \nonumber
    \]
    Applying this bound to \(a = \nu_2 = \nu\) and \(a = \nu_1 = 1\) respectively,
    \[
        m_k^{1/k}
        &= 
        {\lt(
        \nu^k 
        \frac{\Gamma\lt(1/2 + k\rt)}{\Gamma\lt(1/2\rt)}
        \frac{\Gamma\lt(\nu/2 + k\rt)}{\Gamma\lt(\nu/2\rt)} 
        \rt)}^{1/k}
        \nonumber
        \\
        &\leq
        \nu 
        \,
        \frac{k}{2}
        \,
        \frac{\nu + k - 1 }{2}
        \nonumber
        \\
        &<
        \nu 
        \,
        \frac{\nu}{4}
        \,
        \frac{3 \nu}{4}
        && \text{($k < \nu/2$)}
        \nonumber
        \\
        &<
        \frac{\nu^3}{4} \; .
        \nonumber
    \]
    Also, choosing $k = \lceil \nu/2 - 1 \rceil$, we have $d^{1/k} \leq d^{2/\lt(\nu - 2\rt)}$.
    This yields
    \[
        \mathbb{E}\max_{i = 1, \ldots, d} \mathsfit{u}_i^2
        \quad<\quad
        \sqrt{2}
        {\lt(\frac{\nu - 2}{\nu}\rt)}^{2}
        d^{\frac{2}{\nu - 2}}
        \frac{\nu^3}{4}
        \quad<\quad
        \frac{1}{2\sqrt{2}}
        \nu^3
        d^{\frac{2}{\nu - 2}}
        \; .
        \label{eq:expected_maximum_squared_upperbound_studentt}
    \]

    Lastly, the kurtosis of \(\mathsfit{u}_i = (\nu - 2)/\nu \, \mathsfit{v}_i\) follows as~\citep[Eq. 28.5]{johnson_continuous_1995}
    \[
        r_4 
        \quad=\quad 
        {\lt( \frac{\nu - 2}{\nu} \rt)}^4
        \mathbb{E} \mathsfit{w}_i^2
        \quad=\quad
        {\lt(\frac{\nu - 2}{\nu}\rt)}^4
        \frac{3 \nu^2}{ (\nu - 2) (\nu - 4) }
        \quad=\quad
        3 \frac{ {\lt(\nu - 2\rt)}^3 }{ \nu^2 {\lt(\nu - 4\rt)} }
        \quad\leq\quad
        3
        \; .
        \nonumber
    \]
    Plugging the bound in \cref{eq:expected_maximum_squared_upperbound_studentt} and the value of \(r_4\) into $g$ in \cref{thm:bbvi_complexity} yields the statement.
\end{proofEnd}

\vspace{-1ex}
\section{Analysis of Gradient Variance}\label{section:gradient_variance_analysis}
\vspace{-1ex}

\subsection{Overview}

The key technical contribution of this work is analyzing the gradient variance and thus establishing the constants $\mathcal{L}$ (\cref{assumption:expected_smoothness}) and $\sigma^2$ (\cref{assumption:bounded_variance}), which boils down to analyzing 
{%
\setlength{\belowdisplayskip}{0ex} \setlength{\belowdisplayshortskip}{0ex}
\setlength{\abovedisplayskip}{0ex} \setlength{\abovedisplayshortskip}{0ex}
\[
    \mathbb{E} \norm{ \widehat{\nabla f}\lt(\lambda; \mathsfit{u}\rt) - \widehat{\nabla f}\lt(\lambda'; \mathsfit{u}\rt) }_2^2
    &=
    \mathbb{E} \norm*{ 
        \frac{ \partial \mathcal{T}_{\lambda}\lt(\mathsfit{u}\rt) }{\partial \lambda} 
        \nabla
        \ell\lt(\mathcal{T}_{\lambda}\lt(\mathsfit{u}\rt)\rt) 
        - 
        \frac{ \partial \mathcal{T}_{\lambda'}\lt(\mathsfit{u}\rt) }{\partial \lambda'}  
        \nabla \ell\lt(\mathcal{T}_{\lambda'}\lt(\mathsfit{u}\rt)\rt) 
    }_2^2
    \nonumber
    \\
    &=
    \mathbb{E} \norm*{ 
        \frac{ \partial \mathcal{T}_{\lambda}\lt(\mathsfit{u}\rt) }{\partial \lambda}  
        \lt(
        \nabla
        \ell\lt(\mathcal{T}_{\lambda}\lt(\mathsfit{u}\rt)\rt) 
        - 
        \nabla \ell\lt(\mathcal{T}_{\lambda'}\lt(\mathsfit{u}\rt)\rt) 
        \rt)
    }_2^2 \; ,
    \nonumber
\]
}%
where the equality follows from the fact that the Jacobian $\nicefrac{ \partial \mathcal{T}_{\lambda}\lt(\mathsfit{u}\rt) }{\partial \lambda}$ does not depend on $\lambda$.
For mean-field location-scale variational families, the squared Jacobian follows as
{%
\setlength{\belowdisplayskip}{0ex} \setlength{\belowdisplayshortskip}{0ex}
\setlength{\abovedisplayskip}{0ex} \setlength{\abovedisplayshortskip}{0ex}
\[
    {
    \left(
    \frac{
        \partial \mathcal{T}_{\lambda}\left(\mathsfit{u}\right)
    }{
        \partial \lambda
    }
    \right)
    }^{\top}
    \frac{
        \partial \mathcal{T}_{\lambda}\left(\mathsfit{u}\right)
    }{
        \partial \lambda
    }
    =
    \mathrm{I}_d + \mathsfit{U}^2 \; ,
    \nonumber
\]
}%
where $\mathsfit{U} \triangleq \mathrm{diag}\lt(\mathsfit{u}_1, \ldots, \mathsfit{u}_d\rt)$~\citep{kim_practical_2023}.
This implies that 
{%
\setlength{\belowdisplayskip}{0ex} \setlength{\belowdisplayshortskip}{0ex}
\setlength{\abovedisplayskip}{1ex} \setlength{\abovedisplayshortskip}{1ex}
\[
    &\mathbb{E} \norm{
        \widehat{\nabla f}\left(\lambda; \mathsfit{u}\right) - \widehat{\nabla f}\lt(\lambda'; \mathsfit{u}\rt)
    }_2^2
    \nonumber
    \\
    &\qquad\qquad\qquad=
    \underbrace{
    \mathbb{E} 
    \norm{
    \nabla \ell\left(\mathcal{T}_{\lambda}\left(\mathsfit{u}\right)\right)
    -
    \nabla \ell\left(\mathcal{T}_{\lambda'}\left(\mathsfit{u}\right)\right)
    }_2^2
    }_{\triangleq V_{\text{loc}}}
    +
    \underbrace{
    \mathbb{E}\norm{
    \mathsfit{U}
    \lt(
    \nabla \ell\left(\mathcal{T}_{\lambda}\left(\mathsfit{u}\right)\right)
    -
    \nabla \ell\left(\mathcal{T}_{\lambda'}\left(\mathsfit{u}\right)\right)
    \rt)
    }_{2}^2 
    }_{\triangleq V_{\text{scale}}} \; .
    \label{eq:meanfield_gradient_variance_decomposition}
\]
}%
Our goal is to bound each term by $\norm{\lambda - \lambda'}_2^2$.

In order to solve the expectations, we need to simplify the $\nabla \ell$ terms.
For instance, for the gradient of the location $V_{\text{loc}}$, assuming that $\ell$ is $L$-smooth allows for a quadratic approximation.
That is, 
{%
\setlength{\belowdisplayskip}{1ex} \setlength{\belowdisplayshortskip}{1ex}
\setlength{\abovedisplayskip}{1ex} \setlength{\abovedisplayshortskip}{1ex}
\[
    V_{\text{loc}}
    \;\;=\;\;
    \mathbb{E}
    \norm{
    \nabla \ell\left(\mathcal{T}_{\lambda}\left(\mathsfit{u}\right)\right)
    -
    \nabla \ell\left(\mathcal{T}_{\lambda'}\left(\mathsfit{u}\right)\right)
    }_2^2
    \;\;\leq\;\;
    L^2
    \mathbb{E}
    \norm{
    \mathcal{T}_{\lambda}\left(\mathsfit{u}\right)
    -
    \mathcal{T}_{\lambda'}\left(\mathsfit{u}\right)
    }_2^2
    \;\;=\;\;
    L^2 \norm{\lambda - \lambda'}_2^2 \; ,
    \nonumber
\]
}%
where the last equality is by \cref{thm:reparam_identity}.

Now, it is tempting to use the same quadratic approximation strategy for the gradient of the scale $V_{\text{scale}}$.
Indeed, this strategy was used by \citet{domke_provable_2019} to bound the gradient variance of full-rank location-scale variational families and by \citet{ko_provably_2024} for structured location-scale variational families.
Unfortunately, this strategy does not immediately apply to mean-field families due to the matrix $\mathsfit{U}$.
We somehow have to decouple $\nabla \ell\lt(\mathcal{T}_{\lambda}\lt(\mathsfit{u}\rt)\rt) - \nabla \ell\lt(\mathcal{T}_{\lambda'}\lt(\mathsfit{u}\rt)\rt)$ and $\mathsfit{U}$, but in a way that does not lose the correlation between the two; the correlation leads to cancellations critical to obtaining a tight bound.
\citet{kim_practical_2023} used the inequality
{%
\setlength{\belowdisplayskip}{1ex} \setlength{\belowdisplayshortskip}{1ex}
\setlength{\abovedisplayskip}{1ex} \setlength{\abovedisplayshortskip}{1ex}
\[
    V_{\text{scale}} 
    \quad\leq\quad
    \mathbb{E}\norm{\mathsfit{U}^2}_{\mathrm{F}} \norm{ \nabla \ell\lt(\mathcal{T}_{\lambda}\lt(\mathsfit{u}\rt)\rt) - \nabla \ell\lt(\mathcal{T}_{\lambda'}\lt(\mathsfit{u}\rt) \rt) }_2^2 \; ,
    \label{eq:previous_inequality}
\]
}%
which resulted in a dimension dependence of $\mathrm{O}(r_4 \sqrt{d} )$ after solving the expectation.
The key question is whether this dimension dependence can be improved.
Due to the ordering of norms $\norm{\cdot}_2 \leq \norm{\cdot}_{\mathrm{F}}$, it is natural to consider the tighter inequality
{%
\setlength{\belowdisplayskip}{1ex} \setlength{\belowdisplayshortskip}{1ex}
\setlength{\abovedisplayskip}{1ex} \setlength{\abovedisplayshortskip}{1ex}
\[
    V_{\text{scale}} 
    \quad\leq\quad
    \mathbb{E}\norm{\mathsfit{U}}_{2}^2 \norm{ \nabla \ell\lt(\mathcal{T}_{\lambda}\lt(\mathsfit{u}\rt)\rt) - \nabla \ell\lt(\mathcal{T}_{\lambda'}\lt(\mathsfit{u}\rt) \rt) }_2^2 \; .
    \nonumber
\]
}%
(This step corresponds to \cref{eq:thm_gradient_variance_upperbound_sketch_eq1} in the proof sketch of the upcoming result.)
The main challenge, however, is solving the resulting expectation in a way that is also tight with respect to $d$.
We will see that this requires a careful probabilistic analysis.

\vspace{-1ex}
\subsection{Upper Bound on Gradient Variance}
\vspace{-1ex}

We now formally state our upper bound on the gradient variance.
In the context of proving \cref{thm:bbvi_complexity}, the following lemma implies both \cref{assumption:expected_smoothness} and \cref{assumption:noise}.
(See the proof of \cref{thm:bbvi_complexity}.)
We provide a corresponding unimprovability result in \cref{section:unimprovability}.

\begin{theoremEnd}[%
    restate,
    category=gradientvarianceupperboundmeanfieldgeneral, 
    text link={See the \hyperref[proof:prAtEnd\pratendcountercurrent]{\textit{full proof}} in~\cref{section:proof_gradient_variance_upperbound_meanfield_general}, p.~\pageref{proof:prAtEnd\pratendcountercurrent}.},
    text proof={Proof.}
]{lemma}\label{thm:gradient_variance_upperbound_meanfield_general}
    Suppose \cref{assumption:noise,assumption:almost_constant_hessian} hold, $\mathcal{Q}$ is a mean-field location-family, and $\widehat{\nabla f}$ is the reparametrization gradient.
    Then, for any \(\lambda, \lambda' \in \mathbb{R}^d \times \mathbb{D}^d\).
    {%
    \setlength{\belowdisplayskip}{0ex} \setlength{\belowdisplayshortskip}{0ex}
    \setlength{\abovedisplayskip}{1ex} \setlength{\abovedisplayshortskip}{1ex}
    \[
        \mathbb{E}\norm{ \widehat{\nabla f}\lt(\lambda; \mathsfit{u}\rt) - \widehat{\nabla f}\lt(\lambda'; \mathsfit{u}\rt) }_2^2
        \leq
        \Big\{ 
            2 \lt(1 + r_4\rt) \norm{H}_2^2 
            +
            4 \delta^2
            \Big(
            \nicefrac{1}{2}
            +
            r_4 
            +
            \mathbb{E}\max_{j = 1, \ldots, d} \mathsfit{u}_j^2
            \Big)
        \Big\} 
        \norm{ \lambda - \lambda' }_2^2
        \; .
        \nonumber
    \]
    }%
\end{theoremEnd}
\vspace{-2.3ex}
\begin{proofSketch}
    For the proof sketch, we will assume that $\ell$ is $L$-smooth instead of taking \cref{assumption:almost_constant_hessian}.
    This will vastly simplify the analysis and let us focus on the key elements.

    Recall $V_{\text{scale}}$ in  \cref{eq:meanfield_gradient_variance_decomposition}.
    Applying the operator norm and the $L$-smoothness of $\ell$ yields 
    {%
    \[
        V_{\text{scale}} 
        \quad\leq\quad
        \mathbb{E}\norm{\mathsfit{U}}_{2}^2 \norm{ \nabla \ell\lt(\mathcal{T}_{\lambda}\lt(\mathsfit{u}\rt)\rt) - \nabla \ell\lt(\mathcal{T}_{\lambda'}\lt(\mathsfit{u}\rt) \rt) }_2^2
        \quad\leq\quad
        L^2 \mathbb{E}\norm{\mathsfit{U}}_{2}^2 \norm{ \mathcal{T}_{\lambda}\lt(\mathsfit{u}\rt) - \mathcal{T}_{\lambda'}\lt(\mathsfit{u}\rt) }_2^2 \; .
        \label{eq:thm_gradient_variance_upperbound_sketch_eq1}
    \]
    }%
    It remains to solve the expectation over $\mathsfit{u}_1, \ldots, \mathsfit{u}_d$.
    Denote
    {%
    \[
        \lambda = (m, C), \qquad \lambda = (m', C'), \qquad \bar{m} \triangleq m - m',  \quad\text{and}\quad \bar{C} \triangleq C - C' \; ,
        \nonumber
    \]
    }%
    recall that $\mathcal{T}_{\lambda}\lt(u\rt) = C u + m$ (\cref{def:locscale}), and notice that, since $\mathsfit{U}$ is a diagonal matrix, $\norm{\mathsfit{U}}^2_2 = \max_{i = 1, \ldots, d} \mathsfit{u}_j^2$. 
    Then we can rewrite \cref{eq:thm_gradient_variance_upperbound_sketch_eq1} as
    {%
    \[
        V_{\text{scale}} 
        \;\;&\leq\;\;
        L^2 \mathbb{E} \Big( \max_{i=1, \ldots, d} \mathsfit{u}_j^2 \Big) {\textstyle\sum_{i=1}^d} {\lt( C_{ii} \mathsfit{u}_i + m_i - C_{ii}' \mathsfit{u}_i - m_i' \rt)}^2
        \nonumber
        \\
        \;\;&=\;\;
        L^2 \mathbb{E} \Big( \max_{i=1, \ldots, d} \mathsfit{u}_j^2 \Big) {\textstyle\sum_{i=1}^d} {\lt( \bar{C}_{ii} \mathsfit{u}_i + \bar{m}_i \rt)}^2
        \nonumber
        \\
        \;\;&\leq\;\;
        L^2 \mathbb{E} \Big( \max_{i=1, \ldots, d} \mathsfit{u}_j^2 \Big) {\textstyle\sum_{i=1}^d} {\lt( 2 \bar{C}_{ii}^2 \mathsfit{u}_i^2 + 2 \bar{m}_i^2 \rt)} \; .
        &&\text{(Young's inequality)}
        \nonumber
    \]
    }%
    The problematic term is 
    {%
    \[
        \mathbb{E}
        \Big( \max_{j=1, \ldots, d} \mathsfit{u}_j^2 \Big)
        {\textstyle\sum}_{i=1}^d
        \bar{C}_{ii}^2 \mathsfit{u}_i^2
        \quad=\quad
        \mathbb{E}
        \mathsfit{u}_{\mathsfit{i}_*}^2
        {\textstyle\sum}_{i=1}^d
        \bar{C}_{ii}^2 \mathsfit{u}_i^2
        \quad=\quad
        \mathbb{E}
        \lt[
            \bar{C}_{\mathsfit{i}_*\mathsfit{i}_*}^2 \mathsfit{u}_{\mathsfit{i}_*}^4 
            +
            {\textstyle\sum}^d_{j \neq \mathsfit{i}_*} \bar{C}_{jj}^2 \mathsfit{u}_{\mathsfit{i}_*}^2 \mathsfit{u}_j^2  \, 
        \rt]
        \; ,
        \nonumber
    \]
    }%
    where $\mathsfit{i}_*  = \argmax_{i=1, \ldots, d} \mathsfit{u}_i^2 $ is the coordinate of maximum magnitude.
    Here, $\mathsfit{u}_{\mathsfit{i}_*}^4 = \max_{i=1, \ldots, d} \mathsfit{u}_i^4$ is a heavy-tailed quantity that generally grows fast in $d$, unlike $\mathsfit{u}_{\mathsfit{i}_*}^2$.
    (\textit{e.g.}, for a Gaussian $\mathsfit{u}_i$, $\mathsfit{u}_{\mathsfit{i}_*}^2$ has an MGF but $\mathsfit{u}_{\mathsfit{i}_*}^4$ does not.)
    Therefore, a benign dimension dependence might appear futile.
    Notice, however, that the problematic term only affects a single dimension: the maximal axis indicated by $\mathsfit{i}_*$.
    A probabilistic analysis reveals that as $d$ increases, the effect of $\mathsfit{u}_{\mathsfit{i}_*}^4$ becomes averaged out and the effect of the remaining term involving $\mathsfit{u}_{\mathsfit{i}}^2$ dominates.
    More formally, 
    {%
    \setlength{\belowdisplayskip}{0ex} \setlength{\belowdisplayshortskip}{0ex}
    \[
        \mathbb{E}
        \mathsfit{u}_{\mathsfit{i}_*}^2
        {\textstyle\sum}_{i=1}^d
        \bar{C}_{ii}^2 \mathsfit{u}_i^2
        \quad=\quad
        {\textstyle\sum}_{i=1}^d
        \bar{C}_{ii}^2 \,
        \mathbb{E}
        \lt[
            \mathsfit{u}_{\mathsfit{i}_*}^4
            \mathds{1}\lt\{ \mathsfit{i}_{*} = i \rt\}
            +
            \mathsfit{u}_{\mathsfit{i}_*}^2 \mathsfit{u}_{i}^2
            \mathds{1}\lt\{ \mathsfit{i}_{*} \neq i \rt\}
        \rt] \; ,
        \nonumber
    \]
    }%
    where 
    {%
    \[
        \mathbb{E}
        \lt[
            \mathsfit{u}_{\mathsfit{i}_*}^4
            \mathds{1}\lt\{ \mathsfit{i}_{*} = i \rt\}
        \rt]
        \;\;=\;\;
        \mathbb{E} \lt[ \mathsfit{u}_{\mathsfit{i}_*}^4 \rt]
        \mathbb{E} \lt[ \mathds{1}\lt\{ \mathsfit{i}_{*} = i \rt\} \rt]
        \;\;=\;\;
        \mathbb{E} \lt[ \mathsfit{u}_{\mathsfit{i}_*}^4 \rt]
        \mathbb{P} \lt[ \mathsfit{i}_{*} = i \rt]
        \;\;=\;\;
        \mathbb{E} \lt[ \mathsfit{u}_{\mathsfit{i}_*}^4 \rt]
        (1/d)
        \; .
        \nonumber
    \]
    }%
    Since the maximum of $d$ random variable is always smaller than their sum, the probability of the maximally random event, $\mathbb{P}[\mathsfit{i}_* = i] = 1/d$, kills off the dimensional growth of $\mathsfit{u}_{\mathsfit{i}_*}^4$.
    In fact, using the crude bound $\mathbb{E} \mathsfit{u}_{\mathsfit{i}_*}^4 \leq \mathbb{E} \sum_{i=1}^d \mathsfit{u}_i^4 = d r_4$, where the last equality is due to \cref{assumption:noise}, is enough to make this term independent of $d$.
    The remaining dimension dependence comes from $\mathsfit{u}_{\mathsfit{i}_*}^2$:
    {%
    \[
        \mathbb{E}\lt[
        \mathsfit{u}_{\mathsfit{i}_*}^2 \mathsfit{u}_{i}^2
        \mathds{1}\lt\{ \mathsfit{i}_{*} \neq i \rt\}
        \rt]
        \;=\;
        \mathbb{E} 
        \Big[
            \max_{j \neq i} \mathsfit{u}_{j}^2 \mathsfit{u}_{i}^2
            \mathds{1}\lt\{ \mathsfit{i}_{*} \neq i \rt\}
        \Big]
        \;\leq\;
        \mathbb{E}\Big[ \max_{j = 1, \ldots, d - 1} \mathsfit{u}_{j}^2 \Big]
        \mathbb{E}\lt[ \mathsfit{u}_{i}^2 \rt]
        \;=\;
        \mathbb{E} \max_{j = 1, \ldots, d - 1} \mathsfit{u}_{j}^2  \; ,
        \nonumber
    \]
    }%
    where the last equality follows from \cref{assumption:noise}.
    Therefore, we finally obtain
    {%
    \setlength{\abovedisplayskip}{0ex} \setlength{\abovedisplayshortskip}{0ex}
    \[
        V_{\text{scale}}
        &\leq
        2 L^2 \sum^d_{i=1}  \lt[
        \Big( \mathbb{E} \max_{j = 1, \ldots, d - 1} \mathsfit{u}_{j}^2 + r_4 \Big) 
        \bar{C}_{ii}^2
        +
        \mathbb{E} \max_{j = 1, \ldots, d} \mathsfit{u}_{j}^2 \bar{m}_i^2
        \rt]
        \nonumber
        \\
        &\leq
        2 L^2
        \Big( \mathbb{E} \max_{j = 1, \ldots, d} \mathsfit{u}_{j}^2 + r_4 \Big) \lt( \norm{\bar{m}}_2^2 + \norm{\bar{C}}_{\mathrm{F}}^2 \rt)
        \nonumber
        \\
        &=
        2 L^2
        \Big( \mathbb{E} \max_{j = 1, \ldots, d} \mathsfit{u}_{j}^2 + r_4 \Big) \norm{\lambda - \lambda'}_2^2 \; .
        \nonumber
    \]
    }%
    The full proof performs an analogous analysis under the more general \cref{assumption:almost_constant_hessian}.
\end{proofSketch}
\vspace{-2ex}
\begin{proofEnd}
Recall \cref{eq:thm_gradient_variance_upperbound_meanfield_general_vscale}.
The proof consists of bounding the two terms $V_{\text{const}}$ and $V_{\text{non-const}}$.
First, for $V_{\text{const}}$, under \cref{assumption:almost_constant_hessian}, 
\[
    V_{\text{const}} 
    &\leq
    r_4 \, \norm{H}_2^2 \, \norm{ \lambda - \lambda' }_2^2 \; .
    &&\text{(\cref{thm:constant_hessian})}
    \label{eq:thm_gradient_variance_upperbound_meanfield_general_vconst}
\]
It remains to bound $V_{\text{non-const}}$, which is our main challenge.

Denote $\bar{m} \triangleq m - m'$ \text{and} $\bar{C} \triangleq C - C'$
such that
\[
    \mathcal{T}_{\lambda}\lt(\mathsfit{u}\rt)   
    -
    \mathcal{T}_{\lambda'}\lt(\mathsfit{u}\rt)   
    &=
    \lt( C \mathsfit{u} + m \rt)
    -
    \lt( C' \mathsfit{u} + m' \rt)
    \nonumber
    \\
    &=
    \lt( C - C'\rt) \mathsfit{u} + \lt( m - m' \rt)
    \nonumber
    \\
    &=
    \bar{C} \mathsfit{u} + \bar{m} \; .
    \nonumber
\]
Then
\[
    V_{\text{non-const}}
    &=
    \mathbb{E}\norm{\mathsfit{U}}_2^2 \norm{ \mathcal{T}_{\lambda}\lt(\mathsfit{u}\rt) - \mathcal{T}_{\lambda'}\lt(\mathsfit{u}\rt) }_{2}^2
    \nonumber
    \\
    &=
    \mathbb{E}\norm{\mathsfit{U}}_2^2 \norm{ \bar{C}\mathsfit{u} + \bar{m} }_2^2
    \nonumber
    \\
    &\leq
    \mathbb{E}\norm{\mathsfit{U}}_2^2 \lt( 2 \norm{ \bar{C}\mathsfit{u} }_2^2 + 2 \norm{\bar{m}}_2^2 \rt)
    &&\text{(Young's inequality)}
    \nonumber
    \\
    &=
    \mathbb{E}
    \lt( \max_{j=1, \ldots, d} \mathsfit{u}_j^2 \rt)
    \sum_{i=1}^d
    \lt(
    2 \bar{C}_{ii}^2 \mathsfit{u}_i^2
    +
    2 \bar{m}_i^2
    \rt)
    \nonumber
    \\
    &=
    2
    \mathbb{E}
    \sum_{i=1}^d
    \lt( \max_{j=1, \ldots, d} \mathsfit{u}_j^2 \rt)
    \bar{C}_{ii}^2 \mathsfit{u}_i^2
    +
    2
    \mathbb{E}
    \lt( \max_{j=1, \ldots, d} \mathsfit{u}_j^2 \rt)
    \sum_{i=1}^d
    \bar{m}_i^2
    \; .
    \nonumber
\]
We will focus on the first term.
Denoting $\mathsfit{i}_*  = \argmax_{i=1, \ldots, d} \mathsfit{u}_i^2 $, the coordinate of maximum magnitude, we can decompose the expectation by the contribution of the event $\mathsfit{i}_* = i$ and $\mathsfit{i}_* \neq i$.
That is,
\[
    \mathbb{E}
    \mathsfit{u}_{\mathsfit{i}_*}^2
    {\textstyle\sum}_{i=1}^d
    \bar{C}_{ii}^2 \mathsfit{u}_i^2
    &=
    {\textstyle\sum}_{i=1}^d
    \bar{C}_{ii} \,
    \mathbb{E}
    \Big[
        \underbrace{
        \mathsfit{u}_{\mathsfit{i}_*}^4
        \mathds{1}\lt\{ \mathsfit{i}_{*} = i \rt\}
        }_{V_{\text{max}}}
        +
        \underbrace{
        \mathsfit{u}_{\mathsfit{i}_*}^2 \mathsfit{u}_{i}^2
        \mathds{1}\lt\{ \mathsfit{i}_{*} \neq i \rt\}
        }_{V_{\text{non-max}}}
    \Big] \; .
    \nonumber
\]
The expectation of the event $\mathsfit{i}_* = i$ follows as
\[
    V_{\text{max}}
    &=
    \mathbb{E}
    \lt[
        \mathsfit{u}_{\mathsfit{i}_*}^4
        \mathds{1}\lt\{ \mathsfit{i}_{*} = i \rt\}
    \rt]
    \nonumber
    \\
    &=
    \mathbb{E} \lt[ \mathsfit{u}_{\mathsfit{i}_*}^4 \rt]
    \mathbb{E} \lt[ \mathds{1}\lt\{ \mathsfit{i}_{*} = i \rt\} \rt]
    \nonumber
    &&\text{($\mathsfit{u}_{\mathsfit{i}_*} \indep \mathsfit{i}_{*}$)}
    \\
    &=
    \mathbb{E} \lt[ \mathsfit{u}_{\mathsfit{i}_*}^4 \rt]
    \mathbb{P} \lt[ \mathsfit{i}_{*} = i \rt]
    \nonumber
    \\
    &=
    \mathbb{E} \lt[ \mathsfit{u}_{\mathsfit{i}_*}^4 \rt] \, \frac{1}{d}
    \nonumber
    \\
    &\leq
    \mathbb{E}\lt[ \sum_{i=1}^d \mathsfit{u}_i^4 \rt] \, \frac{1}{d}
    \nonumber
    &&\text{($\max_{j=1, \ldots, d} \mathsfit{u}_j^4 \leq {\textstyle\sum_{j=1}^d} \mathsfit{u}_j^4$)}
    \\
    &=
    \lt( d r_4 \rt) \, \frac{1}{d}
    \nonumber
    &&\text{(\cref{assumption:noise})}
    \\
    &=
    r_4
    \; .
    \label{eq:gradient_variance_upperbound_vmax}
\]
On the other hand, for the event $\mathsfit{i}_* \neq i$,
\[
    V_{\text{non-max}}
    &=
    \mathbb{E}\lt[
    \mathsfit{u}_{\mathsfit{i}_*}^2 \mathsfit{u}_{i}^2
    \mathds{1}\lt\{ \mathsfit{i}_{*} \neq i \rt\}
    \rt]
    \nonumber
    \\
    &=
    \mathbb{E} 
    \Big[
    \max_{j \neq i} \mathsfit{u}_{j}^2 \mathsfit{u}_{i}^2
    \mathds{1}\lt\{ \mathsfit{i}_{*} \neq i \rt\}
    \Big]
    \nonumber
    \\
    &=
    \mathbb{E}\Big[ \max_{j \neq i} \mathsfit{u}_{j}^2 \mathsfit{u}_{i}^2 \Big]
    \nonumber
    &&\text{($\mathds{1} \leq 1$)}
    \\
    &\leq
    \mathbb{E}\Big[ \max_{j \neq i} \mathsfit{u}_{j}^2 \Big] \mathbb{E}\lt[ \mathsfit{u}_{i}^2 \rt]
    \nonumber
    &&\text{($\mathsfit{u}_j \indep \mathsfit{u}_i$ for all $i \neq j$)}
    \\
    &=
    \mathbb{E}\Big[ \max_{j = 1, \ldots, d - 1} \mathsfit{u}_{j}^2 \Big]
    \mathbb{E}\lt[ \mathsfit{u}_{i}^2 \rt]
    \nonumber
    &&\text{($\mathsfit{u}_1$, \ldots, $\mathsfit{u}_d$ are i.i.d.)}
    \\
    &=
    \mathbb{E} \max_{j = 1, \ldots, d - 1} \mathsfit{u}_{j}^2  \; .
    \nonumber
    &&\text{(\cref{assumption:noise})}
\]

Therefore, we finally obtain
\[
    V_{\text{non-const}}
    &\leq
    2 \sum^d_{i=1}  \lt[
    \Big( \mathbb{E} \max_{j = 1, \ldots, d - 1} \mathsfit{u}_{j}^2 + r_4 \Big) 
    \bar{C}_{ii}^2
    +
    \mathbb{E} \max_{j = 1, \ldots, d} \mathsfit{u}_{j}^2 \bar{m}_i^2
    \rt]
    \nonumber
    \\
    &\leq
    2 
    \Big( \mathbb{E} \max_{j = 1, \ldots, d} \mathsfit{u}_{j}^2 + r_4 \Big) \lt( \norm{\bar{m}}_2^2 + \norm{\bar{C}}_{\mathrm{F}}^2 \rt)
    \nonumber
    &&\text{($\max_{j = 1, \ldots, d - 1} \mathsfit{u}_{j}^2 \leq \max_{j = 1, \ldots, d} \mathsfit{u}_{j}^2 $)}
    \\
    &=
    2 
    \Big( \mathbb{E} \max_{j = 1, \ldots, d} \mathsfit{u}_{j}^2 + r_4 \Big) \norm{\lambda - \lambda'}_2^2 \; .
    \label{eq:thm_gradient_variance_upperbound_meanfield_general_vnonconst}
\]
Combining \cref{eq:meanfield_gradient_variance_decomposition,eq:general_gradient_variance_vloc,eq:thm_gradient_variance_upperbound_meanfield_general_vscale,eq:thm_gradient_variance_upperbound_meanfield_general_vconst,eq:thm_gradient_variance_upperbound_meanfield_general_vnonconst} yields the statement.
\end{proofEnd}

\vspace{-1ex}
\subsection{Unimprovability}\label{section:unimprovability}
\vspace{-1ex}

We also demonstrate a lower bound, which implies that \cref{thm:gradient_variance_upperbound_meanfield_general} cannot be improved by the spectral bounds of $\nabla^2 \ell$.
From \cref{eq:meanfield_gradient_variance_decomposition} and the fundamental theorem of calculus, 
    {%
    \setlength{\belowdisplayskip}{1ex} \setlength{\belowdisplayshortskip}{1ex}
    \setlength{\abovedisplayskip}{1ex} \setlength{\abovedisplayshortskip}{1ex}
\[
    \mathbb{E}\norm{ 
    \widehat{\nabla f}\lt(\lambda; \mathsfit{u}\rt) }_2^2
    \quad\geq\quad
    \mathbb{E}\norm{
    \mathsfit{U} \lt(
    \nabla \ell\left( \mathsfit{z} \right) 
    - \nabla \ell\left( \bar{z} \right)
    \rt)
    }_{2}^2 
    \quad=\quad
    \mathbb{E}\norm[\big]{\textstyle
    \mathsfit{U}
        \int_0^1
        \nabla^2 \ell\lt( \mathsfit{z}^w \rt)
        \lt( \mathsfit{z} - \bar{z} \rt)
        \mathrm{d}w
    }_{2}^2
    \; ,
    \label{eq:lowerbound_subject}
\]
}%
where $\bar{z} \in \{ z \mid \nabla \ell\lt(z\rt) = 0 \}$ is a stationary point of $\ell$, $\mathsfit{z} \triangleq \mathcal{T}_{\lambda}(\mathsfit{u})$, and $\mathsfit{z}^w \triangleq w \mathsfit{z} + \lt(1 - w\rt) \bar{z}$.
There exists a matrix-valued function with bounded singular values that lower-bounds this quantity:

\begin{theoremEnd}[%
    restate,
    category=gradientvariancelowerbound,
    text link={\textit{Proof}. The \hyperref[proof:prAtEnd\pratendcountercurrent]{\textit{full proof}} can be found in~\cref{section:proof_gradient_variance_lowerbound}, p.~\pageref{proof:prAtEnd\pratendcountercurrent}. \qed},
    text proof={Proof.}
]{proposition}\label{thm:gradient_variance_lowerbound}
    Suppose \cref{assumption:noise} holds and \(\mathcal{Q}\) is a mean-field location-scale family.
    Then, for any $t > 0$, \(d > 0\), $\mu, L \in (0, +\infty)$ satisfying \(\mu \leq L\), there exists a matrix-valued function $H\lt(z\rt) : \mathbb{R}^d \to \mathbb{S}_{\succ 0}^d $ satisfying $\mu \mathrm{I}_d \preceq H \preceq L \mathrm{I}_d$ almost surely and a set of parameters $\lambda = (m, C) \in \mathbb{R}^d \times \mathbb{D}_{\succ 0}^d$ such that
    {%
    \setlength{\belowdisplayskip}{1ex} \setlength{\belowdisplayshortskip}{1ex}
    \setlength{\abovedisplayskip}{1ex} \setlength{\abovedisplayshortskip}{1ex}
    \[
        \mathbb{E}\norm[\big]{\textstyle
            \mathsfit{U}
            \int_0^1
            H\lt( \mathsfit{z}^w \rt)
            \lt( \mathsfit{z} - \bar{z} \rt)
            \mathrm{d}w
        }_{2}^2
        \geq
        \bigg\{ 
        \frac{{(L - \mu)}^2}{4}
        -
        \frac{L^2}{2}
        \frac{\mathbb{E} \underset{i=1, \ldots, d}{\max} \mathsfit{u}_{i}^4 }{d} 
        \bigg\} 
        c\lt(t, \varphi\rt)
        \Big\{
        \mathbb{E}\underset{i = 1, \ldots, d - 1}{\max}
        \mathsfit{u}_{i}^2 
        -
        t
        \Big\}
        \norm{ C }_{\mathrm{F}}^2 \; .
        \nonumber
    \]
    }%
    where $c(t, \varphi) > 0$ is a constant only dependent on $t$ and $\varphi$.
\end{theoremEnd}
\vspace{-1ex}
\begin{proofEnd}


    Recall $H_{\mathrm{worst}}$ in \cref{eq:hworst}.
    By inspection, we know that $H_{\mathrm{worst}}(\mathsfit{z}^w)$ only depends on the quantities $\mathsfit{i}_*$ and $\mathsfit{z}^w$.
    Then \cref{eq:matrix_constant_w} states that $w \mapsto H_{\mathrm{worst}}(\mathsfit{z}^w)$ is a constant function.
    Therefore, 
    \[
        \mathbb{E}\norm*{
        \mathsfit{U}
        \int_0^1
        H_{\mathrm{worst}}\lt( \mathsfit{z}^w \rt)
        \lt( \mathsfit{z} - \bar{z} \rt)
        \mathrm{d}w
        }_{2}^2
        &=
        \mathbb{E}\norm*{
        \mathsfit{U}
        H_{\mathrm{worst}}\lt( \mathsfit{z}^0 \rt)
        \lt( \mathsfit{z} - \bar{z} \rt)
        }_{2}^2
        &&\text{(\cref{eq:matrix_constant_w})}
        \nonumber
        \\
        &=
        \mathbb{E}\norm*{  
            \mathsfit{U}
            \lt(
            \alpha \mathrm{I}_d + \frac{\beta}{2} \lt( \mathrm{e}_{\mathsfit{i}_*} { \hat{\mathsfit{z}} }^{\top} + \hat{\mathsfit{z}} \, \mathrm{e}_{\mathsfit{i}_*}^{\top} \rt)
            \rt)
            \mathsfit{z}
        }_{2}^2 \; .
        &&\text{(\cref{eq:hworst})}
        \label{eq:lowerbound_eq0}
    \]
    This can be decomposed as  
    \[
        &
        \mathbb{E}\norm*{  
            \mathsfit{U}
            \lt(
            \alpha \mathrm{I}_d + \frac{\beta}{2} \lt( \mathrm{e}_{\mathsfit{i}_*} { \hat{\mathsfit{z}} }^{\top} + \hat{\mathsfit{z}} \, \mathrm{e}_{\mathsfit{i}_*}^{\top} \rt)
            \rt)
            \mathsfit{z}
        }_{2}^2
        \nonumber
        \\
        &=
        \mathbb{E}\norm*{  
            \alpha
            \mathsfit{U}
            \mathsfit{z}
            +
            \frac{\beta}{2} 
            \mathsfit{U} \mathrm{e}_{\mathsfit{i}_*} \lt( { \hat{\mathsfit{z}} }^{\top}  \mathsfit{z} \rt)
            + 
            \frac{\beta}{2} 
            \mathsfit{U} \hat{\mathsfit{z}} \, \lt( \mathrm{e}_{\mathsfit{i}_*}^{\top} \mathsfit{z} \rt)
        }_{2}^2
        \nonumber
        \\
        &=
        \mathbb{E}\norm*{  
            \alpha
            \mathsfit{U}
            \mathsfit{z}
            +
            \frac{\beta}{2}  
            \mathsfit{U} \mathrm{e}_{\mathsfit{i}_*} 
            \norm{\mathsfit{z}}_2
            + 
            \frac{\beta}{2} 
            \mathsfit{U} \hat{\mathsfit{z}}
            \mathsfit{z}_{\mathsfit{i}_*}
        }_{2}^2
        \nonumber
        \\
        &=
        \mathbb{E}\norm*{  
            \alpha
            \mathsfit{U}
            \mathsfit{z}
            +
            \lt(
            \frac{\beta}{2}  
            \norm{\mathsfit{z}}_2
            \rt)
            \mathsfit{U} \mathrm{e}_{\mathsfit{i}_*} 
            + 
            \lt(
            \frac{\beta}{2} 
            \hat{\mathsfit{z}}_{\mathsfit{i}_*}
            \rt)
            \mathsfit{U} \mathsfit{z}
        }_{2}^2
        \nonumber
        \\
        &=
        \mathbb{E}\norm*{  
            \lt(
            \frac{\beta}{2}  
            \norm{\mathsfit{z}}_2
            \rt)
            \mathsfit{U} \mathrm{e}_{\mathsfit{i}_*} 
            + 
            \lt(
            \alpha
            +
            \frac{\beta}{2} 
            \hat{\mathsfit{z}}_{\mathsfit{i}_*}
            \rt)
            \mathsfit{U}
            \mathsfit{z}
        }_{2}^2
        \nonumber
        \\
        &=
        \mathbb{E}
        \sbra[\Bigg]{
        \frac{\beta^2}{4}
        \norm{\mathsfit{z}}_2^2
        \norm{\mathsfit{U} \mathrm{e}_{\mathsfit{i}_*}}_{2}^2
        +
        {\lt( \alpha + \frac{\beta}{2} \hat{\mathsfit{z}}_{\mathsfit{i}_*} \rt)}^2 \norm{\mathsfit{U} \mathsfit{z}}_{2}^2
        +
        \beta
        \lt( \alpha  + \frac{\beta}{2} \hat{\mathsfit{z}}_{\mathsfit{i}_*} \rt) 
        \norm{\mathsfit{z}}_2 
        \lt( \mathrm{e}_{\mathsfit{i}_*}^{\top} \mathsfit{U}^2 \mathsfit{z} \rt)
        }
        \nonumber
        \\
        &=
        \mathbb{E}
        \sbra[\Bigg]{
        \frac{\beta^2}{4}
        \mathsfit{u}_{\mathsfit{i}_*}^2 
        \norm{\mathsfit{z}}_2^2 
        }
        +
        \underbrace{
        \mathbb{E}
        \lt[
        {\lt( \alpha + \frac{\beta}{2} \hat{\mathsfit{z}}_{\mathsfit{i}_*} \rt)}^2 \norm{\mathsfit{U} \mathsfit{z}}_{2}^2
        \rt]
        }_{\triangleq V_1}
        +
        \underbrace{
        \mathbb{E}
        \lt[
        \beta
        \lt( \alpha  + \frac{\beta}{2} \hat{\mathsfit{z}}_{\mathsfit{i}_*} \rt) 
        \norm{\mathsfit{z}}_2 
        \lt( \mathrm{e}_{\mathsfit{i}_*}^{\top} \mathsfit{U}^2 \mathsfit{z} \rt)
        \rt]
        }_{\triangleq V_2}
        \; .
        \label{eq:thm_lowerbound_eq2}
    \]
    Here, the first term $\beta^2/4 \mathsfit{u}_{\mathsfit{i}_*}^2 \norm{\mathsfit{z}}_2^2$ is the worst-case behavior we expect from solving \cref{eq:lower_bound_goal}.
    The remaining terms $V_1$ and $V_2$ are the error caused by inexactly solving \cref{eq:lower_bound_goal}.
    It suffices to show that $\beta^2/4 \mathsfit{u}_{\mathsfit{i}_*}^2 \norm{\mathsfit{z}}_2^2$ dominates lower bounds on $V_1$ and $V_2$ asymptotically in $L$ and $d$.

    $V_1 \geq 0$ trivially holds and can immediately be lower-bounded.
    $V_2$, on the other hand, is not necessarily non-negative.
    Therefore, we will use the bound $V_2 \geq -\abs{\mathbb{E} V_2}$.
    \[
        \abs{\mathbb{E}V_2} 
        &\leq
        \beta
        \lt( \alpha  + \frac{\beta}{2} \rt) 
        \mathbb{E} 
        \norm{\mathsfit{z}}_2  \,
        \abs{ \mathrm{e}_{\mathsfit{i}_*}^{\top} \mathsfit{U}^2 \mathsfit{z} }
        \nonumber
        \\
        &\leq
        \beta
        \lt( \alpha  + \frac{\beta}{2} \rt) 
        \mathbb{E} 
        \norm{\mathsfit{z}}_2  \,
        \mathsfit{u}_{\mathsfit{i}_*}^2 \abs{ \mathsfit{z}_{\mathsfit{i}_*}  }
        \nonumber
        \\
        &\leq
        \beta
        \lt( \alpha  + \frac{\beta}{2} \rt) 
        {\lt(
        \mathbb{E} 
            \norm{\mathsfit{z}}_2^2
            \mathsfit{u}_{\mathsfit{i}_*}^2
        \rt)}^{1/2}
        {\Big(
        \underbrace{
            \mathbb{E} 
            \mathsfit{u}_{\mathsfit{i}_*}^2
            \mathsfit{z}_{\mathsfit{i}_*}^2
        }_{\triangleq V_3}
        \Big)}^{1/2}
        \; .
        &&\text{(Cauchy-Schwarz)}
        \label{eq:lowerbound_v2_mag_bound}
    \]
    For $V_3$, we can use an argument similar to \cref{eq:gradient_variance_upperbound_vmax} where we distribute the influence of the maximum coordinate over the $d$ coordinates.
    \[
        V_3 
        &= 
        \mathbb{E} 
        \mathsfit{u}_{\mathsfit{i}_*}^2
        \mathsfit{z}_{\mathsfit{i}_*}^2
        = 
        \mathbb{E} 
        \sum^d_{i=1} 
        \mathsfit{u}_{\mathsfit{i}_*}^2
        \mathsfit{z}_{i}^2
        \,
        \mathds{1}\lt\{ \mathsfit{i}_* = i \rt\}
        \nonumber
        \\
        &= 
        \sum^d_{i=1} 
        \mathbb{E} \lt[
        \mathsfit{u}_{\mathsfit{i}_*}^2
        C_{ii}^2 \mathsfit{u}_i^2
        \,
        \mathds{1}\lt\{ \mathsfit{i}_* = i \rt\}
        \rt]
        \nonumber
        \\
        &= 
        \sum^d_{i=1} 
        C_{ii}^2 
        \mathbb{E} \lt[ \mathsfit{u}_{\mathsfit{i}_*}^4 \rt]
        \mathbb{E} \lt[ \mathds{1}\lt\{ \mathsfit{i}_* = i \rt\} \rt]
        &&\text{($\mathsfit{u}_{\mathsfit{i}_*} \indep \mathsfit{i}_*$)}
        \nonumber
        \\
        &= 
        \sum^d_{i=1} 
        C_{ii}^2 
        \mathbb{E} \lt[ \mathsfit{u}_{\mathsfit{i}_*}^4 \rt]
        \mathbb{P} \lt[ \mathsfit{i}_* = i \rt]
        \nonumber
        \\
        &= 
        \frac{1}{d} 
        \mathbb{E} \lt[ \mathsfit{u}_{\mathsfit{i}_*}^4 \rt]
        \norm{C}_{\mathrm{F}}^2
        \nonumber
        \\
        &= 
        \frac{1}{d} 
        \lt( \mathbb{E} \mathsfit{u}_{\mathsfit{i}_*}^4 \rt) \mathbb{E} \norm{\mathsfit{z}}_2^2 \, .
        \label{eq:lowerbound_v3_bound}
    \]
    The last equality follows by applying \cref{thm:noise} to the identity $\mathbb{E}\norm{\mathsfit{z}}_2^2 = \mathbb{E} \mathsfit{u}^{\top} C^{\top} C \mathsfit{u}$.
    By applying \cref{eq:lowerbound_v3_bound} into \cref{eq:lowerbound_v2_mag_bound}, we can now notice that $V_2$ decreases by a factor of $\mathbb{E}\mathsfit{u}_{\mathsfit{i}_*}^4/d$.
    \[
        \mathbb{E}V_2 
        &\geq 
        - 
        \beta \lt(\alpha + \frac{\beta}{2}\rt) 
        \frac{\mathbb{E} \lt[ \mathsfit{u}_{\mathsfit{i}_*}^4 \rt]}{d} 
        \sqrt{ \mathbb{E} \mathsfit{u}_{\mathsfit{i}_*}^2 \norm{\mathsfit{z}}_2^2 }
        \, \sqrt{ \mathbb{E} \norm{\mathsfit{z}}_2^2 }
        \nonumber
        \\
        &\geq 
        - 
        \beta \lt(\alpha + \frac{\beta}{2}\rt) 
        \frac{\mathbb{E} \lt[ \mathsfit{u}_{\mathsfit{i}_*}^4 \rt]}{d} 
        \mathbb{E} \lt[ \mathsfit{u}_{\mathsfit{i}_*}^2 \norm{\mathsfit{z}}_2^2 \rt]
        \; .
        &&\text{(\cref{assumption:noise})}
        \nonumber
    \]
    It is clear that $V_2$ vanishes as $d \to \infty$.
    
    Applying the lower bound on $V_2$ into \cref{eq:lowerbound_eq0,eq:thm_lowerbound_eq2}, we have
    \[
        &\mathbb{E}\norm*{
        \mathsfit{U}
        \int_0^1
        H_{\mathrm{worst}}\lt( \mathsfit{z}^w \rt)
        \lt( \mathsfit{z} - \bar{z} \rt)
        \mathrm{d}w
        }_{2}^2
        \nonumber
        \\
        &\quad\geq
        \frac{\beta^2}{4}
        \mathbb{E}\lt[
        \mathsfit{u}_{\mathsfit{i}_*}^2
        \norm{\mathsfit{z}}_2^2
        \rt]
        - 
        \beta \lt(\alpha + \frac{\beta}{2}\rt) 
        \frac{\mathbb{E} \lt[ \mathsfit{u}_{\mathsfit{i}_*}^4 \rt]}{d} 
        \mathbb{E} \lt[ \mathsfit{u}_{\mathsfit{i}_*}^2 \norm{\mathsfit{z}}_2^2 \rt]
        \nonumber
        \\
        &\quad=
        \lt\{
        \frac{\beta^2}{4}
        -
        \lt( \alpha \beta + \frac{\beta^2}{2} \rt)
        \frac{\mathbb{E} \lt[ \mathsfit{u}_{\mathsfit{i}_*}^4 \rt]}{d} 
        \rt\}
        \mathbb{E}
        \mathsfit{u}_{\mathsfit{i}_*}^2
        \norm{\mathsfit{z}}_2^2
        \nonumber
        \\
        &\quad=
        \lt\{
        \frac{{(L - \mu)}^2}{16}
        -
        \lt( \frac{L^2 - \mu^2}{4} + \frac{ {(L - \mu)}^2 }{ 8 } \rt)
        \frac{\mathbb{E} \lt[ \mathsfit{u}_{\mathsfit{i}_*}^4 \rt]}{d} 
        \rt\}
        \mathbb{E}
        \mathsfit{u}_{\mathsfit{i}_*}^2
        \norm{\mathsfit{z}}_2^2
        \nonumber
        &&\text{(\cref{eq:thm_gradient_variance_lowerbound_alpha_beta})}
        \\
        &\quad\geq
        \lt\{
        \frac{{(L - \mu)}^2}{16}
        -
        \lt( \frac{L^2 - \mu^2}{4} + \frac{ L^2 + \mu^2 }{ 4 } \rt)
        \frac{\mathbb{E} \lt[ \mathsfit{u}_{\mathsfit{i}_*}^4 \rt]}{d} 
        \rt\}
        \mathbb{E}
        \mathsfit{u}_{\mathsfit{i}_*}^2
        \norm{\mathsfit{z}}_2^2
        \nonumber
        &&\text{(Young's inequality)}
        \\
        &\quad=
        \lt\{
        \frac{{(L - \mu)}^2}{16}
        -
        \frac{L^2}{2}
        \frac{\mathbb{E} \underset{i=1, \ldots, d}{\max} \mathsfit{u}_{i}^4 }{d} 
        \rt\}
        \mathbb{E}
        \mathsfit{u}_{\mathsfit{i}_*}^2
        \norm{C \mathsfit{u}}_2^2 \; .
        \label{eq:thm_lowerbound_eq3}
    \]
    It remains to solve the expectation.
    
    Let us decompose the events where the $i$th coordinate attains the maximum ($\mathsfit{i}_* = i$) or not ($\mathsfit{i}_* \neq i$) as done in \cref{thm:gradient_variance_upperbound_meanfield_general}.
    \[
         \mathbb{E} 
         \mathsfit{u}_{\mathsfit{i}_*}^2
         \norm{ C \mathsfit{u} }_2^2
         &=
         \mathbb{E} \lt[
         \sum_{i=1}^d C_{ii} \mathsfit{u}_i^2 
         \mathsfit{u}_{\mathsfit{i}_*}^2
         \rt]
         \nonumber
         \\
         &=
         \sum_{i=1}^d 
         C_{ii}^2
         \Big\{
         \mathbb{E}\lt[
            \mathsfit{u}_i^2 
            \mathsfit{u}_{\mathsfit{i}_*}^2
            \mathds{1}_{\mathsfit{i}_* = i}
         \rt]
         +
         \mathbb{E}\lt[
            \mathsfit{u}_i^2 
            \mathsfit{u}_{\mathsfit{i}_*}^2
            \mathds{1}_{\mathsfit{i}_* \neq i}
         \rt]
         \Big\}
         \nonumber
         \\
         &\geq
         \sum_{i=1}^d 
         C_{ii}^2
         \mathbb{E}\lt[
            \mathsfit{u}_i^2 
            \mathsfit{u}_{\mathsfit{i}_*}^2
            \mathds{1}_{\mathsfit{i}_* \neq i}
         \rt]
         \; .
         \nonumber
    \]
    We are left with the expectation over the event $\mathsfit{i}_* \neq i$.
    For the upper bound in~\cref{thm:gradient_variance_upperbound_meanfield_general}, the expectation was solved by noticing that $\mathsfit{u}_i^2$ and $\mathsfit{u}_{\mathsfit{i}_*}^2$ can be made independent after upper bounding the indicator.
    For a lower bound, however, breaking up the expectation for $\mathsfit{u}_i^2$ and $\mathsfit{u}_{\mathsfit{i}_*}^2$ is more involved.
    \[
        \mathbb{E}\lt[
           \mathsfit{u}_i^2 
           \mathsfit{u}_{\mathsfit{i}_*}^2 \,
           \mathds{1}_{\mathsfit{i}_* = i}
        \rt]
        &=
        \mathbb{E}\lt[
           \mathsfit{u}_i^2 
           \max_{j \neq i} \mathsfit{u}_{j}^2 \,
           \mathds{1}_{\mathsfit{i}_* = i}
        \rt]
        \nonumber
        \\
        &=
        \mathbb{E}\lt[
           \mathsfit{u}_i^2 
           \max_{j \neq i} \mathsfit{u}_{j}^2 \,
           \mathds{1}\lt\{
                \mathsfit{u}_{i}^2 < \max_{j \neq i} \mathsfit{u}_{j}^2
           \rt\}
        \rt] \; .
        \label{eq:gradient_variance_lowerbound_eq_last}
    \]
    By introducing a free variable $t > 0$, we can break up the indicator 
    \[
        \mathds{1}\lt\{
            \mathsfit{u}_{i}^2 < \max_{j \neq i} \mathsfit{u}_{j}^2
        \rt\}
        &\geq
        \mathds{1}\lt\{
            \mathsfit{u}_{i}^2 < \max_{j \neq i} \mathsfit{u}_{j}^2, \,
            t < \max_{j \neq i} \mathsfit{u}_{j}^2
        \rt\}
        \nonumber
        \\
        &\geq
        \mathds{1}\lt\{
            \mathsfit{u}_{i}^2 < t, \,
            \max_{j \neq i} \mathsfit{u}_{j}^2 > t
        \rt\}
        \nonumber
        \\
        &=
        \mathds{1}\lt\{
            \mathsfit{u}_{i}^2 < t, \,
        \rt\}
        \mathds{1}\Big\{
            \max_{j \neq i} \mathsfit{u}_{j}^2 > t
        \Big\} 
        \; .
        \label{eq:gradient_variance_lowerbound_eq_after_last}
    \]
    This then allows the expectation to break up between terms depending on $\mathsfit{u}_i^2$ and $\max_{j \neq i}\mathsfit{u}_j$, which is the independence that we were after.
    That is, applying \cref{eq:gradient_variance_lowerbound_eq_after_last} to \cref{eq:gradient_variance_lowerbound_eq_last},
    \[
        \mathbb{E}\lt[
           \mathsfit{u}_i^2 
           \mathsfit{u}_{\mathsfit{i}_*}^2 \,
           \mathds{1}_{\mathsfit{i}_* = i}
        \rt]
        &\geq
        \mathbb{E}\lt[
            \mathsfit{u}_i^2 
            \max_{j \neq i} \mathsfit{u}_{j}^2 \,
            \mathds{1}\lt\{
                \mathsfit{u}_{i}^2 < t, \,
            \rt\}
            \mathds{1}\Big\{
                \max_{j \neq i} \mathsfit{u}_{j}^2 > t
            \Big\} 
        \rt]
        \nonumber
        \\
        &=
        \mathbb{E}\lt[
           \mathsfit{u}_i^2  \,
           \mathds{1}\lt\{
                \mathsfit{u}_{i}^2 < t
           \rt\}
        \rt]
        \mathbb{E}\lt[
           \max_{j = 1, \ldots, d - 1} \mathsfit{u}_{j}^2 \,
           \mathds{1}\Big\{
                \max_{j = 1, \ldots, d - 1} \mathsfit{u}_{j}^2 > t
           \Big\}
        \rt]
        \nonumber
        \\
        &=
        \mathbb{E}\lt[
           \mathsfit{u}_i^2  \,
           \mathds{1}\lt\{
                \mathsfit{u}_{i}^2 < t
           \rt\}
        \rt]
        \lt(
        \mathbb{E}\lt[
           \max_{j = 1, \ldots, d - 1} \mathsfit{u}_{j}^2 
        \rt]
        -
        \mathbb{E}\lt[
           \max_{j = 1, \ldots, d - 1} \mathsfit{u}_{j}^2 \,
           \mathds{1}\Big\{
                \max_{j = 1, \ldots, d - 1} \mathsfit{u}_{j}^2 \leq t
           \Big\}
        \rt] 
        \rt)
        \nonumber
        \\
        &\geq
        \lt(
        \int^t_0 \mathbb{P}\lt[ \mathsfit{u}_{i}^2 > s \rt] \mathrm{d}s
        \rt)
        \lt(
        \mathbb{E}\lt[
           \max_{j = 1, \ldots, d - 1} \mathsfit{u}_{j}^2 
        \rt]
        -
        t
        \rt) \; .
        \nonumber
    \]
    Notice that the function $ (t, \varphi) \mapsto \int^t_0 \mathbb{P}\lt[ \mathsfit{u}_{i}^2 > s \rt] \mathrm{d}s$ is strictly positive as long as $t > 0$ and only dependent on $t$ and the base distribution $\varphi$.

    We now obtain our final result by combining the results into \cref{eq:thm_lowerbound_eq3}.
    With explicit constants, 
    \[
        \mathbb{E}\norm*{
        \int_0^1
        H_{\mathrm{worst}}\lt( \mathsfit{z}^w \rt)
        \lt( \mathsfit{z} - \bar{z} \rt)
        \mathrm{d}w
        }_{\mathsfit{U}^2}^2
        &\geq
        \lt\{
        \frac{{(L - \mu)}^2}{4}
        -
        \frac{L^2}{2}
        \frac{\mathbb{E} \underset{i=1, \ldots, d}{\max} \mathsfit{u}_{i}^4 }{d} 
        \rt\}
        \nonumber
        \\
        &\quad\qquad
        \times
        \lt(
        \int^t_0 \mathbb{P}\lt[ \mathsfit{u}_{i}^2 > s \rt] \mathrm{d}s
        \rt)
        \lt(
        \mathbb{E}\lt[
           \max_{i = 1, \ldots, d - 1} \mathsfit{u}_{i}^2 
        \rt]
        -
        t
        \rt)
        \norm{C}_{\mathrm{F}}^2 
        \; .
        \nonumber
    \]
    Substituting $ c\lt(t, \varphi\rt) \triangleq \int^t_0 \mathbb{P}\lt[ \mathsfit{u}_{i}^2 > s \rt] \mathrm{d}s$ into this yields the stated result.
\end{proofEnd}

For Gaussians, $\mathbb{E}\max_{i} \mathsfit{u}_i^4$ is upper bounded as $\mathrm{O}(\sqrt{d})$~\citep[Eq. 1.6]{gumbel_maxima_1954}, which means the negative term vanishes at a $\mathrm{O}(1/\sqrt{d})$ rate.
Furthermore, $\mathbb{E} \max_{i} \mathsfit{u}_i^2 \geq {(\mathbb{E} \max_{i} \mathsfit{u}_i)}^2 = \Omega\lt(\log d\rt)$ by the well-known lower bound on the expected maximum of i.i.d. Gaussians~\citep[Exercise 2.11.(b)]{wainwright_highdimensional_2019}.
Combining these facts with \cref{thm:gradient_variance_lowerbound} yield a $\Omega\lt( L^2 \log d \rt)$ bound on \cref{eq:lowerbound_subject}.

\begin{remark}\label{remark:limitation}
    It is not obvious that the rows of our worst-case example $H_{\mathrm{worst}}$ form conservative vector fields.
    This means that \cref{thm:gradient_variance_lowerbound} does not assert the existence of a function $\ell$ that satisfies $\nabla^2 \ell = H_{\mathrm{worst}}$.
    However, it does suggest that one cannot improve \cref{thm:gradient_variance_upperbound_meanfield_general} by  relying only on spectral bounds on the Hessian.
\end{remark}

\vspace{-1ex}
\section{Discussion}
\vspace{-1ex}
\subsection{Related Works}
\vspace{-1ex}
Early results analyzing VI had to rely on assumptions that either:
\begin{enumerate*}[label=(\roman*)]
    \item do not hold on Gaussian targets,
    \item are difficult to verify, or
    \item require bounds on the domain~\citep{buchholz_quasimonte_2018,khan_faster_2016,regier_fast_2017,fan_fast_2015,fujisawa_multilevel_2021,liu_quasimonte_2021,alquier_concentration_2020,nguyen_wasserstein_2025}.
\end{enumerate*}
coordinate-ascent VI (CAVI), in particular, was studied on specific models~\citep{ghorbani_instability_2019,zhang_theoretical_2020} only. 
Under general and verifiable assumptions, \citet{xu_computational_2022} obtained asymptotic convergence guarantees, while partial results, such as bounds on the gradient variance~\citep{domke_provable_2019,kim_practical_2023,fan_fast_2015}, or regularity of the ELBO~\citep{domke_provable_2020,titsias_doubly_2014,challis_gaussian_2013}, were known.

It was only recently that non-asymptotic quantitative convergence under realizable and verifiable assumptions was established.
For BBVI specifically, \citet{hoffman_blackbox_2020} first proved convergence on Gaussian targets (quadratic $\ell$), while \citet{domke_provable_2023,kim_convergence_2023,kim_linear_2024} proved the first results on strongly convex and smooth functions with location-scale families.
\citet{surendran_theoretical_2025} extended these results to non-convex smooth functions and more complex variational family parametrizations, and \citet{cheng_kernel_2024} analyzed a variant of semi-implicit VI.
The results of \citeauthor{hoffman_blackbox_2020,domke_provable_2023,kim_linear_2024}, who focused on full-rank families, suggest a $\mathrm{O}(d)$ dimension dependence in the iteration complexity.
On the other hand, \citet{kim_convergence_2023} reported a $\mathrm{O}(\sqrt{d})$ dimension dependence for mean-field location-scale families, while conjecturing $\mathrm{O}(\log d)$ dependence, based on the partial result of \citet{kim_practical_2023}.
For targets with a diagonal Hessian structure, \citet[Corollary 1]{ko_provably_2024} show that mean-field families are dimension-independent.

Apart from BBVI, Wasserstein VI algorithms---which minimize the KL divergence on the Wasserstein geometry---provide non-asymptotic convergence guarantees.
In particular, the algorithms by \citet{lambert_variational_2022,diao_forwardbackward_2023} optimize over the full-rank Gaussian family, while that of \citet{jiang_algorithms_2025} optimizes over all mean-field families with bounded second moments.
To guarantee $\mathbb{E}\mathrm{W}_2{\lt(q_{\lambda_T}, q_{\lambda_*}\rt)}^2 \leq \epsilon$ on strongly log-concave and log-smooth targets, they all report an iteration complexity of $\mathrm{O}(d \epsilon^{-1} \log \epsilon^{-1})$.
Meanwhile, under the same conditions,~\citet{arnese_convergence_2024,lavenant_convergence_2024} analyzed (block) CAVI, and reported an iteration complexity of $\mathrm{O}(d \log \epsilon^{-1})$.
\citet{bhattacharya_convergence_2025} provides a concurrent result on CAVI,  but relies on an assumption that departs from log-concavity and smoothness.
Finally, \citet{bhatia_statistical_2022} analyzes a specialized algorithm optimizing over only the scale of Gaussians, which has a gradient query complexity of $\mathrm{O}(d k \epsilon^{-3})$, where $k$ is the user-chosen number of rank-1 factors in the scale matrix.

\vspace{-1ex}
\subsection{Conclusions}
\vspace{-1ex}
In this work, we proved that BBVI with mean-field location-scale families is able to converge with an iteration complexity with only a $\mathrm{O}(\log d)$ dimension dependence, as long as the tails of the family are sub-Gaussian.
For high-dimensional targets, this suggests a substantial speed advantage over BBVI with full-rank families.
In practice, the mean-field approximation can be combined with other design elements such as control variates~\citep{roeder_sticking_2017,miller_reducing_2017,geffner_approximation_2020,geffner_using_2018,wang_joint_2024,boustati_amortized_2020} and data-point subsampling~\citep{titsias_doubly_2014,kucukelbir_automatic_2017}.
Our analysis strategy should easily be combined with existing analyses~\citep{kim_linear_2024,kim_demystifying_2024} for such design elements.

For a target distribution $\pi$ with a condition number of $\kappa$ and a target accuracy level $\epsilon$, we now know how to improve the dependence on $d$ and $\epsilon$ in the iteration complexity:
Using less-expressive families such as mean-field (\cref{thm:bbvi_complexity}) or structured~\citep{ko_provably_2024} families improves the dependence on $d$, while applying control variates to gradient estimators~\citep{kim_linear_2024} improves the dependence on $\epsilon$.
However, it is currently unclear whether the dependence on $\kappa$ is tight or improvable. 
If it is tight, it would be worth investigating whether this can be provably improved through algorithmic modifications, for example, via stochastic second-order optimization methods~\citep{byrd_stochastic_2016,fan_fast_2015,meng_fast_2020,regier_fast_2017,liu_quasimonte_2021}.

Another future direction would be to develop methods that are able to adaptively adjust the computational cost between $\mathrm{O}(\log d)$ and $\mathrm{O}(d)$ by trading statistical accuracy akin to the method of \citet{bhatia_statistical_2022}.
Existing BBVI schemes with ``low-rank(-plus-diagonal)'' families~\citep{tomczak_efficient_2020,ong_gaussian_2018,rezende_stochastic_2014} result in a non-smooth, non-Lipschitz, and non-convex landscape.
This not only rules out typical theoretical convergence guarantees but also exhibits unstable and slow convergence in practice~\citep{modi_batch_2025}.
Furthermore, understanding the statistical side of this trade-off will be an important direction.
As of now, our understanding is restricted to either mean-field or full-rank families~\citep{wang_variational_2019,wang_frequentist_2019,katsevich_approximation_2024,margossian_variational_2025,margossian_shrinkagedelinkage_2023,yang_avariational_2020,zhang_convergence_2020} with little in between except for the work of~\citet{bhatia_statistical_2022}.

\newpage
\begin{ack}
The authors thank Anton Xue for helpful discussions and the reviewers for helpful suggestions.

K. Kim, J. R. Gardner were supported through the NSF award [IIS2145644]; 
Y.-A. Ma was supported by the NSF Award CCF-2112665 (TILOS), the DARPA AIE program, and the CDC-RFA-FT-23-0069;
T. Campbell was supported by the NSERC Discovery Grant RGPIN-2025-04208.
\end{ack}

\bibliographystyle{plainnat}
\bibliography{references}

\begin{thebibliography}{102}
\providecommand{\natexlab}[1]{#1}
\providecommand{\url}[1]{\texttt{#1}}
\expandafter\ifx\csname urlstyle\endcsname\relax
  \providecommand{\doi}[1]{doi: #1}\else
  \providecommand{\doi}{doi: \begingroup \urlstyle{rm}\Url}\fi

\bibitem[Agrawal et~al.(2020)Agrawal, Sheldon, and
  Domke]{agrawal_advances_2020}
Abhinav Agrawal, Daniel~R Sheldon, and Justin Domke.
\newblock Advances in black-box {{VI}}: {{Normalizing}} flows, importance
  weighting, and optimization.
\newblock In \emph{Advances in {{Neural Information Processing Systems}}},
  volume~33, pages 17358--17369. Curran Associates, Inc., 2020.

\bibitem[Alacaoglu et~al.(2025)Alacaoglu, Malitsky, and
  Wright]{alacaoglu_weaker_2025}
Ahmet Alacaoglu, Yura Malitsky, and Stephen~J. Wright.
\newblock Towards weaker variance assumptions for stochastic optimization.
\newblock {{arXiv}} Preprint arXiv:2504.09951, 2025.

\bibitem[Alquier and Ridgway(2020)]{alquier_concentration_2020}
Pierre Alquier and James Ridgway.
\newblock Concentration of tempered posteriors and of their variational
  approximations.
\newblock \emph{The Annals of Statistics}, 48\penalty0 (3):\penalty0
  1475--1497, 2020.

\bibitem[Arnese and Lacker(2024)]{arnese_convergence_2024}
Manuel Arnese and Daniel Lacker.
\newblock Convergence of coordinate ascent variational inference for
  log-concave measures via optimal transport.
\newblock {{arXiv}} Preprint arXiv:2404.08792, 2024.

\bibitem[Bach and Moulines(2011)]{bach_nonasymptotic_2011}
Francis Bach and Eric Moulines.
\newblock Non-asymptotic analysis of stochastic approximation algorithms for
  machine learning.
\newblock In \emph{Advances in {{Neural Information Processing Systems}}},
  volume~24, pages 451--459. Curran Associates, Inc., 2011.

\bibitem[Bhatia et~al.(2022)Bhatia, Kuang, Ma, and
  Wang]{bhatia_statistical_2022}
Kush Bhatia, Nikki~Lijing Kuang, Yi-An Ma, and Yixin Wang.
\newblock Statistical and computational trade-offs in variational inference: A
  case study in inferential model selection.
\newblock {{arXiv}} Preprint arXiv:2207.11208, 2022.

\bibitem[Bhattacharya et~al.(2025)Bhattacharya, Pati, and
  Yang]{bhattacharya_convergence_2025}
Anirban Bhattacharya, Debdeep Pati, and Yun Yang.
\newblock On the convergence of coordinate ascent variational inference.
\newblock \emph{The Annals of Statistics}, 53\penalty0 (3):\penalty0 929--962,
  2025.

\bibitem[Bingham et~al.(2019)Bingham, Chen, Jankowiak, Obermeyer, Pradhan,
  Karaletsos, Singh, Szerlip, Horsfall, and Goodman]{bingham_pyro_2019}
Eli Bingham, Jonathan~P. Chen, Martin Jankowiak, Fritz Obermeyer, Neeraj
  Pradhan, Theofanis Karaletsos, Rohit Singh, Paul Szerlip, Paul Horsfall, and
  Noah~D. Goodman.
\newblock Pyro: {{Deep}} universal probabilistic programming.
\newblock \emph{Journal of Machine Learning Research}, 20\penalty0
  (28):\penalty0 1--6, 2019.

\bibitem[Blei et~al.(2017)Blei, Kucukelbir, and
  McAuliffe]{blei_variational_2017}
David~M. Blei, Alp Kucukelbir, and Jon~D. McAuliffe.
\newblock Variational inference: {{A}} review for statisticians.
\newblock \emph{Journal of the American Statistical Association}, 112\penalty0
  (518):\penalty0 859--877, 2017.

\bibitem[Bottou(1999)]{bottou_online_1999}
L{\'e}on Bottou.
\newblock On-line learning and stochastic approximations.
\newblock In \emph{On-{{Line Learning}} in {{Neural Networks}}}, pages 9--42.
  Cambridge University Press, 1 edition, 1999.

\bibitem[Bottou et~al.(2018)Bottou, Curtis, and
  Nocedal]{bottou_optimization_2018}
L{\'e}on Bottou, Frank~E. Curtis, and Jorge Nocedal.
\newblock Optimization methods for large-scale machine learning.
\newblock \emph{SIAM Review}, 60\penalty0 (2):\penalty0 223--311, 2018.

\bibitem[Boustati et~al.(2020)Boustati, Vakili, Hensman, and
  John]{boustati_amortized_2020}
Ayman Boustati, Sattar Vakili, James Hensman, and S.~T. John.
\newblock Amortized variance reduction for doubly stochastic objective.
\newblock In \emph{Proceedings of the {{Conference}} on {{Uncertainty}} in
  {{Artificial Intelligence}}}, volume 124 of \emph{{{PMLR}}}, pages 61--70.
  JMLR, 2020.

\bibitem[Buchholz et~al.(2018)Buchholz, Wenzel, and
  Mandt]{buchholz_quasimonte_2018}
Alexander Buchholz, Florian Wenzel, and Stephan Mandt.
\newblock Quasi-{{Monte Carlo}} variational inference.
\newblock In \emph{Proceedings of the {{International Conference}} on {{Machine
  Learning}}}, volume~80 of \emph{{{PMLR}}}, pages 668--677. JMLR, 2018.

\bibitem[Byrd et~al.(2016)Byrd, Hansen, Nocedal, and
  Singer]{byrd_stochastic_2016}
R.~H. Byrd, S.~L. Hansen, Jorge Nocedal, and Y.~Singer.
\newblock A {{Stochastic Quasi-Newton Method}} for {{Large-Scale
  Optimization}}.
\newblock \emph{SIAM Journal on Optimization}, 26\penalty0 (2):\penalty0
  1008--1031, 2016.

\bibitem[Carpenter et~al.(2017)Carpenter, Gelman, Hoffman, Lee, Goodrich,
  Betancourt, Brubaker, Guo, Li, and Riddell]{carpenter_stan_2017}
Bob Carpenter, Andrew Gelman, Matthew~D. Hoffman, Daniel Lee, Ben Goodrich,
  Michael Betancourt, Marcus Brubaker, Jiqiang Guo, Peter Li, and Allen
  Riddell.
\newblock Stan: {{A}} probabilistic programming language.
\newblock \emph{Journal of Statistical Software}, 76\penalty0 (1):\penalty0
  1--32, 2017.

\bibitem[Casella and Berger(2001)]{casella_statistical_2001}
George Casella and Roger~L. Berger.
\newblock \emph{Statistical Inference}.
\newblock Cengage Learning, 2 edition, 2001.

\bibitem[Challis and Barber(2013)]{challis_gaussian_2013}
Edward Challis and David Barber.
\newblock Gaussian {{Kullback-Leibler}} approximate inference.
\newblock \emph{Journal of Machine Learning Research}, 14\penalty0
  (68):\penalty0 2239--2286, 2013.

\bibitem[Cheng et~al.(2024)Cheng, Yu, Xie, Zhang, and Zhang]{cheng_kernel_2024}
Ziheng Cheng, Longlin Yu, Tianyu Xie, Shiyue Zhang, and Cheng Zhang.
\newblock Kernel semi-implicit variational inference.
\newblock In \emph{Proceedings of the {{International Conference}} on {{Machine
  Learning}}}, volume 235 of \emph{{{PMLR}}}, pages 8248--8269. JMLR, 2024.

\bibitem[Chewi(2024)]{chewi_logconcave_2024}
Sinho Chewi.
\newblock \emph{Log-Concave Sampling}.
\newblock Unpublished draft, november 3, 2024 edition, 2024.
\newblock URL \url{https://chewisinho.github.io/main.pdf}.

\bibitem[Diao et~al.(2023)Diao, Balasubramanian, Chewi, and
  Salim]{diao_forwardbackward_2023}
Michael~Ziyang Diao, Krishna Balasubramanian, Sinho Chewi, and Adil Salim.
\newblock Forward-backward {{Gaussian}} variational inference via {{JKO}} in
  the {{Bures-Wasserstein}} space.
\newblock In \emph{Proceedings of the {{International Conference}} on {{Machine
  Learning}}}, volume 202 of \emph{{{PMLR}}}, pages 7960--7991. JMLR, 2023.

\bibitem[Domke(2019)]{domke_provable_2019}
Justin Domke.
\newblock Provable gradient variance guarantees for black-box variational
  inference.
\newblock In \emph{Advances in {{Neural Information Processing Systems}}},
  volume~32, pages 329--338. Curran Associates, Inc., 2019.

\bibitem[Domke(2020)]{domke_provable_2020}
Justin Domke.
\newblock Provable smoothness guarantees for black-box variational inference.
\newblock In \emph{Proceedings of the International Conference on Machine
  Learning}, volume 119 of \emph{{{PMLR}}}, pages 2587--2596. JMLR, 2020.

\bibitem[Domke et~al.(2023)Domke, Gower, and Garrigos]{domke_provable_2023}
Justin Domke, Robert Gower, and Guillaume Garrigos.
\newblock Provable convergence guarantees for black-box variational inference.
\newblock In \emph{Advances in Neural Information Processing Systems},
  volume~36, pages 66289--66327. Curran Associates, Inc., 2023.

\bibitem[Fan et~al.(2015)Fan, Wang, Beck, Kwok, and Heller]{fan_fast_2015}
Kai Fan, Ziteng Wang, Jeff Beck, James Kwok, and Katherine~A Heller.
\newblock Fast second order stochastic backpropagation for variational
  inference.
\newblock In \emph{Advances in {{Neural Information Processing Systems}}},
  volume~28, pages 1387--1395. Curran Associates, Inc., 2015.

\bibitem[Fjelde et~al.(2025)Fjelde, Xu, Widmann, Tarek, Pfiffer, Trapp, Axen,
  Sun, Hauru, Yong, Tebbutt, Ghahramani, and Ge]{fjelde_turingjl_2025}
Tor~Erlend Fjelde, Kai Xu, David Widmann, Mohamed Tarek, Cameron Pfiffer,
  Martin Trapp, Seth~D. Axen, Xianda Sun, Markus Hauru, Penelope Yong, Will
  Tebbutt, Zoubin Ghahramani, and Hong Ge.
\newblock Turing.jl: A general-purpose probabilistic programming language.
\newblock \emph{ACM Transactions on Probabilistic Machine Learning}, 1\penalty0
  (3):\penalty0 1--48, 2025.

\bibitem[Fujisawa and Sato(2021)]{fujisawa_multilevel_2021}
Masahiro Fujisawa and Issei Sato.
\newblock Multilevel {{Monte Carlo}} variational inference.
\newblock \emph{Journal of Machine Learning Research}, 22\penalty0
  (278):\penalty0 1--44, 2021.

\bibitem[Garrigos and Gower(2023)]{garrigos_handbook_2023}
Guillaume Garrigos and Robert~M. Gower.
\newblock Handbook of convergence theorems for (stochastic) gradient methods.
\newblock {{arXiv}} Preprint arXiv:2301.11235, 2023.

\bibitem[Ge et~al.(2018)Ge, Xu, and Ghahramani]{ge_turing_2018}
Hong Ge, Kai Xu, and Zoubin Ghahramani.
\newblock Turing: A language for flexible probabilistic inference.
\newblock In \emph{Proceedings of the {{International Conference}} on {{Machine
  Learning}}}, volume~84 of \emph{{{PMLR}}}, pages 1682--1690. JMLR, 2018.

\bibitem[Geffner and Domke(2018)]{geffner_using_2018}
Tomas Geffner and Justin Domke.
\newblock Using large ensembles of control variates for variational inference.
\newblock In \emph{Advances in {{Neural Information Processing Systems}}},
  volume~31, pages 9960--9970. Curran Associates, Inc., 2018.

\bibitem[Geffner and Domke(2020)]{geffner_approximation_2020}
Tomas Geffner and Justin Domke.
\newblock Approximation {{Based Variance Reduction}} for {{Reparameterization
  Gradients}}.
\newblock In \emph{Advances in {{Neural Information Processing Systems}}},
  volume~33, pages 2397--2407. Curran Associates, Inc., 2020.

\bibitem[Ghorbani et~al.(2019)Ghorbani, Javadi, and
  Montanari]{ghorbani_instability_2019}
Behrooz Ghorbani, Hamid Javadi, and Andrea Montanari.
\newblock An instability in variational inference for topic models.
\newblock In \emph{Proceedings of the {{International Conference}} on {{Machine
  Learning}}}, volume~97 of \emph{{{PMLR}}}, pages 2221--2231. JMLR, 2019.

\bibitem[Giordano et~al.(2018)Giordano, Broderick, and
  Jordan]{giordano_covariances_2018}
Ryan Giordano, Tamara Broderick, and Michael~I. Jordan.
\newblock Covariances, robustness, and variational {{Bayes}}.
\newblock \emph{Journal of Machine Learning Research}, 19\penalty0
  (51):\penalty0 1--49, 2018.

\bibitem[Giordano et~al.(2024)Giordano, Ingram, and
  Broderick]{giordano_black_2024}
Ryan Giordano, Martin Ingram, and Tamara Broderick.
\newblock Black box variational inference with a deterministic objective:
  {{Faster}}, more accurate, and even more black box.
\newblock \emph{Journal of Machine Learning Research}, 25:\penalty0 1--39,
  2024.

\bibitem[Glasserman(1991)]{glasserman_gradient_1991}
Paul Glasserman.
\newblock \emph{Gradient Estimation via Perturbation Analysis}.
\newblock Number 116 in The {{Springer International Series}} in
  {{Engineering}} and {{Computer Science}}. Springer, New York, NY, 1991.

\bibitem[Glynn(1990)]{glynn_likelihood_1990}
Peter~W. Glynn.
\newblock Likelihood ratio gradient estimation for stochastic systems.
\newblock \emph{Communications of the ACM}, 33\penalty0 (10):\penalty0 75--84,
  1990.

\bibitem[Gorbunov et~al.(2020)Gorbunov, Hanzely, and
  Richtarik]{gorbunov_unified_2020}
Eduard Gorbunov, Filip Hanzely, and Peter Richtarik.
\newblock A unified theory of {{SGD}}: {{Variance}} reduction, sampling,
  quantization and coordinate descent.
\newblock In \emph{Proceedings of the {{International Conference}} on
  {{Artificial Intelligence}} and {{Statistics}}}, volume 108 of
  \emph{{{PMLR}}}, pages 680--690. JMLR, 2020.

\bibitem[Gower et~al.(2019)Gower, Loizou, Qian, Sailanbayev, Shulgin, and
  Richt{\'a}rik]{gower_sgd_2019}
Robert~Mansel Gower, Nicolas Loizou, Xun Qian, Alibek Sailanbayev, Egor
  Shulgin, and Peter Richt{\'a}rik.
\newblock {{SGD}}: {{General}} analysis and improved rates.
\newblock In \emph{Proceedings of the International Conference on Machine
  Learning}, volume~97 of \emph{{{PMLR}}}, pages 5200--5209. JMLR, 2019.

\bibitem[Gumbel(1954)]{gumbel_maxima_1954}
E.~J. Gumbel.
\newblock The maxima of the mean largest value and of the range.
\newblock \emph{The Annals of Mathematical Statistics}, 25\penalty0
  (1):\penalty0 76--84, 1954.

\bibitem[Halpern(1967)]{halpern_fixed_1967}
Benjamin Halpern.
\newblock Fixed points of nonexpanding maps.
\newblock \emph{Bulletin of the American Mathematical Society}, 73\penalty0
  (6):\penalty0 957--961, 1967.

\bibitem[Hinton and {van Camp}(1993)]{hinton_keeping_1993}
Geoffrey~E. Hinton and Drew {van Camp}.
\newblock Keeping the neural networks simple by minimizing the description
  length of the weights.
\newblock In \emph{Proceedings of the Annual Conference on {{Computational}}
  Learning Theory}, pages 5--13. ACM Press, 1993.

\bibitem[Ho and Cao(1983)]{ho_perturbation_1983}
Y.~C. Ho and X.~Cao.
\newblock Perturbation analysis and optimization of queueing networks.
\newblock \emph{Journal of Optimization Theory and Applications}, 40\penalty0
  (4):\penalty0 559--582, 1983.

\bibitem[Hoffman and Ma(2020)]{hoffman_blackbox_2020}
Matthew Hoffman and Yian Ma.
\newblock Black-box variational inference as a parametric approximation to
  {{Langevin}} dynamics.
\newblock In \emph{Proceedings of the {{International Conference}} on {{Machine
  Learning}}}, volume 119 of \emph{{{PMLR}}}, pages 4324--4341. JMLR, 2020.

\bibitem[Hotti et~al.(2024)Hotti, der Goten, and
  Lagergren]{hotti_benefits_2024}
Alexandra~Maria Hotti, Lennart Alexander~Van der Goten, and Jens Lagergren.
\newblock Benefits of {{Non-Linear Scale Parameterizations}} in {{Black Box
  Variational Inference}} through {{Smoothness Results}} and {{Gradient
  Variance Bounds}}.
\newblock In \emph{Proceedings of the {{International Conference}} on
  {{Artificial Intelligence}} and {{Statistics}}}, volume 238 of
  \emph{{{PMLR}}}, pages 3538--3546. JMLR, 2024.

\bibitem[Jiang et~al.(2025)Jiang, Chewi, and Pooladian]{jiang_algorithms_2025}
Yiheng Jiang, Sinho Chewi, and Aram-Alexandre Pooladian.
\newblock Algorithms for mean-field variational inference via polyhedral
  optimization in the {{Wasserstein}} space.
\newblock \emph{Foundations of Computational Mathematics}, 2025.

\bibitem[Johnson et~al.(1994)Johnson, Kotz, and
  Balakrishnan]{johnson_continuous_1994}
Norman~L. Johnson, Samuel Kotz, and Narayanaswamy Balakrishnan.
\newblock Continuous univariate distributions.
\newblock volume~1 of \emph{Wiley Series in Probability and Mathematical
  Statistics}. Wiley, New York, 2 edition, 1994.

\bibitem[Johnson et~al.(1995)Johnson, Kotz, and
  Balakrishnan]{johnson_continuous_1995}
Norman~L. Johnson, Samuel Kotz, and Narayanaswamy Balakrishnan.
\newblock Continuous univariate distributions.
\newblock volume~2 of \emph{Wiley Series in Probability and Mathematical
  Statistics}. Wiley, New York, 2 edition, 1995.

\bibitem[Jordan et~al.(1999)Jordan, Ghahramani, Jaakkola, and
  Saul]{jordan_introduction_1999}
Michael~I. Jordan, Zoubin Ghahramani, Tommi~S. Jaakkola, and Lawrence~K. Saul.
\newblock An introduction to variational methods for graphical models.
\newblock \emph{Machine Learning}, 37\penalty0 (2):\penalty0 183--233, 1999.

\bibitem[Katsevich and Rigollet(2024)]{katsevich_approximation_2024}
Anya Katsevich and Philippe Rigollet.
\newblock On the approximation accuracy of {{Gaussian}} variational inference.
\newblock \emph{The Annals of Statistics}, 52\penalty0 (4):\penalty0
  1384--1409, 2024.

\bibitem[Khaled et~al.(2023)Khaled, Sebbouh, Loizou, Gower, and
  Richt{\'a}rik]{khaled_unified_2023}
Ahmed Khaled, Othmane Sebbouh, Nicolas Loizou, Robert~M. Gower, and Peter
  Richt{\'a}rik.
\newblock Unified analysis of stochastic gradient methods for composite convex
  and smooth optimization.
\newblock \emph{Journal of Optimization Theory and Applications}, 199:\penalty0
  499--540, 2023.

\bibitem[Khan et~al.(2016)Khan, Babanezhad, Lin, Schmidt, and
  Sugiyama]{khan_faster_2016}
Mohammad~Emtiyaz Khan, Reza Babanezhad, Wu~Lin, Mark Schmidt, and Masashi
  Sugiyama.
\newblock Faster stochastic variational inference using proximal-gradient
  methods with general divergence functions.
\newblock In \emph{Proceedings of the Conference on Uncertainty in Artificial
  Intelligence}, {{UAI}}'16, pages 319--328, Jersey City, New Jersey, USA,
  2016. AUAI Press.

\bibitem[Kim et~al.(2023{\natexlab{a}})Kim, Oh, Wu, Ma, and
  Gardner]{kim_convergence_2023}
Kyurae Kim, Jisu Oh, Kaiwen Wu, Yian Ma, and Jacob~R. Gardner.
\newblock On the convergence of black-box variational inference.
\newblock In \emph{Advances in {{Neural Information Processing Systems}}},
  volume~36, pages 44615--44657. Curran Associates Inc., 2023{\natexlab{a}}.

\bibitem[Kim et~al.(2023{\natexlab{b}})Kim, Wu, Oh, and
  Gardner]{kim_practical_2023}
Kyurae Kim, Kaiwen Wu, Jisu Oh, and Jacob~R. Gardner.
\newblock Practical and matching gradient variance bounds for black-box
  variational {{Bayesian}} inference.
\newblock In \emph{Proceedings of the {{International Conference}} on {{Machine
  Learning}}}, volume 202 of \emph{{{PMLR}}}, pages 16853--16876. JMLR,
  2023{\natexlab{b}}.

\bibitem[Kim et~al.(2024{\natexlab{a}})Kim, Ko, Ma, and
  Gardner]{kim_demystifying_2024}
Kyurae Kim, Joohwan Ko, Yi-An Ma, and Jacob~R. Gardner.
\newblock Demystifying {{SGD}} with doubly stochastic gradients.
\newblock In \emph{Proceedings of the {{International Conference}} on {{Machine
  Learning}}}, volume 235 of \emph{{{PMLR}}}, pages 24210--24247. JMLR,
  2024{\natexlab{a}}.

\bibitem[Kim et~al.(2024{\natexlab{b}})Kim, Ma, and Gardner]{kim_linear_2024}
Kyurae Kim, Yian Ma, and Jacob~R. Gardner.
\newblock Linear convergence of black-box variational inference: Should we
  stick the landing?
\newblock In \emph{Proceedings of the {{International Conference}} on
  {{Artificial Intelligence}} and {{Statistics}}}, volume 238 of
  \emph{{{PMLR}}}, pages 235--243. JMLR, 2024{\natexlab{b}}.

\bibitem[Kingma and Welling(2014)]{kingma_autoencoding_2014}
Diederik~P. Kingma and Max Welling.
\newblock Auto-encoding variational {{Bayes}}.
\newblock In \emph{Proceedings of the {{International Conference}} on
  {{Learning Representations}}}, Banff, AB, Canada, 2014.

\bibitem[Ko et~al.(2024)Ko, Kim, Kim, and Gardner]{ko_provably_2024}
Joohwan Ko, Kyurae Kim, Woo~Chang Kim, and Jacob~R. Gardner.
\newblock Provably scalable black-box variational inference with structured
  variational families.
\newblock In \emph{Proceedings of the {{International Conference}} on {{Machine
  Learning}}}, volume 235 of \emph{{{PMLR}}}, pages 24896--24931. JMLR, 2024.

\bibitem[Kucukelbir et~al.(2017)Kucukelbir, Tran, Ranganath, Gelman, and
  Blei]{kucukelbir_automatic_2017}
Alp Kucukelbir, Dustin Tran, Rajesh Ranganath, Andrew Gelman, and David~M.
  Blei.
\newblock Automatic differentiation variational inference.
\newblock \emph{Journal of Machine Learning Research}, 18\penalty0
  (14):\penalty0 1--45, 2017.

\bibitem[Kullback and Leibler(1951)]{kullback_information_1951}
S.~Kullback and R.~A. Leibler.
\newblock On information and sufficiency.
\newblock \emph{The Annals of Mathematical Statistics}, 22\penalty0
  (1):\penalty0 79--86, 1951.

\bibitem[{Lacoste-Julien} et~al.(2012){Lacoste-Julien}, Schmidt, and
  Bach]{lacoste-julien_simpler_2012}
Simon {Lacoste-Julien}, Mark Schmidt, and Francis Bach.
\newblock A simpler approach to obtaining an $\mathrm{O}(1/t)$ convergence rate
  for the projected stochastic subgradient method.
\newblock {{arXiv}} Preprint arXiv:1212.2002, 2012.

\bibitem[Lambert et~al.(2022)Lambert, Chewi, Bach, Bonnabel, and
  Rigollet]{lambert_variational_2022}
Marc Lambert, Sinho Chewi, Francis Bach, Silv{\`e}re Bonnabel, and Philippe
  Rigollet.
\newblock Variational inference via {{Wasserstein}} gradient flows.
\newblock In \emph{Advances in {{Neural Information Processing Systems}}},
  volume~35, pages 14434--14447. Curran Associates, Inc., 2022.

\bibitem[Lavenant and Zanella(2024)]{lavenant_convergence_2024}
Hugo Lavenant and Giacomo Zanella.
\newblock Convergence rate of random scan coordinate ascent variational
  inference under log-concavity.
\newblock \emph{SIAM Journal on Optimization}, 34\penalty0 (4):\penalty0
  3750--3761, 2024.

\bibitem[Liu and Owen(2021)]{liu_quasimonte_2021}
Sifan Liu and Art~B. Owen.
\newblock Quasi-{{Monte Carlo}} quasi-{{Newton}} in {{Variational Bayes}}.
\newblock \emph{Journal of Machine Learning Research}, 22\penalty0
  (243):\penalty0 1--23, 2021.

\bibitem[Margossian and Saul(2023)]{margossian_shrinkagedelinkage_2023}
Charles~C. Margossian and Lawrence~K. Saul.
\newblock The shrinkage-delinkage trade-off: {{An}} analysis of factorized
  {{Gaussian}} approximations for variational inference.
\newblock In \emph{Proceedings of the {{Conference}} on {{Uncertainty}} in
  {{Artificial Intelligence}}}, volume 216 of \emph{{{PMLR}}}, pages
  1358--1367. JMLR, 2023.

\bibitem[Margossian and Saul(2025)]{margossian_variational_2025}
Charles~C. Margossian and Lawrence~K. Saul.
\newblock Variational {{Inference}} in {{Location-Scale Families}}: {{Exact
  Recovery}} of the {{Mean}} and {{Correlation Matrix}}.
\newblock In \emph{Proceedings of the {{International Conference}} on
  {{Artificial Intelligence}} and {{Statistics}}}, volume 258 of
  \emph{{{PMLR}}}, pages 3466--3474. JMLR, 2025.

\bibitem[Meng et~al.(2020)Meng, Vaswani, Laradji, Schmidt, and
  {Lacoste-Julien}]{meng_fast_2020}
Si~Yi Meng, Sharan Vaswani, Issam~Hadj Laradji, Mark Schmidt, and Simon
  {Lacoste-Julien}.
\newblock Fast and {{Furious Convergence}}: {{Stochastic Second Order Methods}}
  under {{Interpolation}}.
\newblock In \emph{Proceedings of the {{International Conference}} on
  {{Artificial Intelligence}} and {{Statistics}}}, volume 108 of
  \emph{{{PMLR}}}, pages 1375--1386. JMLR, 2020.

\bibitem[Miller et~al.(2017)Miller, Foti, D'~Amour, and
  Adams]{miller_reducing_2017}
Andrew Miller, Nick Foti, Alexander D'~Amour, and Ryan~P Adams.
\newblock Reducing reparameterization gradient variance.
\newblock In \emph{Advances in {{Neural Information Processing Systems}}},
  volume~30, pages 3708--3718. Curran Associates, Inc., 2017.

\bibitem[Modi et~al.(2025)Modi, Cai, and Saul]{modi_batch_2025}
Chirag Modi, Diana Cai, and Lawrence~K. Saul.
\newblock Batch, match, and patch: {{Low-rank}} approximations for score-based
  variational inference.
\newblock In \emph{Proceedings of the {{International Conference}} on
  {{Artificial Intelligence}} and {{Statistics}}}, volume 258 of
  \emph{{{PMLR}}}, pages 4510--4518. JMLR, 2025.

\bibitem[Mohamed et~al.(2020)Mohamed, Rosca, Figurnov, and
  Mnih]{mohamed_monte_2020}
Shakir Mohamed, Mihaela Rosca, Michael Figurnov, and Andriy Mnih.
\newblock Monte {{Carlo}} gradient estimation in machine learning.
\newblock \emph{Journal of Machine Learning Research}, 21\penalty0
  (132):\penalty0 1--62, 2020.

\bibitem[Nemirovski et~al.(2009)Nemirovski, Juditsky, Lan, and
  Shapiro]{nemirovski_robust_2009}
A.~Nemirovski, A.~Juditsky, G.~Lan, and A.~Shapiro.
\newblock Robust stochastic approximation approach to stochastic programming.
\newblock \emph{SIAM Journal on Optimization}, 19\penalty0 (4):\penalty0
  1574--1609, 2009.

\bibitem[Nguyen et~al.(2025)Nguyen, Sakurai, and
  Mamitsuka]{nguyen_wasserstein_2025}
Dai~Hai Nguyen, Tetsuya Sakurai, and Hiroshi Mamitsuka.
\newblock Wasserstein gradient flow over variational parameter space for
  variational inference.
\newblock In \emph{Proceedings of the {{International Conference}} on
  {{Artificial Intelligence}} and {{Statistic}}}, volume 258 of
  \emph{{{PMLR}}}, pages 1756--1764. JMLR, 2025.

\bibitem[Ong et~al.(2018)Ong, Nott, and Smith]{ong_gaussian_2018}
Victor M.-H. Ong, David~J. Nott, and Michael~S. Smith.
\newblock Gaussian variational approximation with a factor covariance
  structure.
\newblock \emph{Journal of Computational and Graphical Statistics}, 27\penalty0
  (3):\penalty0 465--478, 2018.

\bibitem[Parikh and Boyd(2014)]{parikh_proximal_2014}
Neal Parikh and Stephen~P. Boyd.
\newblock \emph{Proximal Algorithms}, volume~1 of \emph{Foundations and
  {{Trends}}{\textregistered} in {{Optimization}}}.
\newblock Now Publishers, Norwell, MA, 2014.

\bibitem[Patil et~al.(2010)Patil, Huard, and Fonnesbeck]{patil_pymc_2010}
Anand Patil, David Huard, and Christopher Fonnesbeck.
\newblock {{PyMC}}: {{Bayesian}} stochastic modelling in {{Python}}.
\newblock \emph{Journal of Statistical Software}, 35\penalty0 (4):\penalty0
  1--81, 2010.

\bibitem[Peterson and Hartman(1989)]{peterson_explorations_1989}
Carsten Peterson and Eric Hartman.
\newblock Explorations of the mean field theory learning algorithm.
\newblock \emph{Neural Networks}, 2\penalty0 (6):\penalty0 475--494, 1989.

\bibitem[Pflug(1996)]{pflug_optimization_1996}
Georg Pflug.
\newblock \emph{Optimization of Stochastic Models: The Interface between
  Simulation and Optimization}.
\newblock Number v.373 in The {{Springer International Series}} in
  {{Engineering}} and {{Computer Science Ser}}. Springer, New York, NY, 1996.

\bibitem[Rana(2017)]{rana_answer_2017}
Rana.
\newblock Answer to ``{{A}} bound for the expectation of the maximum
  independent random variables'', 2017.
\newblock URL \url{https://math.stackexchange.com/q/2177201}.

\bibitem[Ranganath et~al.(2014)Ranganath, Gerrish, and
  Blei]{ranganath_black_2014}
Rajesh Ranganath, Sean Gerrish, and David Blei.
\newblock Black box variational inference.
\newblock In \emph{Proceedings of the International Conference on Artificial
  Intelligence and Statistics}, volume~33 of \emph{{{PMLR}}}, pages 814--822.
  JMLR, 2014.

\bibitem[Regier et~al.(2017)Regier, Jordan, and McAuliffe]{regier_fast_2017}
Jeffrey Regier, Michael~I Jordan, and Jon McAuliffe.
\newblock Fast black-box variational inference through stochastic trust-region
  optimization.
\newblock In \emph{Advances in {{Neural Information Processing Systems}}},
  volume~30, pages 2399--2408. Curran Associates, Inc., 2017.

\bibitem[Rezende et~al.(2014)Rezende, Mohamed, and
  Wierstra]{rezende_stochastic_2014}
Danilo~Jimenez Rezende, Shakir Mohamed, and Daan Wierstra.
\newblock Stochastic backpropagation and approximate inference in deep
  generative models.
\newblock In \emph{Proceedings of the {{International Conference}} on {{Machine
  Learning}}}, volume~32 of \emph{{{PMLR}}}, pages 1278--1286. JMLR, 2014.

\bibitem[Robbins and Monro(1951)]{robbins_stochastic_1951}
Herbert Robbins and Sutton Monro.
\newblock A stochastic approximation method.
\newblock \emph{The Annals of Mathematical Statistics}, 22\penalty0
  (3):\penalty0 400--407, 1951.

\bibitem[Roeder et~al.(2017)Roeder, Wu, and Duvenaud]{roeder_sticking_2017}
Geoffrey Roeder, Yuhuai Wu, and David~K Duvenaud.
\newblock Sticking the landing: {{Simple}}, lower-variance gradient estimators
  for variational inference.
\newblock In \emph{Advances in {{Neural Information Processing Systems}}},
  volume~30, pages 6928--6937. Curran Associates, Inc., 2017.

\bibitem[Rubinstein(1992)]{rubinstein_sensitivity_1992}
Reuven~Y. Rubinstein.
\newblock Sensitivity analysis of discrete event systems by the ``push out''
  method.
\newblock \emph{Annals of Operations Research}, 39\penalty0 (1):\penalty0
  229--250, 1992.

\bibitem[Schmidt and Roux(2013)]{schmidt_fast_2013}
Mark Schmidt and Nicolas~Le Roux.
\newblock Fast convergence of stochastic gradient descent under a strong growth
  condition.
\newblock {{arXiv}} Preprint arXiv:1308.6370, 2013.

\bibitem[{Shalev-Shwartz} et~al.(2011){Shalev-Shwartz}, Singer, Srebro, and
  Cotter]{shalev-shwartz_pegasos_2011}
Shai {Shalev-Shwartz}, Yoram Singer, Nathan Srebro, and Andrew Cotter.
\newblock Pegasos: Primal estimated sub-gradient solver for {{SVM}}.
\newblock \emph{Mathematical Programming}, 127\penalty0 (1):\penalty0 3--30,
  2011.

\bibitem[Surendran et~al.(2025)Surendran, {Godichon-Baggioni}, and
  Corff]{surendran_theoretical_2025}
Sobihan Surendran, Antoine {Godichon-Baggioni}, and Sylvain~Le Corff.
\newblock Theoretical {{Convergence Guarantees}} for {{Variational
  Autoencoders}}.
\newblock In \emph{Proceedings of the {{International Conference}} on
  {{Artificial Intelligence}} and {{Statistics}}}, volume 258 of
  \emph{{{PMLR}}}. JMLR, 2025.

\bibitem[{Symbol-1}(2022)]{symbol-1_answer_2022}
{Symbol-1}.
\newblock Answer to ``{{Meaningful}} lower-bound of $\sqrt{a^2+b} - a$ when $a
  \gg b>0$'', 2022.
\newblock URL \url{https://math.stackexchange.com/q/4360503}.

\bibitem[Titsias and {L{\'a}zaro-Gredilla}(2014)]{titsias_doubly_2014}
Michalis Titsias and Miguel {L{\'a}zaro-Gredilla}.
\newblock Doubly stochastic variational {{Bayes}} for non-conjugate inference.
\newblock In \emph{Proceedings of the {{International Conference}} on {{Machine
  Learning}}}, volume~32 of \emph{{{PMLR}}}, pages 1971--1979. JMLR, 2014.

\bibitem[Tomczak et~al.(2020)Tomczak, Swaroop, and
  Turner]{tomczak_efficient_2020}
Marcin Tomczak, Siddharth Swaroop, and Richard Turner.
\newblock Efficient low rank {{Gaussian}} variational inference for neural
  networks.
\newblock In \emph{Advances in {{Neural Information Processing Systems}}},
  volume~33, pages 4610--4622. Curran Associates, Inc., 2020.

\bibitem[Trefethen and Bau(1997)]{trefethen_numerical_1997}
Lloyd~N. Trefethen and David Bau.
\newblock \emph{Numerical Linear Algebra}.
\newblock {Society for Industrial and Applied Mathematics}, Philadelphia, 1997.

\bibitem[Vaswani et~al.(2019)Vaswani, Bach, and Schmidt]{vaswani_fast_2019}
Sharan Vaswani, Francis Bach, and Mark Schmidt.
\newblock Fast and faster convergence of {{SGD}} for over-parameterized models
  and an accelerated perceptron.
\newblock In \emph{Proceedings of the {{International Conference}} on
  {{Artificial Intelligence}} and {{Statistics}}}, volume~89 of
  \emph{{{PMLR}}}, pages 1195--1204. JMLR, 2019.

\bibitem[Wainwright(2019)]{wainwright_highdimensional_2019}
Martin~J. Wainwright.
\newblock \emph{High-Dimensional Statistics: A Non-Asymptotic Viewpoint}.
\newblock Cambridge {{Series}} in {{Statistical}} and {{Probabilistic
  Mathematics}}. Cambridge University Press, New York, NY, 1st ed edition,
  2019.

\bibitem[Wang et~al.(2024)Wang, Geffner, and Domke]{wang_joint_2024}
Xi~Wang, Tomas Geffner, and Justin Domke.
\newblock Joint control variate for faster black-box variational inference.
\newblock In \emph{Proceedings of the {{International Conference}} on
  {{Artificial Intelligence}} and {{Statistics}}}, volume 238 of
  \emph{{{PMLR}}}, pages 1639--1647. JMLR, 2024.

\bibitem[Wang and Blei(2019{\natexlab{a}})]{wang_variational_2019}
Yixin Wang and David Blei.
\newblock Variational bayes under model misspecification.
\newblock In \emph{Advances in {{Neural Information Processing Systems}}},
  volume~32, pages 13357--13367. Curran Associates, Inc., 2019{\natexlab{a}}.

\bibitem[Wang and Blei(2019{\natexlab{b}})]{wang_frequentist_2019}
Yixin Wang and David~M. Blei.
\newblock Frequentist consistency of variational {{Bayes}}.
\newblock \emph{Journal of the American Statistical Association}, 114\penalty0
  (527):\penalty0 1147--1161, 2019{\natexlab{b}}.

\bibitem[Williams(1992)]{williams_simple_1992}
Ronald~J. Williams.
\newblock Simple statistical gradient-following algorithms for connectionist
  reinforcement learning.
\newblock \emph{Machine Learning}, 8\penalty0 (3):\penalty0 229--256, 1992.

\bibitem[Wingate and Weber(2013)]{wingate_automated_2013}
David Wingate and Theophane Weber.
\newblock Automated variational inference in probabilistic programming.
\newblock {{arXiv}} Preprint arXiv:1301.1299, 2013.

\bibitem[Xu et~al.(2019)Xu, Quiroz, Kohn, and Sisson]{xu_variance_2019}
Ming Xu, Matias Quiroz, Robert Kohn, and Scott~A. Sisson.
\newblock Variance reduction properties of the reparameterization trick.
\newblock In \emph{Proceedings of the {{International Conference}} on
  {{Artificial Intelligence}} and {{Statistics}}}, volume~89 of
  \emph{{{PMLR}}}, pages 2711--2720. JMLR, 2019.

\bibitem[Xu and Campbell(2022)]{xu_computational_2022}
Zuheng Xu and Trevor Campbell.
\newblock The computational asymptotics of {{Gaussian}} variational inference
  and the {{Laplace}} approximation.
\newblock \emph{Statistics and Computing}, 32\penalty0 (4), 2022.

\bibitem[Yang et~al.(2020)Yang, Pati, and Bhattacharya]{yang_avariational_2020}
Yun Yang, Debdeep Pati, and Anirban Bhattacharya.
\newblock {$\alpha$}-variational inference with statistical guarantees.
\newblock \emph{The Annals of Statistics}, 48\penalty0 (2):\penalty0 886--905,
  2020.

\bibitem[Zhang and Zhou(2020)]{zhang_theoretical_2020}
Anderson~Y. Zhang and Harrison~H. Zhou.
\newblock Theoretical and computational guarantees of mean field variational
  inference for community detection.
\newblock \emph{The Annals of Statistics}, 48\penalty0 (5):\penalty0
  2575--2598, 2020.

\bibitem[Zhang and Gao(2020)]{zhang_convergence_2020}
Fengshuo Zhang and Chao Gao.
\newblock Convergence rates of variational posterior distributions.
\newblock \emph{The Annals of Statistics}, 48\penalty0 (4):\penalty0
  2180--2207, 2020.

\bibitem[Zhang et~al.(2022)Zhang, Carpenter, Gelman, and
  Vehtari]{zhang_pathfinder_2022}
Lu~Zhang, Bob Carpenter, Andrew Gelman, and Aki Vehtari.
\newblock Pathfinder: {{Parallel}} quasi-{{Newton}} variational inference.
\newblock \emph{Journal of Machine Learning Research}, 23\penalty0
  (306):\penalty0 1--49, 2022.

\end{thebibliography}

\newpage

\newpage
\section*{NeurIPS Paper Checklist}

\begin{enumerate}

\item {\bf Claims}
    \item[] Question: Do the main claims made in the abstract and introduction accurately reflect the paper's contributions and scope?
    \item[] Answer: \answerYes{}
    \item[] Justification: Yes, we obtain a $\mathrm{O}( (\log d) \kappa^2 \epsilon^{-1})$ iteration complexity result for BBVI on strong log-concave and log-smooth target distributions. 
    This corresponds to a nearly dimension-independent iteration complexity.

\item {\bf Limitations}
    \item[] Question: Does the paper discuss the limitations of the work performed by the authors?
    \item[] Answer: \answerYes{} 
    \item[] Justification: As stated in \cref{remark:limitation}, the main limitation of our work is that the unimprovability result~\cref{thm:gradient_variance_lowerbound} does not fully assert that \cref{thm:gradient_variance_upperbound_meanfield_general} is tight for any target function $\ell$. 
    It only shows that our specific proof strategy, which uses only spectral bounds on the Hessian of $\ell$, is unimprovable.
    In principle, using additional properties of the Hessian could result in a tighter bound.
    

\item {\bf Theory assumptions and proofs}
    \item[] Question: For each theoretical result, does the paper provide the full set of assumptions and a complete (and correct) proof?
    \item[] Answer: \answerYes{}
    \item[] Justification: All assumptions are stated in either \cref{section:background,section:main_results} or the propositional statements.

    \item {\bf Experimental result reproducibility}
    \item[] Question: Does the paper fully disclose all the information needed to reproduce the main experimental results of the paper to the extent that it affects the main claims and/or conclusions of the paper (regardless of whether the code and data are provided or not)?
    \item[] Answer: \answerNA{}
    \item[] Justification: The paper does not contain any experiments.

\item {\bf Open access to data and code}
    \item[] Question: Does the paper provide open access to the data and code, with sufficient instructions to faithfully reproduce the main experimental results, as described in supplemental material?
    \item[] Answer: \answerNA{}
    \item[] Justification: The paper does not contain any experiments.

\item {\bf Experimental setting/details}
    \item[] Question: Does the paper specify all the training and test details (e.g., data splits, hyperparameters, how they were chosen, type of optimizer, etc.) necessary to understand the results?
    \item[] Answer: \answerNA{}
    \item[] Justification: The paper does not contain any experiments.

\item {\bf Experiment statistical significance}
    \item[] Question: Does the paper report error bars suitably and correctly defined or other appropriate information about the statistical significance of the experiments?
    \item[] Answer: \answerNA{}
    \item[] Justification: The paper does not contain any experiments.

\item {\bf Experiments compute resources}
    \item[] Question: For each experiment, does the paper provide sufficient information on the computer resources (type of compute workers, memory, time of execution) needed to reproduce the experiments?
    \item[] Answer: \answerNA{}
    \item[] Justification: The paper does not contain any experiments.
    
\item {\bf Code of ethics}
    \item[] Question: Does the research conducted in the paper conform, in every respect, with the NeurIPS Code of Ethics \url{https://neurips.cc/public/EthicsGuidelines}?
    \item[] Answer: \answerYes{}
    \item[] Justification:
    \item[] Guidelines: The content of the paper is a theoretical study of an inference algorithm and does not involve real data.
    \begin{itemize}
        \item The answer NA means that the authors have not reviewed the NeurIPS Code of Ethics.
        \item If the authors answer No, they should explain the special circumstances that require a deviation from the Code of Ethics.
        \item The authors should make sure to preserve anonymity (e.g., if there is a special consideration due to laws or regulations in their jurisdiction).
    \end{itemize}

\item {\bf Broader impacts}
    \item[] Question: Does the paper discuss both potential positive societal impacts and negative societal impacts of the work performed?
    \item[] Answer: \answerNo{}
    \item[] Justification: The content of the paper is a theoretical study of an inference algorithm and does not have direct societal consequences.
    
\item {\bf Safeguards}
    \item[] Question: Does the paper describe safeguards that have been put in place for responsible release of data or models that have a high risk for misuse (e.g., pretrained language models, image generators, or scraped datasets)?
    \item[] Answer: \answerNA{}
    \item[] Justification: The paper does not involve real data or models.

\item {\bf Licenses for existing assets}
    \item[] Question: Are the creators or original owners of assets (e.g., code, data, models), used in the paper, properly credited and are the license and terms of use explicitly mentioned and properly respected?
    \item[] Answer: \answerNA{}
    \item[] Justification: The paper does not involve real data.

\item {\bf New assets}
    \item[] Question: Are new assets introduced in the paper well documented and is the documentation provided alongside the assets?
    \item[] Answer: \answerNA{}
    \item[] Justification: The paper does not involve real data.

\item {\bf Crowdsourcing and research with human subjects}
    \item[] Question: For crowdsourcing experiments and research with human subjects, does the paper include the full text of instructions given to participants and screenshots, if applicable, as well as details about compensation (if any)? 
    \item[] Answer: \answerNA{}
    \item[] Justification: The paper does not involve human subjects.

\item {\bf Institutional review board (IRB) approvals or equivalent for research with human subjects}
    \item[] Question: Does the paper describe potential risks incurred by study participants, whether such risks were disclosed to the subjects, and whether Institutional Review Board (IRB) approvals (or an equivalent approval/review based on the requirements of your country or institution) were obtained?
    \item[] Answer: \answerNA{}
    \item[] Justification: The paper does not involve human subjects.

\item {\bf Declaration of LLM usage}
    \item[] Question: Does the paper describe the usage of LLMs if it is an important, original, or non-standard component of the core methods in this research? Note that if the LLM is used only for writing, editing, or formatting purposes and does not impact the core methodology, scientific rigorousness, or originality of the research, declaration is not required.
    \item[] Answer: \answerNo{}
    \item[] Justification: Part of the supporting results were obtained after some minor interaction with LLMs. However, all of the proofs were written and proofread by humans.
    Therefore, LLMs did not play an important, original role nor did they contribute any non-standard components.

\end{enumerate}

\newpage
\appendix 

\renewcommand{\baselinestretch}{0.75}\normalsize
{\hypersetup{linkbordercolor=black,linkcolor=black}
\tableofcontents
}
\renewcommand{\baselinestretch}{1.0}\normalsize

\newpage

\section{Auxiliary Lemmas}

\begin{lemma}\label{thm:kurtosis_bound}
    Suppose \cref{assumption:noise} holds.
    Then $r_4 = \mathbb{E}\mathsfit{u}_i^4 \geq 1$.
\end{lemma}
\begin{proof}
    By Jensen's inequality $\mathbb{E}\mathsfit{u}_i^4 \geq {\lt(\mathbb{E}\mathsfit{u}_i^2\rt)}^2$.
    Lastly, $\mathbb{E}\mathsfit{u}_i^4 \geq {(\mathbb{E}\mathsfit{u}_i^2)}^2 = 1$ by \cref{assumption:noise}.
\end{proof}

\begin{lemma}\label{thm:noise}
    Suppose \cref{assumption:noise} holds and denote $\mathsfit{U} = \mathrm{diag}\lt(\mathsfit{u}_1, \ldots, \mathsfit{u}_d\rt)$.
    Then we have the following identities:
    \begin{enumerate*}[label=(\roman*)]
        \item $\mathbb{E} \mathsfit{u} \mathsfit{u}^{\top} = \mathrm{I}_d$, \label{eq:thm_noise_item1}
        \item $\mathbb{E} \mathsfit{U}^2 = \mathrm{I}_d$. \label{eq:thm_noise_item2}
    \end{enumerate*}
\end{lemma}
\begin{proof}
    From \cref{assumption:noise}, we know that $\mathbb{E}\mathsfit{u}_i^2 = 1$.
    Then \labelcref{eq:thm_noise_item1} follows from
    \[
        {[ \mathbb{E}\mathsfit{u} \mathsfit{u}^{\top} ]}_{ij} = \mathbb{E}\mathsfit{u}_i \mathsfit{u}_j
        \quad=\quad
        \begin{cases}
            \mathbb{E}\mathsfit{u}_i^2 & \text{if $i = j$} \\
            \mathbb{E}\mathsfit{u}_i \mathbb{E}\mathsfit{u}_j & \text{if $i \neq j$}
        \end{cases}
        \quad=\quad
        \begin{cases}
            1 & \text{if $i = j$} \\
            0 & \text{if $i \neq j$}
        \end{cases}
        \; .
        \nonumber
    \]
    For \labelcref{eq:thm_noise_item2}, we only need to focus on the diagonal since the off-diagonal is already zero.
    \[
        {[ \mathbb{E}  \mathsfit{U}^2 ]}_{ii}
        \quad=\quad
        {[ \mathbb{E} {\mathrm{diag} \lt(\mathsfit{u}_1, \ldots, \mathsfit{u}_d\rt)}^2 ]}_{ii}
        \quad=\quad
        \mathbb{E} \mathsfit{u}_i^2
        \quad=\quad
        1 \; .
        \nonumber
    \]
\end{proof}

\begin{lemma}\label{thm:reparam_identity}
    Suppose \cref{assumption:noise} holds, $\mathcal{T}_{\lambda}$ is the reparametrization function for a location-scale family, and the linear parametrization is used.
    Then 
    \[
        \mathbb{E}\norm{\mathcal{T}_{\lambda}\lt(\mathsfit{u}\rt) - \mathcal{T}_{\lambda'}\lt(\mathsfit{u}\rt)}_2^2
        =
        \norm{\lambda - \lambda'}_2^2 \; .
        \nonumber
    \]
\end{lemma}
\begin{proof}
    Denoting $\lambda = (m, C)$ and $\lambda^{\prime} = (m', C')$, 
    \[
        &\mathbb{E}\norm{\mathcal{T}_{\lambda}\lt(\mathsfit{u}\rt) - \mathcal{T}_{\lambda'}\lt(\mathsfit{u}\rt)}_2^2
        \nonumber
        \\
        &\;=
        \mathbb{E}\norm{\lt(C \mathsfit{u} + m\rt) - \lt(C' \mathsfit{u} - m'\rt)}_2^2
        \nonumber
        \\
        &\;=
        \mathbb{E}\norm{\lt(C - C'\rt) \mathsfit{u} + \lt(m - m'\rt)}_2^2
        \nonumber
        \\
        &\;=
        \mathbb{E}\norm{\lt(C - C'\rt) \mathsfit{u}}_2^2
        +
        2 \inner{ \lt(C - C'\rt) \mathbb{E}\mathsfit{u} ,\, m - m' }
        +
        \mathbb{E}\norm{m - m'}_2^2
        \nonumber
        \\
        &\;=
        \mathbb{E}\norm{\lt(C - C'\rt) \mathsfit{u}}_2^2
        +
        \mathbb{E}\norm{m - m'}_2^2 \; .
        &&\text{(\cref{assumption:noise})}
        \label{eq:reparam_identity}
    \]
    Lastly, 
    \[
        \mathbb{E}\norm{\lt(C - C'\rt) \mathsfit{u}}_2^2
        &=
        \mathbb{E}
        \mathsfit{u}^{\top} {\lt(C - C'\rt)}^{\top} \lt(C - C'\rt) \mathsfit{u}
        \nonumber
        \\
        &=
        \mathbb{E}
        \tr \mathsfit{u}^{\top} {\lt(C - C'\rt)}^{\top} \lt(C - C'\rt) \mathsfit{u}
        \nonumber
        \\
        &=
        \tr {\lt(C - C'\rt)}^{\top} \lt(C - C'\rt) \mathbb{E} \mathsfit{u} \mathsfit{u}^{\top} 
        &&\text{(cyclic property of trace)}
        \nonumber
        \\
        &=
        \tr {\lt(C - C'\rt)}^{\top} \lt(C - C'\rt) \mathrm{I}
        &&\text{(\cref{thm:noise})}
        \nonumber
        \\
        &=
        \tr {\lt(C - C'\rt)}^{\top} \lt(C - C'\rt)
        \nonumber
        \\
        &=
        \norm{C - C'}_{\mathrm{F}}^2 \; .
        \nonumber
    \]
    Combining this with \cref{eq:reparam_identity} yields the result.
\end{proof}

\begin{lemma}\label{thm:smoothness_implication}
    Suppose $f$ is $\mu$-strongly convex and \cref{assumption:expected_smoothness} holds.
    Then $f$ is $L$-Lipschitz smooth, while the constants satisfy the ordering
    \[
        \mu \quad\leq\quad L \quad\leq\quad \mathcal{L} \; .
        \nonumber
    \]
\end{lemma}
\begin{proof}
For all $\lambda, \lambda' \in \Lambda $, the unbiasedness of $\widehat{\nabla f}$ and Jensen's inequality states that
\[
    \norm{
    \nabla f\lt(\lambda\rt)
    -
    \nabla f\lt(\lambda'\rt)
    }_2^2
    \;=\;
    \norm{
    \mathbb{E}\widehat{\nabla f}\lt(\lambda; \mathsfit{u}\rt)
    -
    \mathbb{E} \widehat{\nabla f}\lt(\lambda'; \mathsfit{u}\rt)
    }_2^2
    \;\leq\;
    \mathbb{E}
    \norm{
    \widehat{\nabla f}\lt(\lambda; \mathsfit{u}\rt)
    -
    \widehat{\nabla f}\lt(\lambda'; \mathsfit{u}\rt)
    }_2^2 
    \nonumber
    \; .
\]
Then the $\mu$-strong convexity of $f$ and \cref{assumption:expected_smoothness} yields the inequality
\[
    \mu^2
    \norm{\lambda - \lambda'}_2^2
    \;\leq\;
    \norm{
    \nabla f\lt(\lambda\rt)
    -
    \nabla f\lt(\lambda'\rt)
    }_2^2
    \;\leq\;
    \mathbb{E}
    \norm{
    \widehat{\nabla f}\lt(\lambda; \mathsfit{u}\rt)
    -
    \widehat{\nabla f}\lt(\lambda'; \mathsfit{u}\rt)
    }_2^2
    \;\leq\;
    \mathcal{L}^2 \norm{\lambda - \lambda'}_2^2 \; ,
    \nonumber
\]
from which the statement follows immediately.
\end{proof}

\newpage

\section{Proofs}\label{section:proofs}

\subsection{Proofs of Results in \cref{section:background}}

\subsubsection{Proof of \cref{thm:spgd_complexity}}
\label{section:proof_spgd_complexity}
Under the stated assumptions, we first establish a convergence bound which bounds $\mathbb{E}\norm{\lambda_T - \lambda_*}_2^2$ after $T$ iterations under a given step size schedule.
We will invert this convergence bound into a complexity guarantee by identifying the conditions on $T$, $t_*$, and $\gamma_0$ that guarantee $\mathbb{E}\norm{\lambda_T - \lambda_*}_2^2 \leq \epsilon$ for a given $\epsilon > 0$.

\begin{theoremEnd}[%
    restate,
    text proof={},
    category=spgdconvergencebound,
    text link={\textit{Proof}. The \hyperref[proof:prAtEnd\pratendcountercurrent]{\textit{full proof}} is deferred to~\cref{section:proof_spgd_convergence_bound}, p.~\pageref{proof:prAtEnd\pratendcountercurrent}. \qed}
]{lemma}\label{thm:spgd_convergence_bound}
    Suppose $f$ is $\mu$-strongly convex, $h$ satisfies \cref{assumption:hregular}, and $\nabla f$ satisfies \cref{assumption:expected_smoothness,assumption:bounded_variance}. 
    Then, for the global optimum $\lambda_* = \argmin_{\lambda \in \Lambda} F\lt(\lambda\rt)$, any $t_*$ satisfying $4 \mathcal{L}^2/\mu^2 \leq t_* \leq T$, and the step size schedule in \cref{eq:stepsize_schedule}, the contraction coefficient $\rho \triangleq 1 - \mu \gamma_0$ satisfies $\rho \in (0, 1)$ and the last iterate of SPGD after $T$ iterations, $\lambda_T$, satisfies
{%
\setlength{\belowdisplayskip}{1ex} \setlength{\belowdisplayshortskip}{1ex}
\setlength{\abovedisplayskip}{1ex} \setlength{\abovedisplayshortskip}{1ex}
    \[
        \mathbb{E} \norm{\lambda_{T} - \lambda_*}_2^2
        \leq
        \norm{\lambda_{0} - \lambda_*}_2^2 \,
        {\rho}^{t_*} 
        \lt(\frac{{t_*}^2}{T^2}\rt)
        +
        2 \gamma_0
        \frac{\sigma^2}{\mu}
        \frac{{t_*}^2}{T^2}
        +
        \frac{8 \sigma^2}{\mu^2}
        \frac{T - t_*}{T^2}
        \; .
        \nonumber
    \]
}%
\end{theoremEnd}
\vspace{-1ex}
\begin{proofEnd}
    Since $f$ is strongly convex and $h$ is convex, $F$ is also strongly convex.
    This implies that $F$ has a unique global optimum, which we denote as $\lambda_*$.
    Furthermore, under the stated assumptions on $h$, the proximal operator $\mathrm{prox}_{\gamma h}\lt(\cdot\rt)$ is non-expansive for any $\gamma \in (0, \infty)$~\citep[Lemma 8.17]{garrigos_handbook_2023} and any $\lambda, \lambda' \in \mathbb{R}^p$ such that
    \[
        \norm{  \mathrm{prox}_{\gamma h}\lt(\lambda\rt) - \mathrm{prox}_{\gamma h}\lt(\lambda'\rt) }_2 \leq \norm{\lambda - \lambda'}_2
        \label{eq:prox_non_expansive}
    \]
    and $\lambda_*$ is the fixed-point of the deterministic proximal gradient descent step~\citep[Lemma 8.18]{garrigos_handbook_2023} such that
    \[
        \mathrm{prox}_{\gamma h}\lt(\lambda_* - \gamma \nabla f\lt(\lambda_*\rt)\rt)
        =
        \lambda_* \; .
        \label{eq:prox_fixed_point}
    \]
    Using these facts, 
    \[
        \norm{\lambda_{t+1} - \lambda_*}_2^2 
        &\leq
        \norm{ \mathrm{prox}_{\gamma_t h}( \lambda_{t} - \gamma_t \widehat{\nabla f}\lt(\lambda_t; \mathsfit{u}\rt) ) - \mathrm{prox}_{\gamma_t h}\lt( \lambda_* - \gamma_t \nabla f\lt(\lambda_*\rt) \rt) }_2^2 
        &&\text{(\cref{eq:prox_fixed_point})}
        \nonumber
        \\
        &\leq
        \norm{  \lambda_{t} - \gamma_t \widehat{\nabla f}\lt(\lambda_t; \mathsfit{u}\rt) ) -  \lambda_* + \gamma_t \nabla f\lt(\lambda_*\rt) }_2^2 \; .
        &&\text{(\cref{eq:prox_non_expansive})}
        \nonumber
    \]
    Expanding the square, 
    \[
        \norm{\lambda_{t+1} - \lambda_*}_2^2 
        &\leq
        \norm{\lambda_{t} - \lambda_*}_2^2
        -
        2 \gamma_t
        \inner{
            \widehat{\nabla f}\lt(\lambda_t; \mathsfit{u}\rt) - \nabla f\lt(\lambda_*\rt) ,
            \lambda_t - \lambda_*
        }
        +
        \gamma_t^2
        \norm{ 
            \widehat{\nabla f}\lt(\lambda_t; \mathsfit{u}\rt) - \nabla f\lt(\lambda_*\rt)
        }_2^2
        \; .
        \nonumber
    \]
    Denoting the filtration of the $\sigma$-field of the iterates generated up to iteration $t$ as $\mathcal{F}_t$, 
    \[
        &
        \mathbb{E}\lt[\norm{\lambda_{t+1} - \lambda_*}_2^2 \mid \mathcal{F}_t \rt]
        \nonumber
        \\
        &\;\leq
        \norm{\lambda_{t} - \lambda_*}_2^2
        -
        2 \gamma_t^2
        \inner*{
            \mathbb{E}\lt[\widehat{\nabla f}\lt(\lambda_t; \mathsfit{u}\rt) \mid \mathcal{F}_t \rt]
            - \nabla f\lt(\lambda_*\rt) ,
            \lambda_t - \lambda_*
        }
        +
        \gamma_t^2
        \mathbb{E}\lt[
        \norm{ 
            \widehat{\nabla f}\lt(\lambda_t; \mathsfit{u}\rt) - \nabla f\lt(\lambda_*\rt)
        }_2^2
        \mid \mathcal{F}_t \rt]
        \nonumber
        \\
        &=
        \norm{\lambda_{t} - \lambda_*}_2^2
        -
        2 \gamma_t
        \inner*{
            \nabla f\lt(\lambda_t\rt) - \nabla f\lt(\lambda_*\rt), \lambda_t - \lambda_*
        }
        +
        \gamma_t^2
        \mathbb{E}\lt[
        \norm{ 
            \widehat{\nabla f}\lt(\lambda_t; \mathsfit{u}\rt) - \nabla f\lt(\lambda_*\rt)
        }_2^2
        \mid \mathcal{F}_t \rt] 
        \; ,
        \label{eq:thm_spgd_convergence_eq1}
    \]
    where the equality follows from the fact that $\widehat{\nabla f}$ is unbiased conditional on any $\lambda_t \in \Lambda$.

    From the $\mu$-strong convexity of $f$,
    \[
        &-
        2 \gamma_t
        \inner*{
            \nabla f\lt(\lambda_t\rt) - \nabla f\lt(\lambda_*\rt), \lambda_t - \lambda_*
        }
        \nonumber
        \\
        &\qquad=
        -
        2 \gamma_t
        \inner*{
            \nabla f\lt(\lambda_t\rt) , \lambda_t - \lambda_*
        }
        + 
        2 \gamma_t
        \inner*{
            \nabla f\lt(\lambda_*\rt) , \lambda_t - \lambda_*
        }
        \nonumber
        \\
        &\qquad\leq
        -
        \gamma_t \mu \norm{\lambda_t - \lambda_*}_2^2
        -
        2 \gamma_t 
        \big\{
        f\lt(\lambda_t\rt) - f\lt(\lambda_*\rt)
        -
        \inner*{
            \nabla f\lt(\lambda_*\rt) , \lambda_t - \lambda_*
        }
        \big\}
        && \text{($\mu$-strong convexity of $f$)}
        \nonumber
        \\
        &\qquad=
        -
        \gamma_t \mu \norm{\lambda_t - \lambda_*}_2^2
        -
        2 \gamma_t \mathrm{D}_f\lt(\lambda_t, \lambda_*\rt)
        && \text{(\cref{eq:bregman})}
        \label{eq:spgd_decrease_objective}
        \; .
    \]
    The gradient variance at $\lambda_t$, on the other hand, can be compared against the gradient variance at $\lambda_*$ through the variance transfer strategy as 
    \[
        &\gamma_t^2
        \mathbb{E}\lt[
        \norm{ 
            \widehat{\nabla f}\lt(\lambda_t; \mathsfit{u}\rt) - \nabla f\lt(\lambda_*; \mathsfit{u}\rt)
        }_2^2
        \mid \mathcal{F}_t \rt] 
        \nonumber
        \\
        &\;=
        \gamma_t^2
        \mathbb{E}\lt[
        \norm{ 
            \widehat{\nabla f}\lt(\lambda_t; \mathsfit{u}\rt) - \widehat{\nabla f}\lt(\lambda_*;\mathsfit{u}\rt) + \widehat{\nabla f}\lt(\lambda_*; \mathsfit{u}\rt) - \nabla f\lt(\lambda_*; \mathsfit{u}\rt)
        }_2^2
        \mid \mathcal{F}_t \rt] 
        \nonumber
        \\
        &\;\leq
        2 \,
        \gamma_t^2
        \mathbb{E}\lt[
        \norm{ 
            \widehat{\nabla f}\lt(\lambda_t; \mathsfit{u}\rt) - \widehat{\nabla f}\lt(\lambda_*; \mathsfit{u}\rt) 
        }_2^2
        \mid \mathcal{F}_t \rt] 
        +
        2
        \gamma_t^2
        \mathbb{E}\lt[
        \norm{ 
            \widehat{\nabla f}\lt(\lambda_*; \mathsfit{u}\rt) - \nabla f\lt(\lambda_*; \mathsfit{u}\rt)
        }_2^2
        \mid \mathcal{F}_t \rt] 
        \nonumber
        &&\text{(Young's inequality)}
        \\
        &\;\leq
        2 \gamma_t^2 \mathcal{L}^2 \norm{\lambda_t - \lambda_*}_2^2
        +
        2 
        \gamma_t^2 \sigma^2 \; , 
        \nonumber
        &&\text{(\cref{assumption:expected_smoothness,assumption:bounded_variance})}
        \\
        &=
        4 \gamma_t^2 \frac{\mathcal{L}^2}{\mu} \, \Big(
            f\lt(\lambda_t\rt) - f\lt(\lambda_*\rt) - \inner{ \nabla f\lt(\lambda_t\rt), \lambda_t - \lambda_* }
        \Big)
        + 2 \gamma_t^2 \sigma^2
        &&\text{($\mu$-strong convexity of $f$)}
        \nonumber
        \\
        &\;=
        4 \gamma_t^2 \frac{\mathcal{L}^2}{\mu} \, \mathrm{D}_f\lt(\lambda_t, \lambda_*\rt)
        + 2 \gamma_t^2 \sigma^2 \; .
        \label{eq:spgd_gradient_variance}
    \]
    
    Applying \cref{eq:spgd_decrease_objective,eq:spgd_gradient_variance} to \cref{eq:thm_spgd_convergence_eq1}, 
    \[
        \mathbb{E}\lt[\norm{\lambda_{t+1} - \lambda_*}_2^2 \mid \mathcal{F}_t \rt]        
        &\leq
        \norm{\lambda_{t} - \lambda_*}_2^2
        -
        \gamma_t \Big(
        \mu
        \norm{ \lambda_t - \lambda_* }_2^2
        +
        2 \, \mathrm{D}_f\lt(\lambda_t, \lambda_*\rt)
        \Big)
        +
        2 \gamma_t^2 \lt(
        2 \frac{\mathcal{L}^2}{\mu} \mathrm{D}_f\lt(\lambda_t, \lambda_*\rt) + \sigma^2
        \rt)
        \nonumber
        \\
        &=
        \lt(
        1 - \gamma_t \mu 
        \rt)
        \norm{\lambda_{t} - \lambda_*}_2^2
        +
        2
        \gamma_t
        \lt( 2 \gamma_t \frac{\mathcal{L}^2}{\mu} - 1 \rt)
        \mathrm{D}_f\lt(\lambda_t, \lambda_*\rt)
        +
        2 \gamma_t^2 \sigma^2
        \; .
        \nonumber
    \]
    Taking expectation over all randomness, we obtain our general partial contraction bound
    \[
        \mathbb{E} \norm{\lambda_{t+1} - \lambda_*}_2^2
        \leq
        \lt(
        1 - \gamma_t \mu 
        \rt)
        \mathbb{E}\norm{\lambda_{t} - \lambda_*}_2^2
        +
        2
        \gamma_t
        \lt( 2 \gamma_t \frac{\mathcal{L}^2}{\mu} - 1 \rt)
        \mathbb{E} \lt[ \mathrm{D}_f\lt(\lambda_t, \lambda_*\rt) \rt]
        +
        2 \gamma_t^2 \sigma^2
        \; .
        \label{eq:thm_spgd_convergence_general_onestep_convergence}
    \]

    Due to the form of the step size schedule, SPGD operates in two different regimes: the first stage with a fixed step size $\gamma_t = \gamma_0$ ($t \in \{0, \ldots, t_*\}$) and the second stage with a decreasing step size $\gamma_{t+1} < \gamma_t$ ($t \in \{t_* + 1, \ldots, T\}$).
    In the first stage, $\gamma_t = \gamma_0 \leq \frac{\mu}{2 \mathcal{L}^2} $.
    Then the Bregman divergence term in \cref{eq:thm_spgd_convergence_general_onestep_convergence} is negative such that
    \[
        \mathbb{E} \norm{\lambda_{t+1} - \lambda_*}_2^2
        \leq
        \lt(
        1 - \gamma_t \mu 
        \rt)
        \mathbb{E}\norm{\lambda_{t} - \lambda_*}_2^2
        +
        2 \gamma_t^2 \sigma^2
    \]
    Unrolling the recursion yields
    \[
        \mathbb{E} \norm{\lambda_{t_*} - \lambda_*}_2^2
        &\leq
        {\lt(
        1 - \gamma_0 \mu 
        \rt)}^{t_*}
        \norm{\lambda_{0} - \lambda_*}_2^2 
        + 
        2 \gamma_0^2 \sigma^2
        \sum_{t=0}^{t_* - 1}
        {\lt(
        1 - \gamma_0 \mu 
        \rt)}^t
        \nonumber
        \\
        &\leq
        {\lt(
        1 - \gamma_0 \mu 
        \rt)}^{t_*}
        \norm{\lambda_{0} - \lambda_*}_2^2 
        + 
        2 \gamma_0 
        \frac{\sigma^2}{\mu}
        \nonumber
        &&\text{(geometric series sum formula)}
        \\
        &\leq
        \rho^{t_*}
        \norm{\lambda_{0} - \lambda_*}_2^2 
        + 
        2 \gamma_0 
        \frac{\sigma^2}{\mu}
        \; .
        \label{eq:thm_spgd_convergence_eq2}
    \]
    From \cref{thm:smoothness_implication}, we deduce that $\gamma_0 \mu = \mu^2/(2 \mathcal{L}^2) \leq 1/2$, which implies $\rho \in (0, 1)$.

    We now turn to the second stage, where the step size starts decreasing.
    Notice that \cref{eq:stepsize_schedule} satisfies
    \[
        \gamma_t
        \quad=\quad
        \frac{1}{\mu} \frac{2 t + 1}{ {(t + 1)}^2 }
        \quad\leq\quad
        \frac{1}{\mu} \frac{2 t_* + 1}{ {(t_* + 1)}^2 }
        \quad\leq\quad
        \frac{1}{\mu} \frac{2}{t_*}
        \quad\leq\quad
        \frac{1}{\mu} \frac{2 \mu^2}{4 \mathcal{L}^2}
        \quad\leq\quad
        \frac{\mu}{2 \mathcal{L}^2} \; .
        \nonumber
    \]
    Therefore, $\gamma_t  \leq \frac{\mu}{2 \mathcal{L}^2}$ for all $t \geq 0$.
    Again, the Bregman term in \cref{eq:thm_spgd_convergence_general_onestep_convergence} is negative such that
    \[
        \mathbb{E}\norm{\lambda_{t+1} - \lambda_*}_2^2
        \leq
        \lt(
        1 - \gamma_t \mu 
        \rt)
        \mathbb{E}\norm{\lambda_{t} - \lambda_*}_2^2
        +
        2 \gamma_t^2 \sigma^2 \; .
        \nonumber
    \]
    Subtituting $\gamma_t$ with the choice in \cref{eq:stepsize_schedule}, we obtain 
    \[
        \mathbb{E} \norm{\lambda_{t+1} - \lambda_*}_2^2
        &\leq
        \lt( 1 -  \frac{ 2t + 1 }{ {\lt(t + 1\rt)}^2 } \rt)
        \mathbb{E}\norm{\lambda_{t} - \lambda_*}_2^2 
        + 
        2
        \frac{\sigma^2}{\mu^2}
        \frac{ {\lt( 2t + 1 \rt)}^2 }{ {\lt(t + 1\rt)}^4 }
        \nonumber
        \\
        &=
        \frac{ t^2 }{ {\lt(t + 1\rt)}^2 }
        \mathbb{E}\norm{\lambda_{t} - \lambda_*}_2^2 
        + 
        2
        \frac{\sigma^2}{\mu^2}
        \frac{ {\lt( 2t + 1 \rt)}^2 }{ {\lt(t + 1\rt)}^4 } \; .
        \nonumber
    \]
    Multiplying ${(t+1)}^2$ to both sides, 
    \[
        {\lt(t + 1\rt)}^2
        \mathbb{E} \norm{\lambda_{t+1} - \lambda_*}_2^2
        &\leq
        t^2
        \mathbb{E}\norm{\lambda_{t} - \lambda_*}_2^2 
        + 
        2
        \frac{\sigma^2}{\mu^2}
        \frac{ {\lt( 2t + 1 \rt)}^2 }{ {\lt(t + 1\rt)}^2 }  \; .
        \nonumber
    \]
    Let us choose the Lyapunov function $V_{t} \triangleq {\lt(t + 1\rt)}^2 \mathbb{E} \norm{\lambda_{t+1} - \lambda_*}_2^2 $ .
    Then the discrete derivative of the Lyapunov,
    \[
        V_{t+1} - V_t 
        \quad\leq\quad
        2
        \frac{\sigma^2}{\mu^2}
        \frac{ {\lt( 2t + 1 \rt)}^2 }{ {\lt(t + 1\rt)}^2 }
        \quad\leq\quad
        8 \frac{\sigma^2}{\mu^2} \; ,
        \nonumber
    \]
    shows that the energy is increasing only by a constant.
    By integrating the Lyapunov over the time interval $t = t_*, \ldots, T - 1$, 
    \[
        &\qquad
        &V_{T} - V_{t_*}
        \;&\leq\;
        8 \frac{\sigma^2}{\mu^2}  \lt( T - t_*\rt) 
        \nonumber
        \\
        &\Leftrightarrow\qquad
        &V_{T} 
        \;&\leq\;
        V_{t_*}
        +
        8 \frac{\sigma^2}{\mu^2} \lt( T - t_*\rt) 
        \nonumber
        \\
        &\Leftrightarrow\qquad
        &T^2 \, \mathbb{E} \norm{\lambda_{T} - \lambda_*}_2^2
        \;&\leq\;
        t_*^2 \mathbb{E} \norm{\lambda_{t_*} - \lambda_*}_2^2
        +
        8 \frac{\sigma^2}{\mu^2} \lt( T - t_*\rt) 
        \nonumber
        \\
        &\Leftrightarrow\qquad
        &\mathbb{E} \norm{\lambda_{T} - \lambda_*}_2^2
        \;&\leq\;
        \frac{{t_*}^2}{T^2}
        \mathbb{E} \norm{\lambda_{t_*} - \lambda_*}_2^2
        +
        8 \frac{\sigma^2}{\mu^2} \frac{T - t_*}{T^2}
        \; .
        \nonumber
    \]
    Substuting $\norm{\lambda_{t_*} - \lambda_*}_2^2$ with the error in 
    \cref{eq:thm_spgd_convergence_eq2}, 
    \[
        \mathbb{E} \norm{\lambda_{T} - \lambda_*}_2^2
        &\leq
        \lt(
        \rho^{t_*}
        \norm{\lambda_{0} - \lambda_*}_2^2 \,
        + 
        2 \gamma_0
        \frac{\sigma^2}{\mu}
        \rt)
        \frac{{t_*}^2}{T^2}
        +
        8 \frac{\sigma^2}{\mu^2} \frac{T - t_*}{T^2}
        \nonumber
        \\
        &=
        \norm{\lambda_{0} - \lambda_*}_2^2 \,
        \rho^{t_*}
        \frac{{t_*}^2}{T^2}
        +
        2 \gamma_0
        \frac{\sigma^2}{\mu}
        \frac{{t_*}^2}{T^2}
        +
        \frac{8 \sigma^2}{\mu^2}
        \frac{T - t_*}{T^2}
        \; ,
        \label{thm:spgd_convergence_precise}
    \]
    which is our stated result.
\end{proofEnd}

This is a slight generalization of a result~\citep[Theorem 7]{domke_provable_2023}, where the switching time $t_*$ was fixed to some $t_* \propto \mathcal{L}/\mu$.
While the choice of $t_* \propto \mathcal{L}/\mu$ results in the typical $\mathrm{O}(1/\epsilon)$ asymptotic complexity, it suffers from a suboptimal polynomial dependence on the initialization error $\Delta = \norm{\lambda_0 - \lambda_*}_2$.
Picking an alternative $t_*$, which is what we do in the proof, improves the iteration complexity to $\mathrm{O}\big(1/\epsilon + 1/\sqrt{\epsilon} \log \Delta^2 + \log (\Delta^2/\epsilon) \big)$.

\printProofs[spgdcomplexity]

\newpage
\subsubsection{Proof of \cref{thm:spgd_convergence_bound}}
\label{section:proof_spgd_convergence_bound}

The proof closely mirrors the strategy of~\citet[Theorem 12.9]{garrigos_handbook_2023}, which is a combination of previous analyses of SPGD~\citep{khaled_unified_2023,gorbunov_unified_2020} with the analysis of SGD strongly convex objectives with a decreasing step size schedule~\citep{gower_sgd_2019}.
The main difference is that \citeauthor{garrigos_handbook_2023} utilize a different condition on the gradient variance instead of \cref{assumption:expected_smoothness}.
Specificially, they assume that, for all $\lambda, \lambda' \in \Lambda$, there exists some function of $L\lt(u\rt)  :  \mathrm{supp}\lt(\mathsfit{u}\rt) \to [0, \infty)$ such that, for each $u \in \mathrm{supp}\lt(\mathsfit{u}\rt)$, the function $\widehat{\nabla f}\lt(\lambda; u\rt) : \Lambda \to \mathbb{R}^p$ is $L\lt(u\rt)$-smooth with respect to $\lambda$.
This then enables the use of the ``convex expected smoothness''~\citep{gorbunov_unified_2020,khaled_unified_2023} condition, which postulates that, for all $\lambda \in \Lambda$, there exists some $\mathcal{L} < \infty$ such that
\[
    \mathbb{E}\norm{ \widehat{\nabla f}\lt(\lambda; \mathsfit{u}\rt) - \widehat{\nabla f}\lt(\lambda'; \mathsfit{u}\rt) }_2^2
    \leq
    \mathcal{L}^2 \mathrm{D}_f\lt(\lambda, \lambda'\rt) \; ,
    \label{eq:convex_expected_smoothness}
\]
where 
\[
    \mathrm{D}_f\lt(\lambda, \lambda'\rt) \triangleq f(\lambda) - f(\lambda') - \inner{\nabla f(\lambda'), \lambda - \lambda'}
    \label{eq:bregman}
\]
is the Bregman divergence associated with $f$.
Note that \cref{assumption:expected_smoothness} and the $\mu$-strong convexity of $f$ implies \cref{eq:convex_expected_smoothness}.
Therefore, under our assumptions, one can invoke the results that assume \cref{eq:convex_expected_smoothness}, which was the strategy by some previous analyses of BBVI~\citep{ko_provably_2024,kim_convergence_2023}.
Here, we will take a more straightforward approach that uses \cref{assumption:expected_smoothness} directly in the convergence proof, but the results are identical to the indirect approach of establishing \cref{eq:convex_expected_smoothness}.

\printProofs[spgdconvergencebound]


\newpage
\subsection{Proofs of Results in \cref{section:main_results}}

\subsubsection{Proof of \cref{thm:bbvi_complexity}}
\label{section:proof_bbvi_complexity}
\printProofs[bbvicomplexity]

\subsubsection{Proof of \cref{thm:special_case_gaussian}}
\label{section:proof_special_case_gaussian}

The result follows from a well-known bound on the expected maximum of sub-exponential random variables.
We state the proof for completeness.

\begin{theoremEnd}[%
    category=expectedmaximummgf,
    text proof={},
    text link={},
]{lemma}\label{thm:expected_maximum_mgf}
    Let \(\mathsfit{x}_1, \ldots, \mathsfit{x}_d\) be i.i.d. random variables.
    Suppose there exists some $t > 0$ such that their moment-generating function (MGF) satisfies $M_{\mathsfit{x}_i}\left(t\right) < \infty$.
    Then
{%
\setlength{\belowdisplayskip}{0ex} \setlength{\belowdisplayshortskip}{0ex}
\setlength{\abovedisplayskip}{0ex} \setlength{\abovedisplayshortskip}{0ex}
    \[
        \mathbb{E} \max_{i=1, \dots, d} \mathsfit{x}_i \quad\leq\quad \frac{1}{t} \lt( \log M_{\mathsfit{x}_i}\left(t\right)  + \log d\rt) \; .
        \nonumber
    \]
}%
\end{theoremEnd}
\vspace{-1ex}
\begin{proofEnd}
    \[
        \mathbb{E}\lt[ t \max_{i=1, \dots, d} \mathsfit{x}_i  \rt]
        &=
        \log \exp\left(
        \mathbb{E}\lt[ t \max_{i=1, \dots, d} \mathsfit{x}_i  \rt]
        \right)
        \nonumber
        \\
        &\leq
        \log \mathbb{E} \exp\left(t \max_{i=1, \dots, d} \mathsfit{x}_i \right)
        &&\quad\text{(Jensen's inequality)}
        \nonumber
        \\
        &=
        \log \mathbb{E}\max_{i=1, \dots, d} \exp\left(t  \mathsfit{x}_i \right)
        \nonumber
        \\
        &\leq
        \log \mathbb{E} \sum_{i=1}^d \exp\left(t  \mathsfit{x}_i \right)
        \nonumber
        \\
        &=
        \log \sum_{i=1}^d M_{\mathsfit{x}_i}\lt(t\rt)
        &&\quad\text{(Definition of MGFs)}
        \nonumber
        \\
        &=
        \log \lt( d M_{\mathsfit{x}_i}\lt(t\rt) \rt) \; .
        &&\quad\text{($\mathsfit{x}_1, \ldots, \mathsfit{x}_d$ are i.i.d.)}
        \nonumber
    \]
    Dividing both sides by \(t\) yields the statement.
\end{proofEnd}
\printProofs[expectedmaximummgf]

Applying \cref{thm:expected_maximum_mgf} to $\mathsfit{u}_i^2$ yields the result.

\printProofs[specialcasegaussian]

\newpage
\subsubsection{Proof of \cref{thm:special_case_studentt}}\label{section:proof_special_case_studentt}

The result follows from the following moment-based bound on the expected maximum of random variables, which is a non-asymptotic refinement of the proof by~\citet{rana_answer_2017}.

\begin{theoremEnd}[%
    category=expectedmaximumhighestmoment,
    text link={},
    text proof={},
]{lemma}\label{thm:expected_maximum_highestmoment}
    Let \(\mathsfit{x}_1, \ldots, \mathsfit{x}_d\) be i.i.d. non-negative random variables where, for \(k \geq 2\), their \(k\)th moment is finite. That is, \(\mathbb{E}\mathsfit{x}_i^k = r_{k} < \infty\).
    Then
{%
\setlength{\belowdisplayskip}{1ex} \setlength{\belowdisplayshortskip}{1ex}
\setlength{\abovedisplayskip}{0ex} \setlength{\abovedisplayshortskip}{0ex}
    \[
        \mathbb{E}\max_{i = 1, \ldots, d}  \mathsfit{x}_i
        \quad\leq\quad
        d^{1/k}
        \lt(k/(k - 1)\rt)^{(k-1)/k} r_k^{1/k} \; .
        \nonumber
    \]
}%
\end{theoremEnd}
\vspace{-1ex}
\begin{proofEnd}
    For any \(\epsilon > 0\), we have
    \[
        \mathbb{E}\max_{i = 1, \ldots, d}  \mathsfit{x}_i
        &=
        \int^{\epsilon d^{1/k}}_0 \mathbb{P}\left[ \max_{i = 1, \ldots, d} \mathsfit{x}_i \geq t \right] \, \mathrm{d}t
        +
        \int_{\epsilon d^{1/k}}^{\infty} \mathbb{P}\left[ \max_{i = 1, \ldots, d} \mathsfit{x}_i \geq t \right] \, \mathrm{d}t
        \nonumber
        \\
        &\leq
        \int^{\epsilon d^{1/k}}_0 \mathrm{d}t
        +
        \int_{\epsilon d^{1/k}}^{\infty} d \, \mathbb{P}\left[ \mathsfit{x}_i \geq t \right] \, \mathrm{d}t
        &&\text{(i.i.d. and \(\mathbb{P}\lt[\cdot \rt] \leq 1\))}
        \nonumber
        \\
        &=
        d^{1/k}
        \lt(
        \epsilon 
        +
        \frac{1}{k \epsilon^{k-1}}
        \int_{\epsilon d^{1/k}}^{\infty} k {\lt( \epsilon d^{1/k} \rt)}^{k - 1} \, \mathbb{P}\left[ \mathsfit{x}_i \geq t \right] \, \mathrm{d}t
        \rt)
        \nonumber
        \\
        &\leq
        d^{1/k}
        \lt(
        \epsilon 
        +
        \frac{1}{k \epsilon^{k-1}}
        \int_{\epsilon d^{1/k}}^{\infty} k t^{k - 1} \, \mathbb{P}\left[ \mathsfit{x}_i \geq t \right] \, \mathrm{d}t
        \rt)
        &&\text{($\epsilon d^{1/k} \leq t$)}
        \nonumber
        \\
        &\leq
        d^{1/k}
        \lt(
        \epsilon 
        +
        \frac{1}{k \epsilon^{k-1}}
        \int_{0}^{\infty} k t^{k - 1} \, \mathbb{P}\left[ \mathsfit{x}_i \geq t \right] \, \mathrm{d}t
        \rt)
        \nonumber
        &&\text{(Decreased lower limit of integral)}
        \; .
        \nonumber
    \]
    Now, from the definition of moments, we know that 
    \[
        \int_{0}^{\infty} k t^{k-1} \, \mathbb{P}\left[ \mathsfit{x}_i \geq t \right] \, \mathrm{d}t
        &=
        \int_{0}^{\infty}
        \int_{-\infty}^{\infty} 
        k t^{k-1} \, 
        \mathds{1}_{\mathsfit{x}_i > t} 
        \,  \mathrm{d}\mathbb{P}[x_i]
        \, \mathrm{d}t
        \nonumber
        \\
        &=
        \int_{-\infty}^{\infty} 
        \int_{0}^{\infty} 
        k t^{k-1} \, 
        \mathds{1}_{\mathsfit{x}_i > t} \, 
        \, \mathrm{d}t
        \, \mathrm{d}\mathbb{P}[x_i]
        &&\quad\text{(Fubini's Theorem)}
        \nonumber
        \\
        &=
        \int_{-\infty}^{\infty} 
        \int_{0}^{x_i} 
        k t^{k-1} 
        \, \mathrm{d}t
        \, \mathrm{d}\mathbb{P}[x_i]
        \nonumber
        \\
        &=
        \int_{-\infty}^{\infty} 
        x_i^k \, 
        \, \mathrm{d}\mathbb{P}[x_i]
        \nonumber
        \\
        &=
        r_k  \; .
        \nonumber
    \]
    Therefore, 
    \[
        \mathbb{E}\max_{i = 1, \ldots, d}  \mathsfit{x}_i
        &\leq
        d^{1/k}
        \lt(
        \epsilon 
        +
        \frac{1}{k \epsilon^{k-1}} r_k
        \rt)
        \; .
        \nonumber
    \]
    The bound is minimized when setting
    \[
        \epsilon = {\lt( \frac{k - 1}{k} r_k \rt)}^{1/k} \; .
        \nonumber
    \]
    Then
    \[
        \mathbb{E}\max_{i = 1, \ldots, d}  \mathsfit{x}_i
        &\leq
        d^{1/k}
        \lt(
        {\lt( \frac{k - 1}{k} r_k \rt)}^{1/k}
        +
        \frac{1}{k} m_k
        {\lt( \frac{k - 1}{k} r_k \rt)}^{-(k - 1)/k}
        \rt)
        \nonumber
        \\
        &=
        d^{1/k}
        \lt(
        {\lt( \frac{k - 1}{k} r_k \rt)}^{1/k}
        +
        \frac{1}{k-1}
        {\lt( \frac{k - 1}{k} r_k \rt)}^{1/k}
        \rt)
        \nonumber
        \\
        &=
        d^{1/k}
        \lt(
        1
        +
        \frac{1}{k-1}
        \rt)
        {\lt( \frac{k - 1}{k} r_k \rt)}^{1/k}
        \nonumber
        \\
        &=
        d^{1/k}
        {\lt(
        \frac{k}{k - 1}
        \rt)}^{(k - 1)/k}
        r_k^{1/k} \; .
        \nonumber
    \]
\end{proofEnd}
\printProofs[expectedmaximumhighestmoment]

If the $k$th moment of $\mathsfit{u}_i^2$ is finite, this then immediately implies a polynomial $\mathrm{O}(d^{1/k})$ bound on $g$.

\printProofs[specialcasestudentt]

\newpage
\subsection{Proofs of Results in \cref{section:gradient_variance_analysis}}
\subsubsection{Proof of \cref{thm:gradient_variance_upperbound_meanfield_general}}\label{section:proof_gradient_variance_upperbound_meanfield_general}

Under the assumption that $\nabla^2 \ell \preceq L \mathrm{I}_d$ and twice differentiability, it is well known that $\nabla^2 \ell \preceq L \mathrm{I}_d \; \Rightarrow \text{$\ell$ is $L$-smooth}$.
We will prove a supporting result analogous to this under \cref{assumption:almost_constant_hessian}, which will allow us to bound the relative growth of $\nabla \ell$.

\begin{theoremEnd}[%
    restate,
    category=weightednormsmoothness,
    text link={\textit{Proof.} The \hyperref[proof:prAtEnd\pratendcountercurrent]{\textit{full proof}} is deferred to~\cref{section:proof_weighted_norm_smoothness}, p.~\pageref{proof:prAtEnd\pratendcountercurrent}. \qed},
    text proof={Proof.}
]{lemma}\label{thm:weighted_norm_smoothness}
    Suppose $\ell : \mathbb{R}^d \to \mathbb{R}$ satisfies \cref{assumption:almost_constant_hessian}.
    Then, for any \(W \in \mathbb{R}^{d \times d}\) satisfying $\norm{W}_2 < \infty$,
    {%
    \setlength{\belowdisplayskip}{1ex} \setlength{\belowdisplayshortskip}{1ex}
    \setlength{\abovedisplayskip}{1ex} \setlength{\abovedisplayshortskip}{1ex}
    \[
        \norm{ W \lt( \nabla \ell\lt(z\rt) - \nabla \ell\lt(z'\rt) \rt) }_{2}
        \leq
        \norm*{W H \lt( z - z' \rt) }_2
        +
        \delta
        \norm{W}_2
        \norm{z - z'}_2 \; .
        \nonumber
    \]
    }%
\end{theoremEnd}
\vspace{-1ex}
\begin{proofEnd}
From twice differentiability of $\ell$ (\cref{assumption:almost_constant_hessian}) and the fundamental theorem of calculus, we know that
\[
    \norm{ W \lt( \nabla \ell\lt(z\rt) - \nabla \ell\lt(z'\rt) \rt) }_{2}
    &=
    \norm*{ 
        W
        \int_0^1
        \nabla^2 \ell\lt(t z + (1 - t) z'\rt) \lt( z - z'\rt) 
        \, \mathrm{d}t
    }_{2} \; .
    \nonumber
\]    
Denoting $z_t \triangleq t z + (1 - t) z'$ for clarity, 
\[
    &\norm{ W \lt( \nabla \ell\lt(z\rt) - \nabla \ell\lt(z'\rt) \rt) }_{2}
    \nonumber
    \\ 
    &\quad=
    \norm*{ 
        \int_0^1
        W
        \nabla^2 \ell\lt(z_t\rt) \lt( z - z'\rt) 
        \, \mathrm{d}t
    }_{2}
    \nonumber
    \\
    &\quad\leq
    \int_0^1
    \norm{ W \nabla^2 \ell\lt(z_t\rt) \lt( z - z'\rt) }_2
    \, \mathrm{d}t
    \nonumber
    &&\text{(Jensen's inequality)}
    \\
    &\quad=
    \int_0^1
    \norm*{
        W \lt( \nabla^2 \ell\lt(z_t\rt) - H + H \rt) \lt( z - z'\rt) 
    }_2
    \, \mathrm{d}t
        \nonumber
    \\
    &\quad\leq
    \int_0^1
    \lt\{
    \norm*{ W H \lt( z - z' \rt) }_2
    +
    \norm{W}_2
    \norm*{\nabla^2 \ell\lt(z_t\rt) - H}_2
    \norm*{z - z'}_2
    \rt\}
    \, \mathrm{d}t
        \nonumber
    &&\text{(Triangle inequality)}
    \\
    &\quad\leq
    \int_0^1
    \lt\{
    \norm*{ W H \lt( z - z' \rt) }_2
    +
    \delta
    \norm{W}_2
    \norm{z - z'}_2
    \rt\}
    \, \mathrm{d}t
    \nonumber
    &&\text{(\cref{assumption:almost_constant_hessian})}
    \\
    &\quad=
    \norm*{ W H \lt( z - z' \rt) }_2
    +
    \delta
    \norm{W}_2
    \norm{z - z'}_2 \; .
    \nonumber
\]
\end{proofEnd}

Using this, we can now simplify the $\nabla \ell$ terms in \cref{eq:meanfield_gradient_variance_decomposition}.
Applying \cref{thm:weighted_norm_smoothness} to $V_{\text{loc}}$ with $W = \mathrm{I}_d$ and Young's inequality, 
\[
    V_{\text{loc}}
    &\leq
    \mathbb{E}
    {\lt(
    \norm{H \lt( \mathcal{T}_{\lambda}\left(\mathsfit{u}\right) - \mathcal{T}_{\lambda'}\left(\mathsfit{u}\right) \rt) }_2
    +
    \delta \, \norm{\mathcal{T}_{\lambda}\left(\mathsfit{u}\right) - \mathcal{T}_{\lambda'}\left(\mathsfit{u}\right) }_2
    \rt)}^2
    \nonumber
    &&\text{(\cref{thm:weighted_norm_smoothness})}
    \\
    &\leq
    2 \, \mathbb{E}\norm{H \lt( \mathcal{T}_{\lambda}\left(\mathsfit{u}\right) - \mathcal{T}_{\lambda'}\left(\mathsfit{u}\right) \rt) }_2^2
    +
    2 \delta^2 \, \mathbb{E}\norm{\mathcal{T}_{\lambda}\left(\mathsfit{u}\right) - \mathcal{T}_{\lambda'}\left(\mathsfit{u}\right) }_2^2
    &&\text{(Young's inequality)}
    \nonumber
    \\
    &\leq
    2 \, \norm{H}_2^2 \mathbb{E}\norm{ \mathcal{T}_{\lambda}\left(\mathsfit{u}\right) - \mathcal{T}_{\lambda'}\left(\mathsfit{u}\right) }_2^2
    +
    2 \delta^2 \, \mathbb{E}\norm{\mathcal{T}_{\lambda}\left(\mathsfit{u}\right) - \mathcal{T}_{\lambda'}\left(\mathsfit{u}\right) }_2^2
    &&\text{(Operator norm)}
    \nonumber
    \\
    &=
    2 \lt( \norm{H}_2^2 +  \delta^2 \rt) \mathbb{E}\norm{ \mathcal{T}_{\lambda}\left(\mathsfit{u}\right) - \mathcal{T}_{\lambda'}\left(\mathsfit{u}\right) }_2^2
    \nonumber
    \\
    &=
    2 \lt( \norm{H}_2^2 +  \delta^2 \rt) \norm{\lambda - \lambda' }_2^2 
    \; .
    &&\text{(\cref{thm:reparam_identity})}
    \label{eq:general_gradient_variance_vloc}
\]
Similarly, applying \cref{thm:weighted_norm_smoothness} to $V_{\text{scale}}$ with $W = \mathsfit{U}$ and Young's inequality,
\[
    V_{\text{scale}}
    &\leq
    \mathbb{E}
    {\lt(
    \norm{\mathsfit{U} H \lt( \mathcal{T}_{\lambda}\left(\mathsfit{u}\right) - \mathcal{T}_{\lambda'}\left(\mathsfit{u}\right) \rt) }_2
    +
    \delta \norm{\mathsfit{U}}_2 \, \mathbb{E}\norm{\mathcal{T}_{\lambda}\left(\mathsfit{u}\right) - \mathcal{T}_{\lambda'}\left(\mathsfit{u}\right) }_2
    \rt)}^2
    \nonumber
    &&\text{(\cref{thm:weighted_norm_smoothness})}
    \\
    &\leq
    2
    \underbrace{
    \mathbb{E}\norm{ \mathsfit{U} H \lt( \mathcal{T}_{\lambda}\lt(\mathsfit{u}\rt) - \mathcal{T}_{\lambda'}\lt(\mathsfit{u}\rt) \rt) }_{2}^2
    }_{V_{\text{const}}}
    +
    2 \delta^2
    \underbrace{
    \mathbb{E}\norm{\mathsfit{U}}_2^2 \norm{ \mathcal{T}_{\lambda}\lt(\mathsfit{u}\rt) - \mathcal{T}_{\lambda'}\lt(\mathsfit{u}\rt) }_{2}^2
    }_{V_{\text{non-const}}}
    &&\text{(Young's inequality)}
    \; .
    \label{eq:thm_gradient_variance_upperbound_meanfield_general_vscale}
\]
$V_{\text{const}}$ corresponds to the constant component of the Hessian $\nabla^2 \ell$, whereas $V_{\text{non-const}}$ corresponds to the non-constant residual.
Denote the location and scale parameters of $\lambda$ and $\lambda'$ as
\[
    \lambda = (m, C) \qquad\text{and}\qquad \lambda' = (m', C') \; .
    \nonumber
\]
For $V_{\text{const}}$, we can use the following lemma:

\begin{theoremEnd}[%
    restate, 
    category=constanthessian, 
    text link={See the \hyperref[proof:prAtEnd\pratendcountercurrent]{\textit{full proof}} in~\cref{section:proof_constant_hessian}, p.~\pageref{proof:prAtEnd\pratendcountercurrent}.},
    text proof={Proof.}
]{lemma}\label{thm:constant_hessian}
    Suppose \(\mathcal{T}_{\lambda}\) is the reparameterization operator of a mean-field location-family and \cref{assumption:noise} holds.
    Then, for any matrix \(H \in \mathbb{R}^{d \times d}\) and any \(\lambda, \lambda' \in \mathbb{R}^d \times \mathbb{D}^d\),
    \[
        \norm{ \mathsfit{U} H \lt( \mathcal{T}_{\lambda}\lt(\mathsfit{u}\rt) - \mathcal{T}_{\lambda'}\lt(\mathsfit{u}\rt) \rt) }_{2}^2   
        \leq
        r_4 \norm{H}_2^2 \norm{\lambda - \lambda'}_2^2
        \; .
        \nonumber
    \]
\end{theoremEnd}
\vspace{-1ex}
\begin{proofEnd}
For clarity, let us denote $\bar{C} \triangleq C - C'$ \text{and} $\bar{m} \triangleq m - m'$ such that
\[
    \mathcal{T}_{\lambda}\lt(\mathsfit{u}\rt)   
    -
    \mathcal{T}_{\lambda'}\lt(\mathsfit{u}\rt)   
    &=
    \lt( C \mathsfit{u} + m \rt)
    -
    \lt( C' \mathsfit{u} + m' \rt)
    \nonumber
    \\
    &=
    \lt( C - C'\rt) \mathsfit{u} + \lt( m - m' \rt)
    \nonumber
    \\
    &=
    \bar{C} \mathsfit{u} + \bar{m} \; .
    \nonumber
\]
Then 
\[
    \norm{\mathsfit{U} H \lt( \mathcal{T}_{\lambda}\lt(\mathsfit{u}\rt) - \mathcal{T}_{\lambda'}\lt(\mathsfit{u}\rt) \rt) }_{2}^2
    &=
    \norm{ \mathsfit{U} H \lt( \bar{C} \mathsfit{u} + \bar{m} \rt) }_{2}^2
    \nonumber
    \\
    &=
    \underbrace{
        \mathbb{E}\norm{ \mathsfit{U} H \bar{C} \mathsfit{u} }_{2}^2
    }_{V_{\text{scale}}}
    + 
    2
    \underbrace{
    \inner{
        \mathsfit{U} H \bar{m}, H \bar{C} \mathsfit{u}
    }
    }_{V_{\text{cross}}}
    + 
    \underbrace{
    \mathbb{E}\norm{ \mathsfit{U} H \bar{m} }_{2}^2
    }_{V_{\text{loc}}} \; .
    \nonumber
\]

$V_{\text{loc}}$ and $V_{\text{cross}}$ are straightforward.
Under \cref{assumption:noise}, it immediately follows that
\[
    V_{\text{loc}}
    &=
    \mathbb{E}\norm{ \mathsfit{U} H \bar{m} }_{2}^2
    \nonumber
    \\
    &=
    \bar{m}^{\top} H^{\top} \mathbb{E} U^2 H \bar{m}
    \nonumber
    \\
    &=
    \bar{m}^{\top} H^{\top} H \bar{m}
    \qquad\qquad\text{(\cref{thm:noise})}
    \nonumber
    \\
    &=
    \norm{H \bar{z} }_2^2 \; .
    \nonumber
\]
On the other hand, 
\[
    V_{\text{cross}}
    &=
    \mathbb{E}
    \inner{
        \mathsfit{U} H \bar{m}, \, \mathsfit{U} H \bar{C} \mathsfit{u}
    }
    \nonumber
    \\
    &=
    \bar{m}^{\top} H^{\top} \lt( \mathbb{E}  \mathsfit{U}^2 H \bar{C} \mathsfit{u} \rt) \; . 
    \nonumber
\]
The expectation follows as
\[
    {\lt[
    \mathbb{E} U^2 H \bar{C} \mathsfit{u}
    \rt]}_i
    \nonumber
    &=
    \mathbb{E}
    \mathsfit{u}_i^2 \sum_{j=1}^d H_{ij} \bar{C}_{jj} \mathsfit{u}_j
    \nonumber
    \\
    &=
    H_{ii} \bar{C}_{ii} \mathbb{E} \mathsfit{u}_i^3 
    +
    \sum_{j \neq i} H_{ij} \bar{C}_{jj} \mathbb{E} \mathsfit{u}_i^2 \mathbb{E}  \mathsfit{u}_j
    \nonumber
    \\
    &=
    0 \; .
    &&\text{(\cref{assumption:noise})}
    \nonumber
\]
Thus, the cross term $V_{\text{cross}}$ vanishes.

$V_{\text{scale}}$ requires careful elementwise inspection in order to apply \cref{assumption:noise}.
That is,
\[
    V_{\text{scale}}
    &=
    \mathbb{E}\norm{ \mathsfit{U} H \bar{C} \mathsfit{u} }_{2}^2
    \nonumber
    \\
    &=
    \mathbb{E} \sum_{i=1}^d \mathsfit{u}_i^2 {\lt\{ \sum^d_{j=1} H_{ij} \bar{C}_{jj} \mathsfit{u}_j \rt\}}^2
    \nonumber
    \\
    &=
    \mathbb{E} \sum_{i=1}^d \mathsfit{u}_i^2 {\lt\{ H_{ii} \bar{C}_{ii} \mathsfit{u}_i + \sum_{j \neq i} H_{ij} \bar{C}_{jj} \mathsfit{u}_j \rt\}}^2
    \nonumber
    \\
    &=
    \mathbb{E} \sum_{i=1}^d \mathsfit{u}_i^2 
    {\lt\{ 
    H_{ii}^2 \bar{C}_{ii}^2 \mathsfit{u}_i^2 
    + 2 H_{ii} \bar{C}_{ii} \mathsfit{u}_i \lt( \sum_{j \neq i} H_{ij} \bar{C}_{jj} \mathsfit{u}_j \rt)
    + {\lt( \sum_{j \neq i} H_{ij} \bar{C}_{jj} \mathsfit{u}_j \rt)}^2
    \rt\}}
    \nonumber
    &&\text{(expand quadratic)}
    \\
    &=
    \sum_{i=1}^d
    {\lt\{ 
    H_{ii}^2 \bar{C}_{ii}^2 \mathbb{E} \mathsfit{u}_i^4
    + 2 H_{ii} \bar{C}_{ii} \mathbb{E} \mathsfit{u}_i^3 \mathbb{E} \lt( \sum_{j \neq i} H_{ij} \bar{C}_{jj} \mathsfit{u}_j \rt)
    + \mathbb{E} \mathsfit{u}_i^2 \mathbb{E} {\lt( \sum_{j \neq i} H_{ij} \bar{C}_{jj} \mathsfit{u}_j \rt)}^2
    \rt\}}
    &&\text{(distribute $\mathsfit{u}_i^2$)}
    \nonumber
    \\
    &=
    \sum_{i=1}^d
    {\lt\{ 
    r_4 H_{ii}^2 \bar{C}_{ii}^2 
    + 
    \mathbb{E} {\lt( \sum_{j \neq i} H_{ij} \bar{C}_{jj} \mathsfit{u}_j \rt)}^2
    \rt\}}
    \nonumber
    &&\text{(\cref{assumption:noise})}
    \\
    &=
    \sum_{i=1}^d
    {\lt\{ 
    r_4 H_{ii}^2 \bar{C}_{ii}^2 
    + 
    \sum_{j \neq i} 
    {\lt( 
        H_{ij}^2 \bar{C}_{jj}^2 \mathbb{E} \mathsfit{u}_j^2
        +
        \sum_{k \neq j} 
        H_{ij} \bar{C}_{jj} \mathbb{E} \mathsfit{u}_j 
        H_{ik} \bar{C}_{kk} \mathbb{E} \mathsfit{u}_k
    \rt)}
    \rt\}}
    \nonumber
    &&\text{(expand quadratic)}
    \\
    &=
    \sum_{i=1}^d
    \lt\{
    r_4 H_{ii}^2 \bar{C}_{ii}^2 
    + 
    \sum_{j \neq i} 
    H_{ij}^2 \bar{C}_{jj}^2 
    \rt\}
    &&\text{(\cref{assumption:noise})}
    \nonumber
    \\
    &=
    \sum_{i=1}^d
    \sum_{j=1}^d 
    H_{ij}^2 \bar{C}_{jj}^2 
    +
    \lt( r_4 - 1 \rt)
    \sum_{i=1}^d
    H_{ii}^2 \bar{C}_{ii}^2 
    \nonumber
    \\
    &=
    \norm{ H \bar{C} }_{\mathrm{F}}^2
    +
    \lt( r_4 - 1 \rt)
    \norm{\mathrm{diag}\lt(H \bar{C}\rt)}_{\mathrm{F}}^2 \; .
    \nonumber
\]
Combining everything, 
\[
    \norm{ H \lt( \mathcal{T}_{\lambda}\lt(\mathsfit{u}\rt) - z \rt) }_{\mathsfit{U}^2}^2   
    &=
    V_{\text{loc}} + 2 V_{\text{cross}} + V_{\text{scale}}
    \nonumber
    \\
    &=
    \norm{H \bar{m} }_2^2
    +
    \norm{ H \bar{C} }_{\mathrm{F}}^2
    +
    \lt( r_4 - 1 \rt)
    \norm{\mathrm{diag}\lt(H \bar{C}\rt)}_{\mathrm{F}}^2
    \label{eq:constant_hessian_combine}
\]
From the property of the Frobenius norm, for any matrix $A \in \mathbb{R}^{d \times d}$, we can decompose 
\[
    \norm{ A }_{\mathrm{F}}^2
    \quad=\quad
    \sum_{i=1}^d \sum_{j=1}^d A_{ij}^2
    \quad=\quad
    \sum_{i=1}^d A_{ii}^2
    +
    \sum_{i=1}^d \sum_{i \neq j} A_{ij}^2
    \quad=\quad
    \norm{ \mathrm{diag}\lt( A \rt) }_{\mathrm{F}}^2
    +
    \norm{ \mathrm{off}\lt( A \rt) }_{\mathrm{F}}^2  \; ,
    \nonumber
\]
where $\mathrm{off}(A)$ is a function that zeroes-out the diagonal of $A$.
Then from \cref{eq:constant_hessian_combine}, 
\[
    \norm{ H \lt( \mathcal{T}_{\lambda}\lt(\mathsfit{u}\rt) - z \rt) }_{\mathsfit{U}^2}^2   
    &=
    \norm{H \bar{m} }_2^2
    +
    \norm{ \mathrm{off}\lt( H \bar{C} \rt) }_{\mathrm{F}}^2
    +
    \norm{ \mathrm{diag}\lt( H \bar{C} \rt) }_{\mathrm{F}}^2
    +
    \lt( r_4 - 1 \rt)
    \norm{\mathrm{diag}\lt(H \bar{C}\rt)}_{\mathrm{F}}^2
    \nonumber
    \\
    &=
    \norm{H \bar{m} }_2^2
    +
    \norm{ \mathrm{off}\lt( H \bar{C} \rt) }_{\mathrm{F}}^2
    +
    r_4
    \norm{\mathrm{diag}\lt(H \bar{C}\rt)}_{\mathrm{F}}^2
    \nonumber
    \\
    &\leq
    r_4 \norm{H \bar{m} }_2^2
    +
    r_4
    \norm{ \mathrm{off}\lt( H \bar{C} \rt) }_{\mathrm{F}}^2
    +
    r_4
    \norm{\mathrm{diag}\lt(H \bar{C}\rt)}_{\mathrm{F}}^2
    &&\text{(\cref{thm:kurtosis_bound})}
    \nonumber
    \\
    &=
    r_4
    \lt(
    \norm{H \bar{m} }_2^2
    +
    \norm{ H \bar{C} }_{\mathrm{F}}^2
    \rt)
    \nonumber
    \\
    &\leq
    r_4 \norm{H}_2^2 
    \lt(
    \norm{\bar{m}}_2^2
    +
    \norm{\bar{C} }_{\mathrm{F}}^2
    \rt)
    &&\text{(operator norm)}
    \nonumber
    \\
    &=
    r_4 \norm{H}_2^2 \norm{\lambda - \lambda'}_2^2 \; ,
    \nonumber
\]
which is the stated result.
\end{proofEnd}

The remaining part of the proof closely resembles the proof sketch of   \cref{thm:gradient_variance_upperbound_meanfield_general}.
For convenience, we first restate \cref{thm:gradient_variance_upperbound_meanfield_general} and then proceed to the full proof.

\printProofs[gradientvarianceupperboundmeanfieldgeneral]

\newpage
\subsubsection{Proof of \cref{thm:weighted_norm_smoothness}}\label{section:proof_weighted_norm_smoothness}
\printProofs[weightednormsmoothness]

\newpage
\subsubsection{Proof of \cref{thm:constant_hessian}}\label{section:proof_constant_hessian}
\printProofs[constanthessian]

\newpage
\subsubsection{Proof of \cref{thm:gradient_variance_lowerbound}}\label{section:proof_gradient_variance_lowerbound}

For any $\mu, L \in (0, \infty)$ such that $\mu \leq L$, our goal is to obtain a matrix-valued function $H_{\mathrm{worst}} : \mathbb{R}^d \to \mathbb{S}_{\succ 0}^d $ satisfying 
\[
    \mu \mathrm{I}_d \quad\preceq\quad H_{\mathrm{worst}} \quad\preceq\quad  L \mathrm{I}_d 
    \nonumber
\]
that, under the choice $H = H_{\mathrm{worst}}$, maximizes the quantity 
\[
    \norm*{\mathsfit{U} \int^1_0 H\lt(\mathsfit{z}^w\rt) \lt( \mathsfit{z} - \bar{z} \rt) \mathrm{d}w }^2_{2} \; ,
    \label{eq:lower_bound_target}
\]
where $\bar{z} \in \{ z \mid \nabla \ell\lt(z\rt) = 0 \}$ is any stationary point of $\ell$,  $\mathsfit{z} \triangleq \mathcal{T}_{\lambda}(\mathsfit{u})$, and $\mathsfit{z}^w \triangleq w \mathsfit{z} + \lt(1 - w\rt) \bar{z}$.
Given the norm constraint, the worst-case example that maximizes \cref{eq:lower_bound_target} will be the matrix-valued function that approximately results in
\[
    \norm*{\mathsfit{U} \int^1_0 H\lt(\mathsfit{z}^w\rt) \lt( \mathsfit{z} - \bar{z} \rt) \mathrm{d}w }^2_{2}
    \quad\asymp\quad
    L^2
    \norm{\mathsfit{U}}_2^2
    \norm*{\mathsfit{z} - \bar{z}}^2_{2} \; 
    \nonumber
\]
for \textit{any} realization of $\mathsfit{u}$ on $\mathbb{R}^d$.
For this, we will establish the relations
\[
    \norm*{\mathsfit{U} \int^1_0 H\lt(\mathsfit{z}^w\rt) \lt( \mathsfit{z} - \bar{z} \rt) \mathrm{d}w }^2_{2} 
    \quad=\quad
    \norm*{\mathsfit{U} H\lt( \mathsfit{z}^w \rt) \lt( \mathsfit{z} - \bar{z} \rt) }^2_{2} 
    \quad\asymp\quad
    L \norm{\mathsfit{U}}_2^2 \norm*{\mathsfit{z} - \bar{z}}^2_{2} \; 
    \; .
    \label{eq:lowerbound_steps}
\]

The first equality in \cref{eq:lowerbound_steps} follows from identifying the conditions where $H(\mathsfit{z}^w)$ is independent of the value of $w$.
For the specific choice of 
\[
    m = \bar{z} = 0_d , \qquad C = \mathrm{diag}\lt(\delta, \ldots, \delta\rt) ,
    \qquad \text{any $\delta > 0$} \; ,
    \nonumber
\]
$H(\mathsfit{z}^w)$ is independent of $w$ if it only depends on the quantities
\[
  \mathsfit{i}_* = \argmax_{i=1, \ldots, d} \abs{\mathsfit{z}_i^w}
  \qquad\text{and}\qquad
  \hat{\mathsfit{z}}^w \triangleq \frac{\mathsfit{z}^w}{\norm{\mathsfit{z}^w}_2} \; .
\]
That is, with some abuse of notation, $H(\mathsfit{z}^w) = H(\hat{\mathsfit{z}}^w, \mathsfit{i}_*)$.

\begin{lemma}\label{eq:matrix_constant_w}
    Suppose $m = \bar{z} = 0_d$, and for any $\delta > 0$, $C = \mathrm{diag}\lt(\delta, \ldots, \delta\rt)$.
    Then, if $H(\mathsfit{z}^w)$ is a function of $\mathsfit{i}_*$ and $\hat{\mathsfit{z}}^w$, then $H(\mathsfit{z}^w)$ is constant with respect to $w \in [0, 1]$.
\end{lemma}
\begin{proof}
    It suffices to show that, under the stated conditions, the values of $\mathsfit{i}_*$ and $\hat{\mathsfit{z}}^w$ are invariant to $w$.
    For $\hat{\mathsfit{z}}^w$, this is trivially follows from the assumption that $\bar{z} = 0$ as
    \[
        \hat{\mathsfit{z}}^w 
        = \frac{\mathsfit{z}^w}{\norm{\mathsfit{z}^w}_2}
        = \frac{w \mathsfit{z} + (1 - w) \bar{z}}{\norm{w \mathsfit{z} + (1 - w) \bar{z}}_2}
        = \frac{w \mathsfit{z}}{\norm{w \mathsfit{z}}_2}
        = \frac{\mathsfit{z}}{\norm{\mathsfit{z}}_2} \; .
        \nonumber
    \]
    For $\mathsfit{i}_*$, we use the fact that the diagonal matrix $C$ is isotropic as
    \[
        \argmax_{i=1, \ldots, d}\; \abs{
            \mathsfit{z}^w_i
        }
        \quad=\quad \argmax_{i=1, \ldots, d}\; w \, C_{ii} \abs{\mathsfit{u}_i}
        \quad=\quad \argmax_{i=1, \ldots, d}\; w \delta \, \abs{\mathsfit{u}_i}
        \quad=\quad \argmax_{i=1, \ldots, d}\; \abs{\mathsfit{u}_i} \; .
        \nonumber
    \]
\end{proof}
From $H(\mathsfit{z}^w) = H(\hat{\mathsfit{z}}^w, \mathsfit{i}_*)$, the integral in \cref{eq:lower_bound_target} can be solved as
\[
    \norm*{\mathsfit{U} \int^1_0 H\lt(\mathsfit{z}^w\rt) \lt( \mathsfit{z} - \bar{z} \rt) \mathrm{d}w }^2_{2} 
    \quad=\quad
    \norm*{\mathsfit{U} H\lt( \mathsfit{z}^w \rt) \lt( \mathsfit{z} - \bar{z} \rt) }^2_{2} \; .
    \nonumber
\]
It remains to construct $H$ in a way that depends only on $\hat{\mathsfit{z}}^w$ and $\mathsfit{i}_*$ such that 
\[
    \norm*{\mathsfit{U} H\lt( \mathsfit{z}^w \rt) \lt( \mathsfit{z} - \bar{z} \rt) }^2_{2} 
    \quad\asymp\quad
    L \norm{\mathsfit{U}}_2^2 \norm*{\mathsfit{z} - \bar{z}}^2_{2} \; .
    \nonumber
\]
Recalling the spectral constraints, this is equivalent to, for all $z \in \mathbb{R}^d$, $H$ solving the equation
\[
    H\lt( \mathsfit{i}_* \rt) \mathsfit{z}
    =
    L \, \norm{\mathsfit{z}}_2 \,
    \mathrm{e}_{\mathsfit{i}_*}
    \quad \text{subject to} \quad
    \mu \mathrm{I}_d \leq H\lt(  \mathsfit{z}^w \rt) \leq L \mathrm{I}_d
    \; .
    \label{eq:lower_bound_goal}
\]
Notice the equivalence
\[
    H\lt(  \mathsfit{z}^w \rt) \mathsfit{z}
    =
    L \, \norm{\mathsfit{z}}_2 \mathrm{e}_{\mathsfit{i}_*}
    \qquad\Leftrightarrow\qquad
    H\lt(  \mathsfit{z}^w \rt) \,
    \frac{\mathsfit{z}^w}{\norm{\mathsfit{z}^w}_2}
    =
    L
    \mathrm{e}_{\mathsfit{i}_*}
    \; .
    \nonumber
\]
Thus, $\hat{\mathsfit{z}}^w$ and $\mathsfit{i}_*$ contain all the information we need.
The following matrix-valued function almost solves \cref{eq:lower_bound_goal}:
\[
    H_{\mathrm{worst}}\lt(z\rt) = \alpha \mathrm{I}_d + \frac{\beta}{2} \lt( \mathrm{e}_{\mathsfit{i}_*} { \hat{z} }^{\top} + \hat{z} 
 \, \mathrm{e}_{\mathsfit{i}_*}^{\top} \rt) , 
    \quad\text{where}\quad \hat{z} = \frac{z}{\norm{z}_2} \; .
    \label{eq:hworst}
\]
This function is reminiscent of a householder reflector~\citep[Eq. 10.4]{trefethen_numerical_1997} with some modifications to satisfy the eigenvalue constraint.
That is, from the fact that both $\mathrm{e}_{\mathsfit{i}_*}$ and $\hat{z}$ have a unit norm, it is apparent that this matrix satisfies \cref{assumption:almost_constant_hessian} with $H = \alpha \mathrm{I}_d$ and $\delta = \beta$.
Furthermore, by setting the constants as
\[
    \alpha = \frac{L + \mu}{2} 
    \quad\text{and}\quad
    \beta = \frac{L - \mu}{2} \; ,
    \label{eq:thm_gradient_variance_lowerbound_alpha_beta}
\]
the triangle inequality asserts that the eigenvalue constraint $\mu \mathrm{I}_d \leq H_{\mathrm{worst}} \leq L \mathrm{I}_d $ is satisfied almost surely.

Given the specific form of $H_{\mathrm{worst}}$, we are now ready to formally prove \cref{thm:gradient_variance_lowerbound}.
Let us first restate the proposition for convenience and then proceed to the proof.

\printProofs[gradientvariancelowerbound]

\end{document}